%% file: main_convergence.tex
\documentclass[11pt]{article}

\input{commands.tex}

\usepackage[margin=1in]{geometry}

\usepackage{adjustbox}
\usepackage{multirow} 
\usepackage{booktabs}
\usepackage{multirow}
\usepackage{colortbl}
\usepackage{array}
\usepackage{pgf}
 
\usepackage{xcolor}   
\usepackage{pifont}

\definecolor{darkgreen}{RGB}{0,100,0}  
\usepackage{adjustbox}
\usepackage{threeparttable}
\usepackage{subcaption}
\usepackage{titletoc}

\begin{document}

\title{Convergence of Gradient Descent for General Neural Network Architectures Beyond the NTK Regime}

\author{Yuqing Wang\\
Johns Hopkins University\\
\texttt{ywan1050@jh.edu} \\
       }
\date{}

\maketitle

\begin{abstract}

Training dynamics is central to understanding neural networks, yet its theoretical analysis remains difficult even for simple architectures and becomes substantially more challenging for general modern architectures. In this paper, we propose a convergence framework for analyzing gradient descent (GD) dynamics under a broad family of neural network architectures and datasets beyond the neural tangent kernel (NTK) regime. The framework is formulated at the level of network blocks and covers architectures including pre-normalized multi-layer transformers. More precisely, under mild assumptions, we prove that for almost all initializations, GD with regular learning rates converges to the neighbourhood of a stationary point. This is mainly proved by establishing an iterate-dependent PL-type inequality through analyticity and measure-zero arguments, and by proving Lipschitz smoothness along the GD trajectory through polynomial generalized smoothness and a local relaxed dissipative condition. We further interpret the theorem under Xavier initialization and practical architectural scaling, showing that the learning rate scale depends on the depth and effective bottleneck dimensions rather than the largest width. Finally, we derive structural nondegeneracy implications for residual connections and function composition, and provide a generic characterization of global minimizers within our framework.



\end{abstract}

\tableofcontents
\newpage

\section{Introduction}

Understanding the training dynamics of neural networks is a fundamental problem in modern machine learning. The behavior of gradient-based training determines how representations are learned, which stationary regions are reached, and how architectural or algorithmic choices affect generalization. Many architectural designs and training strategies that are now routine in practice, such as residual connections \citep{he2016deep}, normalization \citep{ioffe2015batch,ba2016layer}, large learning rates \citep{lewkowycz2020large,wang2022large,wang2023good}, weight decay \citep{loshchilov2017decoupled}, and compositional architectures, remain difficult to explain because the underlying dynamics is highly nonlinear. Thus, a better understanding of training dynamics is not only important for theoretical purposes; it is also a main barrier to understanding many empirical phenomena and improving generalization performance.

A large body of theory studies this problem through the neural tangent kernel (NTK) framework. The NTK viewpoint gives powerful convergence guarantees in highly overparameterized regimes, but these guarantees typically rely on very large widths, small or infinitesimal learning rates, and training dynamics that remains close the initialization \citep{jacot2018neural,allen2019convergence,du2019gradient,lee2019wide,zou2020gradient,zou2019improved,chen2020much,oymak2020toward}. In this sense, the NTK provides a local characterization of training dynamics, and the resulting networks behave close to linear models around initialization. 

This limitation has motivated a growing body of work on feature learning, where representations evolve during training rather than staying close to their initialization \citep{yang2020feature,chen2022feature,chen2025global,ba2022high,allen2020towards,allen2022feature,cao2022benign,telgarsky2022feature}. This line of research is closer to the behavior expected in practical training, and has revealed mechanisms that are invisible in the lazy regime. However, most existing analyses still focus on analytically tractable settings, such as two-layer networks, infinite-width limits, or other simplified models.


In contrast, the architectures used in modern machine learning are far more heterogeneous. Transformers \citep{vaswani2017attention}, U-Nets \citep{ronneberger2015u}, and state-space models such as Mamba \citep{gu2023mamba} achieve strong empirical performance across many domains, but they are built from compositions of several nonlinear modules, including attention, normalization, residual connections, convolutional or structured state-space blocks, and feedforward layers. Therefore, understanding only simple network structures is not sufficient for explaining the training behavior of practical models. This motivates convergence tools that apply directly to broad families of multi-layer and compositional neural networks.

Despite this need, the theoretical analysis of complicated neural network dynamics faces major obstacles. Classical nonconvex convergence analyses of GD usually rely on two types of assumptions: a relaxation of convexity, such as a Polyak-Lojasiewicz (PL) condition, and a smoothness condition, such as Lipschitz smoothness of the objective. The neural network losses generally do not satisfy either assumptions globally. For the first condition, \citet{liu2022loss} show that overparameterized nonlinear systems and neural networks satisfy a PL condition outside certain singular regions. This gives an important characterization of the loss landscape, but it does not fully describe all regions that may be visited by training; for example, convergence can be sublinear in some regimes \citep{xu2023over}, which is not captured by a uniform PL condition. For the smoothness condition, several extensions of Lipschitz smoothness have been proposed, including generalized smoothness originated from $(L_0,L_1)$-smoothness \citep{zhang2019gradient} and further developed by \citet{li2023convex}. However, for complicated architectures, these conditions can remain implicit and difficult to verify directly from the network structure. This calls for smoothness conditions that are formulated at the layer or block level, remain mild for broad neural network families, and can be verified directly from the architectural components of modern networks.

Motivated by these considerations, this paper develops a convergence framework for GD beyond the NTK regime. The framework is designed to be stated and verified at the level of network blocks, while still tracking the actual nonlinear GD trajectory rather than a fixed linearization around initialization. It covers broad neural network families built from feedforward blocks, convolutional blocks, attention blocks, analytic activations, normalization layers, residual connections, and finite compositions such as pre-normalized multi-layer transformers. Its key idea is to combine analytic genericity, polynomial block-level regularity, and a local dynamical condition to obtain loss decay without assuming global PL or global Lipschitz-smoothness. More precisely, our main contributions are the following:
\begin{itemize}
     \item In place of global Lipschitz smoothness, we introduce mild architecture-level regularity and dynamical conditions for convergence. We propose polynomial generalized continuity and smoothness (Definitions~\ref{def:poly_generalized_continuous} and~\ref{def:poly_generalized_smoothness}), which are preserved under basic operations (Proposition~\ref{prop:poly_generalized_properties}); additionally, we establish a local relaxed dissipative condition near stationary regions (Definition~\ref{def:dissipative}) and show that $C^1$ functions satisfy this condition (Proposition~\ref{prop:dissipativity_rho}).
     \item In place of global PL assumption, we prove that a parameter-dependent PL-type inequality holds for almost all points in the parameter space, and GD avoids stagnation along the trajectory (Lemma~\ref{lem:mu_low_along_gd}).
     \item As our main theorem, we prove the convergence of GD with regular learning rates to the neighbourhood of a stationary point for almost every initialization and a broad class of datasets and neural networks (Theorem~\ref{thm:training}). The result applies to architectures built from the covered blocks (Example~\ref{ex:blocks}); in particular, a pre-normalized multi-layer transformer satisfies the model assumptions under the stated width condition (Lemma~\ref{lem:prenorm_transformer_model_assumptions}).
     \item We interpret the convergence theorem under practical initialization and architectural settings (Corollary~\ref{cor:Gaussian_learning_rate_order} and~\ref{cor:uniform_learning_rate_order}). For Xavier-normal and Xavier-uniform initializations, the learning rate scale is controlled by depth, bottleneck dimensions, and input dimension, rather than by the largest hidden width used in NTK analyses.
     \item We provide architectural insights into residual connections and function composition: adding a small identity component generically restores nondegeneracy of analytic maps, which helps explain the role of residual and compositional structures (Lemma~\ref{lem:full_rank_jacobian_mainbody}).
     \item We characterize a global-minimum implication of the generic nondegeneracy condition: outside the measure-zero bad parameter set excluded in the convergence theorem, any stationary point is a global minimizer (Theorem~\ref{thm:global_min_generic_nonlinearity}).
\end{itemize}

\section{Related Work}
\label{sec:related_works}

\textbf{Convergence of Neural Networks.} In the overparameterized regime, the neural tangent kernel (NTK) framework, introduced by \citet{jacot2018neural}, has led to a series of convergence results showing that when the network width is sufficiently large---typically a high-order polynomial in the number of samples $N$, the network depth $L$, and the inverse of the angle between data points---(S)GD with infinitesimal learning rates can drive the training loss to zero exponentially \citep{allen2019convergence, du2019gradient, lee2019wide, zou2020gradient, zou2019improved, ji2019polylogarithmic, chen2020much, song2019quadratic, oymak2020toward}. However, in the NTK regime, networks essentially behave as linear models and do not exhibit meaningful feature learning. To address this limitation, recent work has moved towards analyzing feature learning, using either gradient flow or a finite number of large steps \citep{yang2020feature,chen2022feature,chen2025global,ba2022high,allen2020towards, allen2022feature, cao2022benign, shi2021theoretical, telgarsky2022feature}. There is also a series of works analyzing SGD convergence under various regularity assumptions, including the PL condition, quadratic growth, the aiming condition, and Lipschitz smoothness \citep[etc.]{liu2023aiming,ma2018power}. Another line of work adopts a mean-field perspective \citep{song2018mean, chizat2018global, rotskoff2018neural, wei2019regularization, chen2020generalized, sirignano2020mean, fang2021modeling}, a dynamical mean-field theory perspective \citep{bordelon2022self,bordelon2023dynamics}, and a distributional perspective \citep{han2025precise}. In contrast to these approaches, which typically assume fixed architectures and/or infinitesimal learning rates, our work introduces tools applicable to a broad class of networks and remains effective under larger learning rates, offering a theoretical foundation for convergence in more practical and complex settings.

\textbf{Residual Connections.} Residual connections have been widely studied for their theoretical and practical benefits in deep learning. \citet{hardt2016identity} showed that deep linear residual networks exhibit no spurious local optima and maintain strong expressivity, a result further supported by \citet{liu2019towards}, who demonstrated the absence of spurious minima in residual architectures. \citet{huang2020deep} highlighted the role of residual connections in preserving learnability across depth, while \citet{scholkemper2024residual} found that they help mitigate oversmoothing in graph neural networks. In the context of transformers, \citet{qin2025convergence} showed that residual connections improve the conditioning of output matrices in single-layer architectures with feedforward networks and attention. In contrast to these works, our study analyzes more general and complex architectures and provides a theoretical explanation based on analytic nondegeneracy for the intrinsic nondegeneracy of the layerwise mappings induced by residual connections throughout training.

\section{Preliminary}
\label{sec:preliminary}
We use $\|\cdot\|$ to denote the $\ell^2$ norm for vectors and matrices, and use $\|\cdot\|_p$ to denote the $\ell^p$ norm for vectors and the $L^p(\Omega)$ norm for functions. The inner product between two vectors is denoted by $\ip{\cdot}{\cdot}$. For a function $\psi:\RR^m\to\RR^n$, we use $\nabla \psi\in\RR^{n\times m}$ to denote the Jacobian of $\psi$ with respect to all variables, and use $\nabla_w\psi$ to denote the Jacobian with respect to $w$. For a multi-index $\alpha=(\alpha_1,\cdots,\alpha_d)$, we denote $|\alpha|=\sum_{i=1}^d\alpha_i$; for a set $\Omega$, we use $|\Omega|$ to denote its volume and $\bar{\Omega}$ to denote its closure. For two measures $\nu$ and $\mu$, we use $\nu\ll\mu$ to denote absolute continuity of $\nu$ with respect to $\mu$. We denote by $\Leb_d(\cdot)$ the Lebesgue measure of a set in $\RR^d$. We denote by $B_w(r)$ the open Euclidean ball of radius $r$ centered at $w$. We use $\RR_{\ge 0}$ to denote the set of non-negative real numbers. We use $I$ for the identity matrix of compatible dimension, $\mathbf{1}$ for the all-ones vector, and $\odot$ for entrywise product. For a symmetric positive semidefinite matrix $A$, $\lambda_{\min}(A)$ denotes its smallest eigenvalue. We use $\cO(\cdot)$, $\Omega(\cdot)$, and $\Theta(\cdot)$ in their standard asymptotic senses.

Next, we introduce polynomial generalized continuity and smoothness, analyticity, and a local relaxed dissipative condition as key properties for proving the convergence of GD for neural networks.

\subsection{Polynomial Generalized Continuity and Smoothness}
\label{subsec:poly_bound_conti_smooth}

Classical non-convex optimization theory typically assumes Lipschitz smoothness of the objective function, a condition that is rarely satisfied in neural network training. One generalization of Lipschitz smoothness stems from the concept $(L_0, L_1)$-smoothness \citep{zhang2019gradient}. Building on this, \citet{li2023convex} introduced generalized smoothness for a broader class of non-Lipschitz smooth functions. Motivated by these developments, we propose a new formulation in nonconvex optimization tailored to neural network architectures, based on the following concepts.

We first separate the required control into three related notions: polynomial boundedness of the map, polynomial generalized continuity of the map, and polynomial smoothness through the continuity of its Jacobian.
\begin{definition}[Poly-Boundedness]
\label{def:poly_bound}
    A function $\psi(w):\Omega\subseteq\RR^n\to\RR^d$ is polynomially bounded if
    $$
        \| \psi(w)\|\le S(\|w_{1}\|,\cdots,\|w_{n_w}\|),\ \forall w\in \Omega,
    $$
    where $w_i\in\RR^{n_{i,1}\times n_{i,2}}$, and $\dim w=n=\sum_{i=1}^{n_w}n_{i,1}n_{i,2}$; $S(\cdot)$ is a polynomial whose coefficients are all positive.
\end{definition}

\begin{definition}[Poly-Continuity]
\label{def:poly_generalized_continuous}
    A function $\psi(w):\Omega\subseteq\RR^n\to\RR^d$ satisfies polynomial generalized continuity if
    $$
        \| \psi(w)- \psi(w')\|\le S(\|w_{\max,1}\|,\cdots,\|w_{\max,n_w}\|)\|w-w'\|,\ 
    $$
    $\forall w,w'\in \Omega,$ where $\|w_{\max,i}\|=\max \{\|w_i\|,\|w_i'\|\}$, and $S(\cdot)$ is a polynomial whose coefficients are all positive.
\end{definition}

\begin{definition}[Poly-Smoothness]
\label{def:poly_generalized_smoothness}
    A function $\psi(w):\Omega\subseteq\RR^n\to\RR^d$ in $C^1$ satisfies polynomial generalized smoothness if $\nabla \psi$ satisfies polynomial generalized continuity.
\end{definition}

One advantage of these definitions is that they are closed under the basic operations used to build neural networks. This property is summarized in the following proposition.

\begin{proposition}
\label{prop:poly_generalized_properties}
    Assume $\psi$ and $\tilde{\psi}$ satisfy polynomial boundedness for both the functions and their gradients, polynomial generalized continuity, and polynomial generalized smoothness. Then, whenever the following operations are evaluated on compatible domains,
    \begin{enumerate}
        \item $\psi(w)+\tilde{\psi}(w)$, where $\psi,\tilde{\psi}:\RR^d\to\RR^m$,
        \item $\psi(w)\tilde{\psi}(w)$, where $\psi(w)\in\RR^{m\times d}$, and $\tilde{\psi}(w)\in\RR^{d\times n}$,
        \item $(\psi\circ \tilde{\psi})(w)$, where $\psi:\RR^d\to\RR^m$, and $\tilde{\psi}:\RR^n\to\RR^d$,
    \end{enumerate}
    satisfy the same properties with possibly different polynomials.
\end{proposition}

This property makes it sufficient to verify the definitions for standard elementary blocks used in neural networks. Below, we list some representative activation and normalization functions and show that they all satisfy polynomial boundedness, polynomial generalized continuity, and polynomial generalized smoothness. The proof of Proposition~\ref{prop:poly_generalized_properties} is in Appendix~\ref{subapp:poly_bound_continuity_smooth}.

\begin{example}
\label{ex:activation_normalization}
Let $t\in\RR$, $z\in\RR^m$, $w\in\RR^d$, and let $\epsilon>0$. We consider the following examples.
\begin{enumerate}
    \item \textbf{Activations.} Typical activations include 
    $$
       \text{polynomial } p(w)=\sum_{|\alpha|\le r}c_\alpha w^\alpha;\qquad
       \text{sigmoid } \sigma_{\mathrm{sig}}(t)=\frac{1}{1+e^{-t}};\qquad
        \tanh(t)=\frac{e^t-e^{-t}}{e^t+e^{-t}};
    $$
    $$
        \operatorname{softmax}(z)_j=\frac{e^{z_j}}{\sum_{q=1}^m e^{z_q}},\quad j=1,\cdots,m;\qquad  \operatorname{SiLU}(t)=t\,\sigma_{\mathrm{sig}}(t)=\frac{t}{1+e^{-t}};    $$
    $$
        \operatorname{GELU}(t)=t\Phi(t),\qquad
        \text{where }\Phi(t)=\frac{1}{\sqrt{2\pi}}\int_{-\infty}^t e^{-s^2/2}\,ds.
    $$
    \item \textbf{Normalizations.} Typical normalizations include
    $$
       \text{smoothed normalization } \tau_\epsilon(w)=\frac{w}{\sqrt{\|w\|^2+\epsilon^2}},
    $$
    $$
       \text{layer normalization } \mathrm{LN}(z)
        =\gamma\odot\frac{z-\bar z\mathbf{1}}{\sqrt{m^{-1}\|z-\bar z\mathbf{1}\|^2+\epsilon^2}}+\beta,
        \qquad \text{where }\bar z=m^{-1}\mathbf{1}^\top z,
    $$
    \begin{align*}
        \text{batch normalization }\mathrm{BN}(U)_{a,c}
        &=\gamma_c\frac{U_{a,c}-\frac{1}{|\cI|}\sum_{a\in\cI}U_{a,c}}{\sqrt{\frac{1}{|\cI|}\sum_{a\in\cI}\left(U_{a,c}-\frac{1}{|\cI|}\sum_{a\in\cI}U_{a,c}\right)^2+\epsilon^2}}+\beta_c, 
    \end{align*}
    for some tensor $U$ and index set $\cI$.
    \item \textbf{Finite Sums, Products, and Compositions} of the above maps, whenever the dimensions are compatible, are also included.
\end{enumerate}
\end{example}

\begin{lemma}
\label{cor:poly_smooth_with_NN}
    The functions in Example~\ref{ex:activation_normalization} and their gradients are polynomially bounded; moreover, these functions satisfy polynomial generalized continuity and polynomial generalized smoothness.
\end{lemma}
Together with Proposition~\ref{prop:poly_generalized_properties}, Lemma~\ref{cor:poly_smooth_with_NN} implies that many neural networks, including transformers, residual networks, and feedforward networks built from the above functions, satisfy the three properties (see Section~\ref{subsec:model_family_examples}). The proof of Lemma~\ref{cor:poly_smooth_with_NN} is in Appendix~\ref{subapp:poly_bound_continuity_smooth}.

For the convergence analysis, polynomial generalized smoothness plays the same structural role as a smoothness assumption in classical descent arguments. In particular, it yields the following descent-type upper bound.
\begin{lemma}
\label{lem:smoothness_inequalities}
    If the function $\psi(w)\in C^1$ satisfies polynomial generalized smoothness, then the following inequality holds
    \begin{align*}
        \psi(w')\le \psi(w)+\nabla \psi(w)^\top (w'-w)+\frac{S(\|w_{\max,1}\|,\cdots,\|w_{\max,d}\|)}{2}\|w-w'\|^2
    \end{align*}
\end{lemma}
\noindent The proof of Lemma~\ref{lem:smoothness_inequalities} is in Appendix~\ref{subapp:poly_bound_continuity_smooth}.

A brief comparison between polynomial generalized smoothness and the generalized smoothness condition in~\citet{li2023convex} is given in Appendix~\ref{subapp:generalized_smoothness_vs_poly}.

\subsection{Analytic Functions}
\label{subsec:analytic}

Analyticity is a useful property for measure-zero nondegeneracy arguments. For a compact domain $\cD$, real-analytic functions are dense in the space of continuous functions $C(\cD)$ with respect to uniform convergence on compact sets \citep{grauert1958levi}. Thus, assuming analyticity is reasonable in many settings. We use the following definition.
\begin{definition}[Analytic Function]
\label{def:analytic}
    A function  
    $\psi(w)=\begin{pmatrix}
        \psi_1(w)\\\vdots\\ \psi_m(w)
    \end{pmatrix}:U\to\RR^m,$
    where $U\subseteq\RR^m$ is open, is called real-analytic on $U$ if, for each $w_0\in U$ and each $i=1,\cdots,m$, the component $\psi_i$ can be represented by a convergent power series in a neighborhood of $w_0$.
\end{definition}
Many activation functions and neural network components are real-analytic, as summarized below.

\begin{lemma}
\label{cor:analytic_with_NN}
   The functions in Example~\ref{ex:activation_normalization} are all real-analytic.
\end{lemma}
The lemma above implies that feedforward layers using tanh, GELU, or SiLU, as well as attention layers with softmax or linear activation, are real-analytic. Consequently, architectures such as transformers, residual networks, and feedforward networks built from these real-analytic blocks are themselves real-analytic. The proof of Lemma~\ref{cor:analytic_with_NN} is in Appendix~\ref{subapp:analytic_function}.

\subsection{Local Relaxed Dissipative Condition}
\label{subsec:dissipative}

Besides poly-boundedness, continuity, and smoothness, we need an additional condition to control the local training dynamics. Motivated by weak dissipative conditions in operator theory, such as hypomonotonicity \citep{moudafi2004remark}, we introduce the following local relaxed dissipative condition.
    \begin{definition}[$(w,w^*,r,\rho,\epsilon)$-Dissipativity]
    \label{def:dissipative}
    A function $\psi(w)\in C^1$ satisfies the $(w,w^*,r,\rho,\epsilon)$-dissipative condition near $w$ if there exists some stationary point $w^*$ and some constant $r> \|w-w^*\|$, s.t., $\forall \tilde{w}\in B_{w^*}(r)\backslash\{v:\|\nabla \psi(v)\|\le \epsilon\}$, we have $$
        \nabla \psi(\tilde{w})^\top (\tilde{w}-w^*)\ge-\rho\|\nabla \psi(\tilde{w})\|^2,$$
    where $\epsilon>0$, and $\rho\in\RR$ is a constant depending on $w,w^*,\epsilon$, but independent of $\tilde{w}$.
\end{definition}
The above condition characterizes the quality of the local landscape near a stationary point by comparing the gradient direction with the direction toward the stationary point. Notably, it is a very general and mild property: once the tolerance $\epsilon$ is fixed, every continuously differentiable function satisfies the condition for some dissipativity constant. This is summarized in the following proposition.

\begin{proposition}
    \label{prop:dissipativity_rho}
Let $\psi\in C^1$ and let $w^*$ be any stationary point of $\psi$. For any constant $r>\|w-w^*\|$ and any $\epsilon>0$, $\psi$ satisfies the $(w,w^*,r,\rho,\epsilon)$-dissipative condition near $w$ for some dissipativity constant $\rho$ satisfying
$$
    \rho\le \frac{r}{\epsilon}.
$$
\end{proposition}

The proposition provides a default upper bound $r/\epsilon$ for the dissipativity constant. The value of $\rho$ can be negative, and more favorable local geometry leads to a smaller $\rho$. For instance, if $\psi$ is locally convex near $w^*$ on $B_{w^*}(r)$, then one can take $\rho=0$. The proof of Proposition~\ref{prop:dissipativity_rho} is in Appendix~\ref{subapp:dissipativity}.

\section{Convergence Theory}
\label{sec:theory}

In this section, we present a convergence framework for a family of neural networks under GD. We first specify the training setup in Section~\ref{subsec:training_setup}, state the model assumptions in Section~\ref{subsec:model_assumptions}, give concrete examples of architectures covered by the model family in Sections~\ref{subsec:architecture_example} and~\ref{subsec:model_family_examples}, and present the main convergence theorem in Section~\ref{subsec:main_convergence_theorem}. We discuss the interpretations of this convergence theory in Section~\ref{sec:interpretation} and implications in Section~\ref{sec:implications_convergence_theory}.

\subsection{Training Setup}
\label{subsec:training_setup}

We consider supervised data satisfying the following assumption.
\begin{assumption}[Data]
\label{assump:data_distribution_absolutely_continuous_Lebesgue_ground_truth}
    Let $x_1,\cdots,x_N\in\RR^d$ be input data, and $y_1,\cdots,y_N\in\RR^d$ be the corresponding outputs. Assume $$(x_1,\cdots,x_N)\,{\sim}\,\pi(z_1,\cdots,z_N)dz_1\cdots dz_N,$$ where $\pi(z_1,\cdots,z_N)\in L^\infty$ is supported on an open set $\Omega\times\cdots\times\Omega\subseteq\RR^{Nd}$ with $|\Omega|<\infty$, and $\pi\ll \Leb_{Nd}$. 
    
\end{assumption}

The input distribution assumption is used only to exclude degenerate input configurations, by requiring the joint law of $(x_1,\cdots,x_N)$ to be absolutely continuous with respect to Lebesgue measure. In particular, the convergence argument does not rely on the input samples being iid; it applies to any absolutely continuous joint input distributions, including those induced by data preprocessing, selection, or resampling. The outputs $y_i$ are arbitrary vectors in $\RR^d$, so the framework covers both noiseless labels and noisy observations.

We consider a family of neural networks as follows:
\begin{align}
\label{eqn:NN}
    u_{0,i}&=x_i;\quad\notag\\
    u_{\ell+1,i}&=u_{\ell,i}+{\varphi}_\ell(\theta_\ell;u_{\ell,i}),\forall\ell=0,\cdots,L-1;\quad \\
    f(\theta;x_i)&={\varphi}_L(\theta_L;u_{L,i})\notag
\end{align}
where $\theta=(\theta_0,\cdots,\theta_L)$ is the collection of weights with the $\theta_i$ corresponding to the $i$th layer. For the rest of the paper, we sometimes omit the subscript of $\theta$ for simplicity. We use the following $\ell^2$ loss $$\cL(\theta;\{x_i,y_i\}_{i=1}^N)=\frac{1}{N}\sum_{i=1}^Nl(f(x_i),y_i),$$ where  $l(\theta;x_i,y_i)=\frac{1}{2}\|y_i-f(\theta;x_i)\|^2$. We consider minimizing $\cL(\theta)$ by gradient descent
    $$\theta^{k+1}=\theta^k-\eta\nabla_\theta\cL(\theta^k),$$
where $\eta>0$, and $\theta^k$ is the vectorization of all parameters $\theta$ at the $k$th iteration.

\subsection{Model Assumptions}
\label{subsec:model_assumptions}

In this section, we introduce the regularity and nonlinear assumptions used by the main theorem.
\begin{assumption}[Real-Analyticity]
\label{assump:analytic}
    Assume ${\varphi}_\ell(\theta;u)$ are real-analytic (Definition~\ref{def:analytic}) for all $\theta$, $u$, and $\ell=1,\cdots,L$, and ${\varphi}_0(\theta;u)$ is real-analytic for all $\theta$, $u$ in some bounded open neighborhood of $\bar{\Omega}$. 
\end{assumption}
\begin{assumption}[Poly-Boundedness, Continuity, Smoothness]
    \label{assump:generalized_smoothness_detail}
    Assume for all $\ell=0,\cdots,L$, $\varphi_\ell(\theta;u)$ satisfies polynomial generalized continuity (Definition~\ref{def:poly_generalized_continuous}) and smoothness (Definition~\ref{def:poly_generalized_smoothness}) for both $\theta$ and $u$. In addition, $\varphi_\ell(\theta;u),\nabla_\theta\varphi_\ell(\theta;u),\nabla_u\varphi_\ell(\theta;u)$ are polynomially bounded (Definition~\ref{def:poly_bound}).
\end{assumption}
Assumption~\ref{assump:analytic} and~\ref{assump:generalized_smoothness_detail} hold true for many practical structures, including transformers, feedforward networks, and residual networks. See more discussions in Section~\ref{subsec:analytic},~\ref{subsec:poly_bound_conti_smooth}, and~\ref{subsec:model_family_examples}.

\subsubsection{Nonlinear Architecture and a Concrete Example}
\label{subsec:architecture_example}
We next introduce the architectural nonlinear condition.
Let $\tilde{f}_\ell(\theta;u_{\ell,i})=f(\theta;x_i)$ and it denotes the output map obtained by starting from the hidden state $u$ at layer $\ell$ and composing the remaining layers up to the output layer.
\begin{assumption}[Nonlinear Architecture]
\label{assump:architecture} 
Assume there exists some $\bar{\ell}\in\{0,\cdots,L-1\}$, s.t., for any $\theta_{\bar{\ell}},\cdots,\theta_L$ except for a measure-zero set, there exists one tuple $(u_1,\cdots,u_N)$, s.t.,
$\begin{pmatrix}
        \nabla_{\theta_{\bar{\ell}}}\tilde{f}_{\bar{\ell}}(\theta;u_1)\\
        \vdots\\
        \nabla_{\theta_{\bar{\ell}}}\tilde{f}_{\bar{\ell}}(\theta;u_N)
    \end{pmatrix}$
has full row rank $Nd$, and $\dim \theta_{\bar{\ell}}>Nd$. 
\end{assumption}
The above assumption is an architectural nondegeneracy condition that captures the nonlinearity of the network. It requires that, for some layer $\bar{\ell}$, the sample-wise Jacobians with respect to $\theta_{\bar{\ell}}$ can be stacked to form a full-row-rank matrix. Together with analyticity and Theorem~\ref{thm:zero_set_analytic_function} \citep{mityagin2015zero}, this local rank condition implies that the network is generically nondegenerate outside a measure-zero set. This is the only width requirement among all assumptions. Similar weak width requirements have appeared in existing works \citep[etc.]{nguyen2021tight}, although not directly in the convergence-dynamics setting considered here.

The following lemma gives a concrete example that satisfies Assumption~\ref{assump:architecture}. 
\begin{lemma}[Example of Assumption~\ref{assump:architecture}]
\label{lem:example_exist_Nd_full_rank_bias}
    The following structure satisfies Assumption~\ref{assump:architecture} with $\bar\ell=L-1$:
    $$
        \varphi_L(\theta;u)=\theta_Lu,\qquad
        \varphi_{L-1}(\theta;u)=\theta_{L-1,2}\sigma(\theta_{L-1,1}u+b_{L-1,1})+b_{L-1,2},
    $$
    where $\sigma(\cdot)$ is GELU, $\theta_L\in\RR^{d\times d}$, $\theta_{L-1,1}\in\RR^{m_{L-1}\times d}$, $\theta_{L-1,2}\in\RR^{d\times m_{L-1}}$, $b_{L-1,1}\in\RR^{m_{L-1}}$, $b_{L-1,2}\in\RR^d$, and $m_{L-1}>Nd$.
\end{lemma}
The above lemma deals with the case where the $(L-1)$th layer is a biased two-layer feedforward block and the final layer is linear. It shows that the width condition $m_{L-1}>Nd$ is sufficient for Assumption~\ref{assump:architecture}. This sufficient lower bound may not be sharp; nevertheless, requiring only the $(L-1)$th layer width to exceed $Nd$ is much weaker than typical NTK-type overparameterization requirements (see Section~\ref{subsec:convergence_beyond_NTK}). The proof of Lemma~\ref{lem:example_exist_Nd_full_rank_bias} is in Appendix~\ref{subapp:proof_example_full_rank_Nd}; the bias-free version is also given there.

\subsection{Examples of the Model Family}
\label{subsec:model_family_examples}

The neural network family~\eqref{eqn:NN} is broad enough to include many standard architectures built from common blocks. In this subsection, we first list representative feedforward, convolutional, attention, and normalization blocks that satisfy Assumptions~\ref{assump:analytic} and~\ref{assump:generalized_smoothness_detail}. We then present a pre-normalized multi-layer transformer as a complete architecture obtained by composing these blocks and show that it satisfies all model assumptions in Section~\ref{subsec:model_assumptions}.


\begin{example}[Blocks]
    \label{ex:blocks}
Representative blocks included in the neural network family~\eqref{eqn:NN} are:

\begin{enumerate}
    \item \textbf{Feedforward Blocks.}
Let $\sigma$ be one of the activation functions listed in Example~\ref{ex:activation_normalization}, such as $\tanh$, sigmoid, GELU, SiLU, or a polynomial activation. A two-layer feedforward block has the form
\begin{align}
\label{eqn:ex_block_ff}
    \varphi_\ell(\theta_\ell;u)
    &=\theta_{\ell,2}\sigma(\theta_{\ell,1}u+b_{\ell,1})+b_{\ell,2}.
\end{align}

    \item \textbf{Convolutional Blocks.}
Let $U$ be a feature tensor and let $\ast$ denote the convolution operator. A convolutional block with activation in Example~\ref{ex:activation_normalization} has the form
\begin{align}
\label{eqn:ex_block_conv}
    \varphi_\ell(\theta_\ell;U)
    &=\theta_{\ell,2}\ast\sigma(\theta_{\ell,1}\ast U+b_{\ell,1})+b_{\ell,2}.
\end{align}

    \item \textbf{Attention Blocks.}
Let $U\in\RR^{n\times d_h}$ be a token matrix. For layer $\ell$ and head $h=1,\cdots,H$, define
\begin{align}
\label{eqn:ex_block_att}
    \mathrm{Attn}_{\ell,h}(U)
    &=\mathrm{softmax}\left(\frac{U\theta_{\ell,h}^Q(U\theta_{\ell,h}^K)^\top}{\sqrt{d_k}}\right)U\theta_{\ell,h}^V.
\end{align}
Then, the multi-head attention block has the form
\begin{align}
\label{eqn:ex_block_mha}
    \varphi_\ell(\theta_\ell;U)
    &=\mathrm{MHA}_\ell(\theta_\ell;U)=\big[\mathrm{Attn}_{\ell,1}(U),\cdots,\mathrm{Attn}_{\ell,H}(U)\big]\theta_\ell^O.
\end{align}

    \item \textbf{The Above Blocks with Pre-normalization.}
Normalizations also fit the assumptions. More precisely, the normalized blocks have the form 
$$ \varphi_\ell(\theta_\ell;\text{Norm}(u))  \text{ and }\varphi_\ell(\theta_\ell;\text{Norm}(U)) $$
where $\varphi_\ell(\theta;\cdot)$ is the feedforward, convolution, or attention block defined above, and $\text{Norm}(u)$ is any of the normalizations in Example~\ref{ex:activation_normalization}.

\end{enumerate}
\end{example}

Architectures that are compositions of the above blocks satisfy all the model assumptions in Section~\ref{subsec:model_assumptions}. In particular, this includes the pre-normalized multi-layer transformer:
{
\begin{lemma}[Pre-normalized Multi-Layer Transformer]
\label{lem:prenorm_transformer_model_assumptions}
Consider the following transformer architecture:
\begin{align*}
    V_\ell&=U_\ell+\mathrm{MHA}_\ell(\mathrm{LN}_{\ell,1}(U_\ell)),\\
    U_{\ell+1}&=V_\ell+\mathrm{FF}_\ell(\mathrm{LN}_{\ell,2}(V_\ell)),
    \qquad \ell=0,\cdots,L-1,\\
    f(\theta;x)&=\theta_L U_L.
\end{align*}
where $\mathrm{MHA}_\ell$ is the multi-head attention block in~\eqref{eqn:ex_block_mha}, and $\mathrm{FF}_\ell$ is the biased two-layer feedforward block in~\eqref{eqn:ex_block_ff}. If the last feedforward block has hidden width $m_{L-1}>Nd$, then the pre-normalization transformer satisfies Assumptions~\ref{assump:analytic},~\ref{assump:generalized_smoothness_detail}, and~\ref{assump:architecture}.
\end{lemma}
\noindent The proof of Lemma~\ref{lem:prenorm_transformer_model_assumptions} is in Appendix~\ref{app:example_model_assumptions}.
}

\subsection{Main Convergence Theorem}
\label{subsec:main_convergence_theorem}

Below is the main theorem showing convergence of the $\ell^2$ loss for neural networks~\eqref{eqn:NN} under GD.
\begin{theorem}
\label{thm:training}
Suppose GD is applied to the neural network~\eqref{eqn:NN} under the following requirements, with probability 1 over the joint distribution $\pi$ of the input data $x_1,\cdots,x_N$:
\begin{enumerate}
    \item \textbf{Data and Model Assumptions.} Assumptions~\ref{assump:data_distribution_absolutely_continuous_Lebesgue_ground_truth},~\ref{assump:analytic},~\ref{assump:generalized_smoothness_detail}, and~\ref{assump:architecture} hold.
    \item \textbf{Initialization.} The initialization $\theta^0$ belongs to $\RR^{\dim \theta}$ except for a measure-zero set $\cB_\theta$.
    \item \textbf{Local Dissipative Condition.} The loss $\cL(\theta)$ satisfies the $(\theta^0,\theta^*,R,\rho,\epsilon_\cL)$-dissipative condition in Definition~\ref{def:dissipative}, where $\theta^*$ is a stationary point, $0<\delta<2$, $\rho\in\RR$, $\epsilon_\cL\ge 0$, and
    $$
        R=\sqrt{\|\theta^{0}-\theta^{*}\|^2+\max\left\{\frac{4\rho+2}{\delta},0\right\}\cL(\theta^0)}.
    $$
    \item \textbf{Learning Rate.} The learning rate $\eta$ is chosen from the interval
    $$
        0<\eta\le \min\left\{\frac{2-\delta}{\mathsf{L}},1\right\},
    $$
    excluding at most finitely many exceptional values,
    where
    $$
        \mathsf{L}=S\left(\cdots,\|\theta_i^*\|+R+\max_{\theta\in B_{\theta^*}(R)}\|\nabla_{\theta_i}\cL(\theta)\|,\cdots\right),
    $$
    and $S(\cdot)$ is a polynomial whose coefficients are all positive.
\end{enumerate}
Then, under an arbitrarily small adjustment to the scale of $\varphi_\ell$ for $\ell=0,\cdots,L-1$, GD converges to the region $\|\nabla\cL(\theta)\|\le \epsilon_\cL$. Moreover, for all $k\ge 1$ such that $\|\nabla\cL(\theta^k)\|>\epsilon_\cL$,
$$
    \cL(\theta^k)\le \prod_{s=0}^{k-1}\left(1-\eta\left(1-\frac{\eta\,\mathsf{L}}{2}\right)\frac{2\mu_{\mathrm{low},s,X}}{N}\right)\cL(\theta^0),
$$
where $\mu_{\mathrm{low},s,X}>0$ depends on $\theta^s$ and $x_1,\cdots,x_N$.
\end{theorem}

In the above theorem, the loss along GD is shown to decay until it converges to the neighborhood of a stationary point, under practical assumptions. Theorem~\ref{thm:training} is proved by combining analytic nondegeneracy, measure-zero arguments, dynamical systems, and optimization; its proof sketch is shown in Figure~\ref{fig:proof_sketch_convergence}.
Detailed proofs are in Appendix~\ref{app:optimization}. Another version of the above theorem for finite-step decay without the dissipative condition can be found in Theorem~\ref{thm:complete_convergence}.

The quantity $\mathsf{L}$ is the local smoothness constant along the GD trajectory, and is proved, rather than assumed, to be bounded. It depends on the poly-smoothness function $S(\cdot)$ of $\cL(\theta)$ (Lemma~\ref{lem:loss_function_generalized_smoothness}), the reference stationary point $\theta^*$, the radius $R$, and the maximum of $\nabla\cL(\theta)$ over the local ball $B_{\theta^*}(R)$. For a concrete architecture, this poly-smoothness function can be computed explicitly from the layer-wise polynomial bounds; see Appendix~\ref{subapp:poly_smoothness_loss}. $\theta^*$ can be any stationary point that satisfies the dissipative condition (Definition~\ref{def:dissipative}), and is not necessarily the limit of GD iterates. The radius $R$ of the bounded region for the local relaxed dissipative condition is determined by the initialization $\theta^0$, the initial loss $\cL(\theta^0)$, and the dissipativity parameter $\rho$. The constant $\mu_{\mathrm{low},s,X}$ is the lower frame bound of the derivatives of $(f(x_1)^\top,\cdots,f(x_N)^\top)$ 
which are shown to form a frame with probability 1 for almost all $\theta$, and consequently $\mu_{\mathrm{low},s,X}>0$ (see Section~\ref{subsec:mu}).

The learning rate condition has two components. The quantitative component is the standard descent requirement
$
    0<\eta\le \min\left\{\frac{2-\delta}{\mathsf L},1\right\},
$
where the scale is governed by the regular order $\mathsf L^{-1}$. The generic component excludes at most finitely many exceptional values of $\eta$ from this interval. These exceptional values come from the analytic nondegeneracy argument in Lemma~\ref{lem:full_rank_jacobian}: at such values, the GD update map can become degenerate and have a non-full-rank Jacobian, leading to stagnation of the dynamics. Since only finitely many values are excluded, the set is still essentially the full interval for practical purposes, and an arbitrarily small perturbation of $\eta$ avoids the exceptional set. The small adjustment to the scale of $\varphi_\ell$ plays the same role and only needs to be made once before the GD iterations.


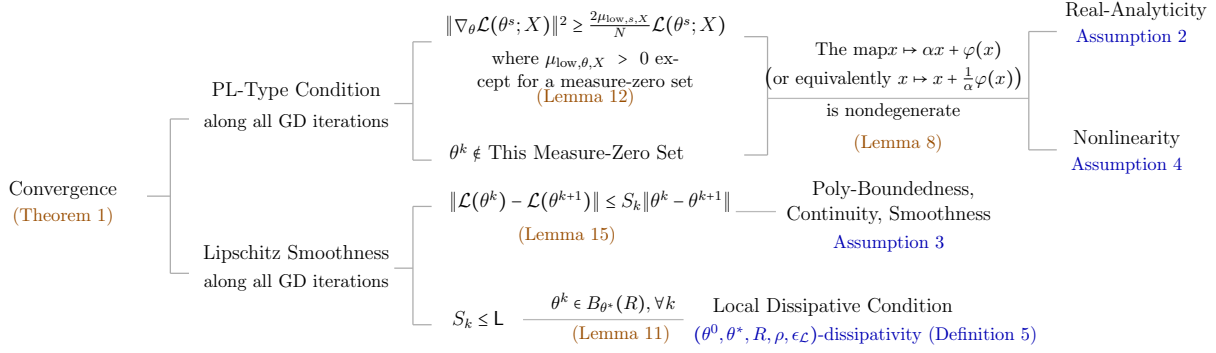
\begin{figure}[tb!]
\centering
\resizebox{0.96\linewidth}{!}{%
\begin{tikzpicture}[
    x=0.78cm,y=0.86cm,
    every node/.style={
        font=\normalsize,
        inner xsep=0cm,
        inner ysep=0cm,
        rounded corners=0cm,
        align=center
    },
    title/.style={font=\fontsize{12pt}{8pt}\selectfont},
    theorem/.style={text=orange!60!black},
    bracket/.style={draw=gray!65,line width=0.75pt},
    connector/.style={draw=gray!55,line width=0.75pt},
    compact/.style={font=\normalsize},
    smallref/.style={font=\normalsize,text=orange!60!black}
]
\node[title] at (3.75,11.65) {Convergence};
\node[theorem, compact] at (3.75,11) {(Theorem~\ref{thm:training})};
\draw[bracket] (6.625,13.375) -- (6.625,9.625);
\draw[bracket] (6,11.5) -- (6.625,11.5);
\draw[bracket] (6.625,13.375) -- (7.125,13.375);
\draw[bracket] (6.625,9.625) -- (7.125,9.625);

\node[title] at (10,14.05) {PL-Type Condition};
\node[compact] at (10,13.30) {along all GD iterations};
\draw[bracket] (13.125,15.25) -- (13.125,12.375);
\draw[bracket] (12.625,13.75) -- (13.125,13.75);
\draw[bracket] (13.125,15.25) -- (13.625,15.25);
\draw[bracket] (13.125,12.375) -- (13.625,12.375);

\node[title] at (17.75,15.675) {$\|\nabla_\theta\mathcal{L}(\theta^s;X)\|^2\ge\frac{2\mu_{{\rm low},s,X}}{N}\mathcal{L}(\theta^s;X)$};
\node[compact, text width=6.4cm] at (17.75,14.55) {where $\mu_{{\rm low},\theta,X}>0$ except for a measure-zero set};
\node[smallref] at (17.75,13.90) {(Lemma~\ref{lem:gradient_lower_bound})};
\node[title] at (17.25,12.55) {$\theta^k\notin$ This Measure-Zero Set};
\draw[bracket] (22.625,15.375) -- (22.625,12.5);
\draw[bracket] (22,12.5) -- (22.625,12.5);
\draw[bracket] (22,15.375) -- (22.625,15.375);
\draw[connector] (22.625,13.875) -- (29.625,13.875);

\node[compact] at (24.875,15.05) {The map};
\node at (27.375,15.05) {$x\mapsto \alpha x+\varphi(x)$};
\node[compact, text width=5.6cm] at (26,14.35) {$\left(\text{or equivalently }x\mapsto x+ \frac{1}{\alpha}\varphi(x)\right)$};
\node at (26,13.55) {is nondegenerate};
\node[smallref, text width=6.1cm] at (26.125,12.78) {(Lemma~\ref{lem:full_rank_jacobian_mainbody})};
\draw[bracket] (29.625,15.5) -- (30.375,15.5);
\draw[bracket] (29.625,15.5) -- (29.625,12.625);
\draw[bracket] (29.625,12.625) -- (30.375,12.625);

\node[title] at (32.5,16.00) {Real-Analyticity};
\node[blue!70!black, compact] at (32.375,15.35) {Assumption~\ref{assump:analytic}};
\node[title] at (32.25,12.90) {Nonlinearity};
\node[blue!70!black, compact] at (32.25,12.25) {Assumption~\ref{assump:architecture}};

\node[title] at (10,10.10) {Lipschitz Smoothness};
\node[compact] at (10,9.40) {along all GD iterations};
\draw[bracket] (13.125,11.25) -- (13.125,8.375);
\draw[bracket] (12.625,9.75) -- (13.125,9.75);
\draw[bracket] (13.125,11.25) -- (13.625,11.25);
\draw[bracket] (13.125,8.375) -- (13.625,8.375);

\node[title] at (17.875,11.375) {$\|\mathcal{L}(\theta^k)-\mathcal{L}(\theta^{k+1})\|\le S_k\|\theta^k-\theta^{k+1}\|$};
\node[smallref] at (17.25,10.55) {(Lemma~\ref{lem:loss_function_generalized_smoothness})};
\node[title] at (14.875,8.5) {$S_k\le \mathsf{L}$};
\draw[connector] (16.075,8.5) -- (20.375,8.5);
\node at (18.5,8.9) {$\theta^k\in B_{\theta^*}(R), \forall k$};
\node[smallref] at (18.7,8.15) {(Lemma~\ref{lem:bound_S_k_theta_k})};
\node[title] at (24.375,8.80) {Local Dissipative Condition};
\node[blue!70!black, compact] at (25.25,8.15) {$(\theta^0,\theta^*,R,\rho,\epsilon_{\mathcal{L}})$-dissipativity (Definition~\ref{def:dissipative})};
\draw[connector] (21.775,11.125) -- (22.85,11.125);
\node[title] at (26,11.65) {Poly-Boundedness,};
\node[title] at (25.875,11.05) {Continuity, Smoothness};
\node[blue!70!black, compact] at (25.875,10.35) {Assumption~\ref{assump:generalized_smoothness_detail}};
\end{tikzpicture}%
}
\caption{Proof sketch of the convergence framework. Assumptions are highlighted in blue, and key lemmas in brown.
}
\label{fig:proof_sketch_convergence}
\end{figure}

\section{Interpretation of the Convergence Results}
\label{sec:interpretation}

We now discuss several consequences and interpretations of Theorem~\ref{thm:training}. The theorem contains three types of quantities: the trajectory bound determined by the local dissipative condition, the lower frame bound $\mu_{{\rm low},s,X}$ that controls the descent rate, and the local smoothness constant $\mathsf L$ that determines the admissible learning rate scale. The following subsections explain how these quantities should be read in practical training settings, and how the result differs from convergence analyses in the NTK regime.

\subsection{Bounded GD Dynamics}

The first implication of Theorem~\ref{thm:training} is that the GD trajectory does not escape to infinity. This is not assumed a priori. Instead, the local dissipative condition controls the direction of the gradient relative to a reference stationary point $\theta^*$, while the descent estimate controls the accumulated squared gradient norm. Together these two facts imply that every iterate remains in the same local ball where the smoothness constant $\mathsf L$ is evaluated.

\begin{lemma}[Bounded GD Trajectory]
\label{lem:bounded_gd_trajectory}
Under the assumptions of Theorem~\ref{thm:training}, 
the GD trajectory is bounded in the sense that
$$
    \theta^k\in B_{\theta^*}(R),\qquad \forall k\ge0
$$
as long as $\|\nabla\cL(\theta^k)\|>\epsilon_\cL$. 
\end{lemma}

The radius $R$ depends only on the initialization, the reference stationary point, the initial loss, and the dissipativity parameter. Therefore the local smoothness constant in Theorem~\ref{thm:training} is evaluated on a region that is determined before GD iterates. The detailed proof of this boundedness statement is contained in Lemma~\ref{lem:bound_S_k_theta_k} in Appendix~\ref{app:optimization}.

\subsection{The Nondegeneracy of $\mu_{{\rm low},s,X}$ and Nonstagnation}
\label{subsec:mu}

The convergence rate in Theorem~\ref{thm:training} depends on a data- and iterate-dependent quantity $\mu_{{\rm low},s,X}$. This quantity measures the nondegeneracy of the stacked network Jacobian at the current iterate: when it is positive, the loss value controls the gradient norm through a PL-type inequality along the GD trajectory. Thus $\mu_{{\rm low},s,X}$ is the mechanism that turns the generic full-rank property of the model into an effective descent estimate.

We now make this nondegeneracy quantity explicit for a fixed parameter $\theta$ and dataset $X=(x_1,\cdots,x_N)$. We define
$$
   \mu_{{\rm low},\theta,X}:=\lambda_{\min} \left[
    \begin{pmatrix}
        \nabla_\theta f(\theta;x_1)\\
        \vdots\\
        \nabla_\theta f(\theta;x_N)
    \end{pmatrix}\begin{pmatrix}
        \nabla_\theta f(\theta;x_1)\\
        \vdots\\
        \nabla_\theta f(\theta;x_N)
    \end{pmatrix}^\top\right].
$$
Thus along the GD trajectory, $\mu_{{\rm low},s,X}:=\mu_{{\rm low},\theta^s,X}$. The next lemma shows that this quantity is positive generically and remains positive along the GD trajectory.

\begin{lemma}[About $\mu_{{\rm low},\theta,X}$ and $\mu_{{\rm low},s,X}$]
\label{lem:mu_low_along_gd}
Under Assumption~\ref{assump:architecture} and some small adjustment of the scale of $\varphi_\ell$ for $\ell=0,\cdots,L-1$, for fixed $\theta$ except for a measure-zero set in $\RR^{\dim\theta}$, we have
$$
    \mu_{{\rm low},\theta,X}>0
$$
except for a measure-zero set of $X=(x_1,\cdots,x_N)$ in $\RR^{Nd}$. Moreover, under the same assumptions as Theorem~\ref{thm:training}, for every GD iterate $\theta^s$,
$$
    \mu_{{\rm low},s,X}>0.
$$
In particular,
$$
    \|\nabla_\theta\cL(\theta^s;X)\|^2
    \ge
    \frac{2\mu_{{\rm low},s,X}}{N}\cL(\theta^s;X).
$$
\end{lemma}

The first part of the lemma is independent of the GD dynamics: after excluding a measure-zero set of parameters, degeneracy can occur only for a measure-zero set of input. The second part applies this statement to the countable sequence of GD iterates that avoids the measure-zero set, so the lower bound remains valid at every step of the trajectory. The final inequality is the PL-type inequality where this nondegeneracy enters the optimization argument.

If $\mu_{{\rm low},s,X}$ is uniformly bounded below by a positive constant, the above inequality is a PL inequality and yields exponential loss decay. For general neural networks, such a uniform lower bound may not hold; $\mu_{{\rm low},s,X}$ may decay along training, leading to a polynomial rate of convergence (see e.g. \citep{xu2023over}). Lemma~\ref{lem:mu_low_along_gd} and Theorem~\ref{thm:training} cover both the two regimes discussed above. The condition $\mu_{{\rm low},s,X}>0$ should be viewed as a minimal nondegeneracy requirement along the trajectory: it keeps the descent bound informative, whereas $\mu_{{\rm low},s,X}=0$ allows stagnation in the corresponding region. The proof of Lemma~\ref{lem:mu_low_along_gd} is in Appendix~\ref{subapp:lower_bound_gradient}.

\subsection{Scaling of Learning Rate $\eta$ and Related Parameters $R,\rho,\mathsf{L}$}
\label{subsec:learning_rate}

We now interpret the scale of the learning rate bound in Theorem~\ref{thm:training}. The exceptional learning rate values have already been discussed under Theorem~\ref{thm:training}; here we focus on the quantitative part, which is governed by the inverse local smoothness constant $\mathsf{L}^{-1}$. Since $\mathsf{L}$ is evaluated on the bounded region determined by $R$ and the dissipative condition, its order depends on the initialization scale, the reference stationary point, and the effective dimensions of the architecture. The goal of this subsection is to estimate the orders of $R$, $\rho$, $\mathsf{L}$, and the learning rate scale under a practical initialization and architecture setting.

We first introduce notation for the effective dimensions that enter the scaling estimates. Let $L_{\rm wm}$ be the number of trainable weight matrices in the architecture. For each weight matrix $\theta_j\in \RR^{d_{j,\mathrm{out}}\times d_{j,\mathrm{in}}}$ for $j=1,\cdots,L_{\rm wm}$, let
$$
    r_j:=\min\{d_{j,\mathrm{in}},d_{j,\mathrm{out}}\},\text{ and }r=\max_j r_j.
$$

These quantities allow us to state the following practical setting in terms of depth and bottleneck dimensions rather than the largest width.
\begin{assumption}[Practical Initialization and Architecture Scaling.]
    \label{assump:example_setting}
    Consider the following setting:
\begin{enumerate}
    \item \textbf{Initialization.} All weight matrices use Xavier-normal initialization~\citep{glorot2010understanding}:
    $$
        (\theta_j^0)_{ab}\sim \mathcal N\left(0,\frac{2}{d_{j,\mathrm{in}}+d_{j,\mathrm{out}}}\right),
        \text{ i.i.d. over }a,b,j.
    $$

    \item \textbf{Weight Matrices.} Assume the number of trainable weight matrices in the architecture
    $$
        L_{\rm wm}=c_{\rm wm}L=\Theta(L)
    $$
    for an architecture-dependent constant $c_{\rm wm}>0$.

    \item \textbf{Model Dimension.} Assume that for some constant $c_{\rm dim}>0$,
    $$
        c_{\rm dim}rL_{\rm wm}\le \sum_{j=1}^{L_{\rm wm}}r_j\le rL_{\rm wm}.
    $$
    \item \textbf{Architecture.} Assume for some constant $C_0,C_{\rm grad},p_{\rm grad}>0$,
    \begin{align*}
        &\max_{\|\theta\|=\cO(\sqrt{rL})}\cL(\theta)\le C_0,\\
        &\max_{\|\theta\|=\cO(\sqrt{rL})}\|\nabla\cL(\theta)\|\le C_{\rm grad}(1+\sqrt{rL})^{p_{\rm grad}}.
    \end{align*}

    \item \textbf{Dissipative Condition $(\theta^*,\epsilon_\cL)$.} Assume we can find some stationary point $\theta^*$ in the local dissipative condition such that
    $$
        \|\theta^*\|\le c_*\|\theta^0\|,
    $$
    for some constants $c_*\ge0$. Assume for some constants $c_\epsilon> 0$,
    $$
        \epsilon_\cL\ge c_\epsilon \sqrt{rL}^{-1}.
    $$

    \item \textbf{Learning Rate.} Assume $\delta=\Theta(1)$.
\end{enumerate}
    
\end{assumption}

The above assumptions describe a practical scaling regime for interpreting Theorem~\ref{thm:training}; they are not additional requirements of the convergence theorem itself. First, Xavier-normal initialization is a standard default initialization for deep networks, and its variance is chosen to keep the scale of activations stable across layers \citep{glorot2010understanding}. Second, the condition $L_{\rm wm}=\Theta(L)$ simply says that each layer contains a constant number of trainable weight matrices, which is the case for feedforward, convolutional, residual, and transformer-type architectures with fixed block structure. Third, the parameter $r$ summarizes the effective bottleneck dimension of the trainable matrices: the condition on $\sum_{j=1}^{L_{\rm wm}}r_j$ allows different matrices to have different input and output dimensions, while requiring only that their minimum dimensions are of comparable order on average. 

The architecture bounds are local bounds on the loss and gradient over the parameter scale naturally produced by the initialization. They are consistent with the polynomial boundedness and polynomial smoothness assumptions: on a bounded parameter region and bounded data domain, the loss and its gradient grow at most polynomially. The dissipative condition is also local. The condition $\|\theta^*\|\le c_*\|\theta^0\|$ requires that the reference stationary point used in the dissipativity argument lies at a comparable parameter scale to the initialization. The parameter $\epsilon_\cL$ is the gradient-norm threshold in the dissipative condition and can be chosen freely; here, imposing a lower bound on $\epsilon_\cL$ helps prevent the dissipativity constant $\rho$ from becoming too large, especially with respect to width. Finally, taking $\delta=\Theta(1)$ fixes a constant safety margin in the descent condition and does not introduce additional dependence on width or depth.

Under this practical scaling regime, the quantities appearing in Theorem~\ref{thm:training} can be estimated explicitly. The next corollary shows that the bounded trajectory radius $R$, the dissipativity parameter $\rho$, the local smoothness constant $\mathsf{L}$, and the admissible learning rate scale $(2-\delta)/\mathsf{L}$ depend on the depth $L$ and effective bottleneck dimension $r$.

\begin{corollary}[Scale of Learning Rate $\eta$ and Related Parameters $R,\rho,\mathsf{L}$]
\label{cor:Gaussian_learning_rate_order}
Under the assumption of Theorem~\ref{thm:training} and Assumption~\ref{assump:example_setting}, there exists a universal constant $c_{\rm G}>0$ such that, with probability at least $1-2\exp(-c_{\rm G}c_{\rm dim}c_{\rm wm}rL)$ over the random initialization, we have
\begin{align*}
    & \text{radius of the dynamics: } R=\cO(\sqrt{rL}),\\
    & \text{dissipativity constant: } \rho=\cO(rL),\\
    & \text{local Lipschitz smoothness constant: } \mathsf{L}=\cO\left((1+\sqrt{rL})^{q_{\mathsf{L}}}\right),\\
    & \text{learning rate scale: } \frac{2-\delta}{\mathsf{L}}=\Omega\left(\frac{1}{(1+\sqrt{rL})^{q_{\mathsf{L}}}}\right),
\end{align*}
for some constant $q_{\mathsf{L}}\ge 0$.
\end{corollary}

Here $r$ should be understood as an effective bottleneck dimension rather than the ambient hidden width. In many wide architectures, each trainable matrix either has one dimension comparable to the data or feature dimension, or has a smaller bottleneck, head, or channel dimension. In such cases,
$$r=\max\{ d,\text{some bottleneck dimension}\}\ll m:=\max\{ d_{1,\mathrm{in}},d_{1,\mathrm{out}},\cdots,d_{L_{\rm wm},\mathrm{in}},d_{L_{\rm wm},\mathrm{out}} \}.$$
Thus the orders in Corollary~\ref{cor:Gaussian_learning_rate_order} depend on the depth $L$ and the effective dimension $r$, instead of the largest hidden width $m$ (see comparisons with NTK analyses in Section~\ref{subsec:convergence_beyond_NTK}).

The same interpretation applies to Xavier-uniform initialization in the following corollary. 

\begin{corollary}
\label{cor:uniform_learning_rate_order}

Consider the same assumption as Corollary~\ref{cor:Gaussian_learning_rate_order} except for the initialization replaced by Xavier-uniform initialization:
$$
    (\theta_j^0)_{ab}\sim \mathrm{Unif}\left[-\sqrt{\frac{6}{d_{j,\mathrm{in}}+d_{j,\mathrm{out}}}},\sqrt{\frac{6}{d_{j,\mathrm{in}}+d_{j,\mathrm{out}}}}\right],\text{ i.i.d. over }a,b,j.
$$
 In this case, with probability at least $1-2\exp(-c_{\rm dim}^2c_{\rm wm}r^2L/18)$ over the random initialization, we have
the same conclusions as Corollary~\ref{cor:Gaussian_learning_rate_order}.
\end{corollary}


The proofs of Corollary~\ref{cor:Gaussian_learning_rate_order} and~\ref{cor:uniform_learning_rate_order} are in Appendix~\ref{app:interpretation_convergence_theorem}.

\section{Implications of the Convergence Theory}
\label{sec:implications_convergence_theory}

We next discuss several consequences of the convergence framework beyond the statement of Theorem~\ref{thm:training}. First, the analytic measure-zero arguments can be interpreted structurally through residual connections and function composition, and thereby provide theoretical support for using residual connections and function compositions (Section~\ref{subsec:insights_residual_connection_composition_normalization}). Second, we clarify how the analysis differs from the NTK regime in terms of width, learning rate, and global minimization (Section~\ref{subsec:convergence_beyond_NTK}). Finally, we discuss a landscape implication: outside the generic bad set, stationary points are global minimizers of the empirical square loss (Section~\ref{subsec:global_min}).

\subsection{Nondegeneracy in Residual and Compositional Architectures}
\label{subsec:insights_residual_connection_composition_normalization}

When constructing a neural network, it is important to ensure that the model has sufficient expressivity. Theoretically, this corresponds to requiring the neural network mapping $f$ to be nondegenerate; that is, its Jacobian $\nabla f$ should be full rank on some region. Otherwise, the mapping may collapse onto a lower-dimensional subspace and thus lose full expressivity.

In fact, nondegeneracy is a generic property among real-analytic functions. If a function is degenerate, we can typically restore nondegeneracy by adding a small identity component, equivalently, by incorporating a residual connection, so that the modified mapping becomes nondegenerate, as stated in the following lemma.

    \begin{lemma}
\label{lem:full_rank_jacobian_mainbody}
     Let $f(x,y):\RR^{m}\times \RR^{n_2}\to \RR^m$ be an analytic function. Fix either $y=y_0$ or $x=x_0$. Then the following statements hold:
    \begin{enumerate}
        \item For any $\alpha\in\RR$ except for a finite set of points $E_{y_0}$ or $E_{x_0}$, the Jacobian with respect to $x$ of $\alpha x+f(x,y)$ is full rank for any free variable $x$ or $y$ except for a measure-zero set.
        \item Similarly, for any $\eta\in\RR$ except for a finite set of points $1/E_{y_0}\cup\{0\}$ or $1/E_{x_0}\cup\{0\}$, the Jacobian with respect to $x$ of $x+\eta f(x,y)$ is full rank for any free variable except for a measure-zero set.
    \end{enumerate} 
\end{lemma}

The above lemma partially explains the empirical success of various residual architectures in practice \citep[etc.]{zhu2025hyper,zhang2026deep}. The proof of Lemma~\ref{lem:full_rank_jacobian_mainbody} is given in Appendix~\ref{subapp:measure_theory_differential_topology}.

Moreover, neural networks are constructed as compositions of layer-wise transformations. If each layer has a full-rank Jacobian (e.g., ensured via residual connections), then the Jacobian of the overall composition remains full rank by the chain rule: $\nabla (f\circ g)(x)=\nabla f(g(x))\nabla g(x)$. This shows that the compositional structure of neural networks preserves nondegeneracy. 

\subsection{Comparison with the NTK Regime}
\label{subsec:convergence_beyond_NTK}
We now compare the convergence guarantee in Theorem~\ref{thm:training} with the standard NTK-type convergence picture. 
This leads to different interpretations of the required width, the learning rate, and the relation between stationarity and global optimality.

\textbf{Width.} In NTK analyses, convergence is usually obtained by taking all hidden layers sufficiently wide so that the training dynamics remains close to its linearization at initialization. In contrast, the only width requirement of Theorem~\ref{thm:training} is the nonlinear condition in Assumption~\ref{assump:architecture}. For example, Lemma~\ref{lem:example_exist_Nd_full_rank_bias} verifies this condition when one feedforward block near the output has width larger than $Nd$, while the other layers are not required. Thus our result is compatible with heterogeneous architectures who have both wide and narrow layers.

\textbf{Learning Rate.} NTK convergence typically relies on lazy dynamics, where the parameters move little and the learning rate is tied to the width-dependent linearization scale, usually inversely propotional to width and hence very small. Theorem~\ref{thm:training} instead, allows much larger regular learning rates: the interval is governed by the local smoothness constant $\mathsf L$ with only finitely many exceptional values being excluded. Under the Xavier initialization setting in Corollary~\ref{cor:Gaussian_learning_rate_order} and~\ref{cor:uniform_learning_rate_order}, the learning rate scale depends on the depth and the effective bottleneck dimension $r$, and need not scale with the width.

\textbf{Global Minimum.} One strength of many NTK results is global convergence of the training loss under strong overparameterization. Our theorem has a different scope: it proves convergence to the neighbourhood of a stationary point for nonlinear training dynamics beyond lazy training. Therefore, it does not claim that GD always reaches a global minimizer. Nevertheless, the global-minimum result in the following section explains what can still be said at the landscape level: exact stationary points outside the generic bad set are global minimizers.

\subsection{Global Minimum}
\label{subsec:global_min}


Next, we show that stationary points outside the measure-zero bad set in Theorem~\ref{thm:training} are global minimizers, inspired by \cite{nguyen2017loss}. 
\begin{theorem}[Global Minimum]
\label{thm:global_min_generic_nonlinearity}
Consider the same assumptions as Theorem~\ref{thm:training}. For any $\theta^*\notin$ the measure-zero set $\cB_\theta$ defined in Theorem~\ref{thm:training}, if $\theta^*$ is a stationary point of $\cL(\theta)$, then $\theta^*$ is a global minimizer of $\cL(\theta)$.
\end{theorem}

The theorem should be interpreted as a landscape statement rather than a direct optimization guarantee. It rules out spurious stationary points outside $\cB_\theta$: any such stationary point must fit all training samples exactly, and hence attains the global minimum value of the nonnegative square loss. Note that this does not imply that GD will always converge to a global minimizer. In particular, even if the GD trajectory does not enter $\cB_\theta$, it may still approach stationary points contained in $\cB_\theta$ from outside this set. The proof of Theorem~\ref{thm:global_min_generic_nonlinearity} is in Appendix~\ref{app:global_min_generic_nondegeneracy}.

\section{Discussions}

In this paper, we develop a convergence framework for GD on a broad family of neural networks beyond the NTK regime. The analysis combines mild and architecture-level ingredients: polynomial generalized smoothness, a local relaxed dissipative condition, and analytic nondegeneracy of the network Jacobian. These assumptions are formulated at the level of layer maps and network blocks, and are compatible with finite compositions. Under these conditions, we prove that GD with a regular learning rate converges to the neighbourhood of a stationary point, and we also interpret the scales of relevant smoothness and learning rate under practical initialization settings. This framework covers complicated model families built from common analytic blocks, including pre-normalized transformers, and gives structural explanations for the role of residual connections and composition in preserving nondegeneracy.

At the same time, this generality leaves several questions open. The current theory proves convergence to stationarity rather than global convergence, and focuses on deterministic full-batch GD with a fixed learning rate. Extending the result to global convergence would require sharper landscape conditions or stochasticity; meanwhile, stochastic, adaptive, or spectral optimizers such as SGD, Adam, and Muon, together with training strategies such as weight decay and learning rate schedules, introduce additional dynamics beyond the current analysis. 
We leave these extensions and sharper global convergence conditions for future work.

\clearpage
\appendix

\section*{Appendix}

\section{Preliminary Properties and Assumption Justification}
\label{app:preliminary_assumption}

\subsection{Analytic Function}
\label{subapp:analytic_function}
Real-analyticity is preserved under many operations. Specifically:
\begin{proposition}
\label{prop:analytic_function}
    Real-analytic functions have the following properties: 
    \begin{enumerate}
        \item The sums, products, divisions where the denominators are not zero, and compositions of real-analytic functions are real-analytic.
        \item The derivative and integral of real-analytic functions are real-analytic.
        \item Real-analytic function is $C^\infty$.
    \end{enumerate}
\end{proposition}

We then prove that many neural network architectures are real-analytic.
\begin{proof}[Proof of Lemma~\ref{cor:analytic_with_NN}]
    Since exponential function admits Taylor series in some neighbourhood of any point in $\RR^{d}$, and the denominators of softmax and sigmoid functions are strictly positive at any point, by Proposition~\ref{prop:analytic_function}, we have that softmax, sigmoid, GELU, SiLU are analytic functions on $\RR^d$. 
    
    Additionally, polynomial functions admit Taylor series in some neighbourhood of any point in $\RR^{d}$, and the denominator of $\tau_\epsilon$ is strictly positive at any point; therefore, $\tau_\epsilon$ and polynomial functions are analytic. 
    
    Similarly, layer norm and batch norm both admits strictly positive denominators that are analytic; thus each elements of the two normalizations are also analytic.
    
    Also, tanh can be Taylor expanded at any point and thus analytic. 
    
    By Proposition~\ref{prop:analytic_function}, any product, sum, compositions of the above functions are still analytic.
\end{proof}

\subsection{Polynomial Boundedness, Polynomial Generalized Continuity and Smoothness}
\label{subapp:poly_bound_continuity_smooth}
\begin{proof}[Proof of Proposition~\ref{prop:poly_generalized_properties}]
    Let $S_{\psi,0},\ S_{\tilde{\psi},0}$ be the corresponding polynomials in polynomial generalized continuity of $\psi$ and $\tilde{\psi}$, $S_{\psi,1},\ S_{\tilde{\psi},1}$ be the ones in polynomial generalized smoothness, $S_{\psi},\ S_{\tilde{\psi}}$ be the corresponding polynomials in polynomial bound, and $S_{\nabla \psi},\ S_{\nabla \tilde{\psi}}$ be the corresponding polynomials in polynomial bound of their gradients. For notational simplicity, we use $S(w)$ to denote $S(\|w_{1}\|,\cdots,\|w_{n_w}\|)$, and $S(w_{\max})$ to denote $S(\|w_{\max,1}\|,\cdots,\|w_{\max,n_w}\|)$.

    Then
    \begin{align*}
        \|\psi(w)+\tilde{\psi}(w)-\psi(w')-\tilde{\psi}(w')\|
        &\le \|\psi(w)-\psi(w')\|+\|\tilde{\psi}(w)-\tilde{\psi}(w')\|\\
        &\le (S_{\psi,0}+S_{\tilde{\psi},0})(w_{\max})\|w-w'\|.
    \end{align*}
    Similarly, $\nabla \psi+\nabla \tilde{\psi}$ follows the same idea with $(S_{\psi,1}+S_{\tilde{\psi},1})(w_{\max})$. Moreover, $\psi+\tilde{\psi}$ is upper bounded by $S_\psi(w_{\max})+S_{\tilde{\psi}}(w_{\max})$, and $\nabla \psi+\nabla \tilde{\psi}$ is upper bounded by $S_{\nabla \psi}(w_{\max})+S_{\nabla \tilde{\psi}}(w_{\max})$.

    Also,
    \begin{align*}
        \|\psi(w)\tilde{\psi}(w)-\psi(w')\tilde{\psi}(w')\|
        &=\|\psi(w)\tilde{\psi}(w)-\psi(w)\tilde{\psi}(w')+\psi(w)\tilde{\psi}(w')-\psi(w')\tilde{\psi}(w')\|\\
        &\le \|\psi(w)\|\|\tilde{\psi}(w)-\tilde{\psi}(w')\|+\|\psi(w)-\psi(w')\|\|\tilde{\psi}(w')\|\\
        &\le (S_\psi(w)S_{\tilde{\psi},0}(w_{\max})+S_{\psi,0}(w_{\max})S_{\tilde{\psi}}(w'))\|w-w'\|\\
        &\le (S_\psi(w_{\max})S_{\tilde{\psi},0}(w_{\max})+S_{\psi,0}(w_{\max})S_{\tilde{\psi}}(w_{\max}))\|w-w'\|.
    \end{align*}
    Obviously, $(S_\psi(w_{\max})S_{\tilde{\psi},0}(w_{\max})+S_{\psi,0}(w_{\max})S_{\tilde{\psi}}(w_{\max}))$ is a polynomial with positive coefficients, and therefore $\psi\tilde{\psi}$ satisfies polynomial generalized continuity. 
    
    Based on the above derivation, $\nabla (\psi\tilde{\psi})=\nabla \psi\cdot \tilde{\psi}+\psi\nabla \tilde{\psi}$ is a sum of two products, and thus satisfies polynomial continuity with $$(S_{\nabla \psi}S_{\tilde{\psi},0}+S_{\psi,1}S_{\tilde{\psi}}+S_{\nabla \tilde{\psi}}S_{\psi,0}+S_{\tilde{\psi},1}S_{\psi}).$$ 
    Since $\psi$ and $\tilde{\psi}$ may be matrix-valued, the above gradient identity should be interpreted after vectorization.
    
    Also, $\psi\tilde{\psi}$ is upper bounded by $S_\psi S_{\tilde{\psi}}$, and $\nabla (\psi\tilde{\psi})=\nabla \psi\cdot \tilde{\psi}+\psi\nabla \tilde{\psi}$ is upper bounded by $S_{\nabla \psi}S_{\tilde{\psi}}+S_{\nabla \tilde{\psi}}S_\psi$.

    Next, since all the $S$ have positive coefficients, we have
    \begin{align*}
        \|\psi(\tilde{\psi}(w))-\psi(\tilde{\psi}(w'))\|
        &\le S_{\psi,0}(\cdots,\|\tilde{\psi}(w)_{\max,i}\|,\cdots)\|\tilde{\psi}(w)-\tilde{\psi}(w')\|\\
        &\le S_{\psi,0}(S_{\tilde{\psi}}(w_{\max}),\cdots,S_{\tilde{\psi}}(w_{\max}))S_{\tilde{\psi},0}(w_{\max})\|w-w'\|.
    \end{align*}
    Also, $\nabla_w(\psi\circ\tilde{\psi})(w)=\nabla \psi(\tilde{\psi}(w))\nabla \tilde{\psi}(w)$. For the continuity of this Jacobian, we write
    \begin{align*}
        &\|\nabla_w(\psi\circ\tilde{\psi})(w)-\nabla_w(\psi\circ\tilde{\psi})(w')\|\\
        &\le \|\nabla\psi(\tilde{\psi}(w))-\nabla\psi(\tilde{\psi}(w'))\|\|\nabla\tilde{\psi}(w)\|
        +\|\nabla\psi(\tilde{\psi}(w'))\|\|\nabla\tilde{\psi}(w)-\nabla\tilde{\psi}(w')\|\\
        &\le \Big(S_{\psi,1}(S_{\tilde{\psi}}(w_{\max}),\cdots,S_{\tilde{\psi}}(w_{\max}))S_{\tilde{\psi},0}(w_{\max})S_{\nabla \tilde{\psi}}(w_{\max})\\
        &\qquad\quad+S_{\nabla \psi}(S_{\tilde{\psi}}(w_{\max}),\cdots,S_{\tilde{\psi}}(w_{\max}))S_{\tilde{\psi},1}(w_{\max})\Big)\|w-w'\|.
    \end{align*}
    Therefore, $\nabla_w(\psi\circ\tilde{\psi})(w)$ is polynomial continuous.

    In the end, it can be shown that $\psi(\tilde{\psi}(w))$ is upper bounded by $$S_\psi(S_{\tilde{\psi}}(w_{\max}),\cdots,S_{\tilde{\psi}}(w_{\max})),$$ and $\nabla(\psi\circ\tilde{\psi})(w)$ is upper bounded by $$S_{\nabla \tilde{\psi}}(w_{\max})S_{\nabla \psi}(S_{\tilde{\psi}}(w_{\max}),\cdots,S_{\tilde{\psi}}(w_{\max})).$$
\end{proof}

\begin{proof}[Proof of Lemma~\ref{cor:poly_smooth_with_NN}]
    By Proposition~\ref{prop:poly_generalized_properties}, we only need to show that polynomials, softmax, tanh, sigmoid, $\tau_\epsilon$, layer normalization, batch normalization, GELU, SiLU satisfy all the properties.

    First, polynomials obviously satisfy all the properties. It can be shown by simple calculations that softmax, tanh, sigmoid, $\tau_\epsilon$, and their derivatives are all upper bounded by constants, and therefore satisfy all the above properties.

    For layer normalization and batch normalization, the denominators are both lower bounded by some positive constant. Additionally by their definition, the numerators are polynomials. Thus by Proposition~\ref{prop:poly_generalized_properties}, the two normalizations satisfy all the properties.

    For SiLU, it is the product of $w$ and sigmoid, and therefore by Proposition~\ref{prop:poly_generalized_properties} satisfies all the above properties.

    For GELU, $$\sigma(w)=w\Phi(w),$$ where $\Phi(w)=\int_{-\infty}^w \frac{e^{-s^2/2}}{\sqrt{2\pi}}ds$. Also, we know that $\Phi'(w)=\phi(w)=\frac{e^{-w^2/2}}{\sqrt{2\pi}}$. Then $\Phi(w)$ and $\Phi'(w)$ are both upper bounded by constants, and thus satisfy all the above properties. By Proposition~\ref{prop:poly_generalized_properties}, GELU, the product of $w$ and $\Phi(w)$, also satisfies these properties.
\end{proof}

\begin{proof}[Proof of Lemma~\ref{lem:smoothness_inequalities}]
Since $\psi(w)\in C^1$ satisfies the polynomial generalized smoothness, we have
\begin{align*}
    \|\nabla \psi(w)-\nabla \psi((1-t)w+tw')\|&\le S(\cdots,\max\{\|w_i\|,\|(1-t)w_i+tw'_i\|,\cdots\})\|w-((1-t)w+tw')\|\\
    &\le t\,S(\cdots,\max\{\|w_i\|,\|w'_i\|\},\cdots)\|w-w'\|\\
    &=t\,  S(\|w_{\max,1}\|,\cdots,\|w_{\max,d}\|)\|w-w'\|
\end{align*}
Then
    \begin{align*}
        \psi(w')-\psi(w)&=\int_0^1 \nabla \psi((1-t)w+tw')^\top (w'-w)dt\\
        &=\int_0^1 \nabla \psi(w)^\top (w'-w)dt+\int_0^1 \big(\nabla \psi((1-t)w+tw')-\nabla \psi(w)\big)^\top (w'-w)dt\\
        &\le \nabla \psi(w)^\top (w'-w)+\|w-w'\|\int_0^1 \,\|\nabla \psi((1-t)w+tw')-\nabla \psi(w)\|\, dt\\
        &\le \nabla \psi(w)^\top (w'-w)+S(\|w_{\max,1}\|,\cdots,\|w_{\max,d}\|)\|w-w'\|^2\int_0^1 \,t\, dt\\
        &=\nabla \psi(w)^\top (w'-w)+\frac{S(\|w_{\max,1}\|,\cdots,\|w_{\max,d}\|)}{2}\|w-w'\|^2
    \end{align*}
\end{proof}

\subsubsection{Comparison Between Poly-Smoothness and Generalized Smoothness}
\label{subapp:generalized_smoothness_vs_poly}
We compare polynomial generalized smoothness with the generalized smoothness condition in~\citet{li2023convex}. The differences lie in two parts. (1) \citet{li2023convex} defined the Hessian bound first, and therefore evaluates the smoothness function at one fixed point for any two points in its neighborhood, while we directly consider the continuity of the gradient and evaluate the smoothness function $S(\cdot)$ at the two points. (2) Regarding the smoothness function, \citet{li2023convex} employed a function of $\|\nabla \psi(w)\|$, while we use a polynomial of $\|w_i\|,\|w_i'\|$. Note that such $S(\cdot)$ is monotonically increasing in $\RR_{\ge 0}\times\cdots\times\RR_{\ge 0}$.

\subsection{Local Relaxed Dissipative Condition}
\label{subapp:dissipativity}
\begin{proof}[Proof of Proposition~\ref{prop:dissipativity_rho}]
    Consider any $\tilde{w}\in B_{w^*}(r)\backslash\{v:\|\nabla \psi(v)\|\le \epsilon\}$. Then $\|\tilde{w}-w^*\|<r$ and $\|\nabla\psi(\tilde{w})\|>\epsilon$. By Cauchy-Schwarz,
    $$
        \nabla \psi(\tilde{w})^\top(\tilde{w}-w^*)
        \ge
        -\|\nabla\psi(\tilde{w})\|\,\|\tilde{w}-w^*\|
        >
        -r\|\nabla\psi(\tilde{w})\|.
    $$
    Since $\|\nabla\psi(\tilde{w})\|>\epsilon$, we have
    $$
        r\|\nabla\psi(\tilde{w})\|
        \le
        \frac{r}{\epsilon}\|\nabla\psi(\tilde{w})\|^2.
    $$
    Therefore,
    $$
        \nabla \psi(\tilde{w})^\top(\tilde{w}-w^*)
        \ge
        -\frac{r}{\epsilon}\|\nabla\psi(\tilde{w})\|^2.
    $$
    Thus $\rho=r/\epsilon$ is a valid dissipativity constant, and hence the dissipative condition admits a default choice with $\rho\le r/\epsilon$.
\end{proof}

\subsection{Example of Assumption~\ref{assump:architecture}}
\label{subapp:proof_example_full_rank_Nd}

Assumption~\ref{assump:architecture} is used in the convergence of neural networks. Before introducing the convergence proof, we first verify that the Assumption~\ref{assump:architecture} can be easily satisfied. We first show one example in Lemma~\ref{lem:example_2_exist_Nd_full_rank}. Lemma~\ref{lem:example_exist_Nd_full_rank} can be viewed as a corollary of Lemma~\ref{lem:example_2_exist_Nd_full_rank}. We then present the biased version stated in the main text.

\begin{lemma}
\label{lem:example_2_exist_Nd_full_rank}
    The following structure satisfies Assumption~\ref{assump:architecture} with $\ell=L-1$: $\varphi_L(\theta;u)=\theta_L \tau_\epsilon(u)$ and $\varphi_{L-1}(\theta;u)=\theta_{L-1,2}\sigma(\theta_{L-1,1}u)$, where $\sigma(\cdot)$ is GELU, $\theta_L\in\RR^{d\times d}$, $\theta_{L-1,1}\in\RR^{m_{L-1}\times d}$, $\theta_{L-1,2}\in\RR^{d\times m_{L-1}}$, and $m_{L-1}>Nd$.
\end{lemma}
\begin{proof}[Proof of Lemma~\ref{lem:example_2_exist_Nd_full_rank}]
First, we compute the Jacobian of the neural network w.r.t. $\theta_{L-1,2}$
\begin{align*}
     &\nabla_{\theta_{L-1,2}}f(\theta;x_i)\\
     &=\nabla _{\tau_\epsilon}\varphi_{L}(\tau_\epsilon(u_{L,i}))\nabla_{u}\tau_\epsilon(u_{L,i})\nabla_{\theta_{L-1,2}}{\varphi}(u_{L-1,i})\\
     &=\underbrace{\theta_L\left(\frac{1}{\sqrt{\|u_{L,i}\|^2+\epsilon^2}}I-\frac{1}{(\sqrt{\|u_{L,i}\|^2+\epsilon^2})^3}u_{L,i}u_{L,i}^\top\right)}_{P_i}\underbrace{\begin{pmatrix}
         \sigma(\theta_{L-1,1}u_{L-1,i} )^\top &&\\
         & \ddots &\\
         & &\sigma(\theta_{L-1,1}u_{L-1,i} )^\top
     \end{pmatrix}}_{Q_i}
\end{align*}
where in $\nabla_{\theta_{L-1,2}}{\varphi}(u_{L-1,i})$, we vectorize $\theta_{L-1,2}$ in row.

Also, $\sigma(x)=x\Phi(x)$, where $\Phi(x)=\int_{-\infty}^x \frac{e^{-s^2/2}}{\sqrt{2\pi}}ds$. Let $\phi(x)=\frac{e^{-x^2/2}}{\sqrt{2\pi}}$. Then we have
\begin{align*}
    \sigma'(x)&=\Phi(x)+x\phi(x)\\
    \sigma^{(n)}(x)&=x\phi^{(n-1)}(x)+n\phi^{(n-2)}(x),\qquad n\ge2\\
    &=(-1)^{n-2}\left(nH_{n-2}(x)-xH_{n-1}(x)\right)\phi(x),\qquad n\ge2
\end{align*}
where $H_{n}(x)$ is the probabilist's Hermite polynomial.

By Taylor expansion,
\begin{align*}
    \sigma(v^\top u+a)=\sum_{n=0}^\infty \frac{\sigma^{(n)}(a)}{n!}(v^\top u)^n
\end{align*}

 Consider $\sigma(t\theta u)$ near $t=0$. Choose $0<t_1<t_2<\cdots<t_N\ll 1$. For probabilist's Hermite polynomials, 
 \begin{align*}
     H_{2n}(0)\ne 0, H_{2n-1}(0)=0,\ \forall n\ge 1.
 \end{align*}
 Thus, 
 \begin{align*}
     \sigma(0)=0,\qquad
     \sigma'(0)=\frac{1}{2},\qquad
     \sigma^{(2q)}(0)\ne 0,\ \forall q\ge 1,\qquad
     \sigma^{(2q+1)}(0)=0,\ \forall q\ge 1.
 \end{align*}
 
Let $n_1=1$ and $n_q=2q-2$ for $q=2,\cdots,N+1$.
Let $\theta_{L-1,1,j}$ be the $j$th row of $\theta_{L-1,1}$. Then
\begin{align*}
    \begin{pmatrix}
        \sigma(t_1\theta_{L-1,1,j} u)\\
        \vdots\\
        \sigma(t_N\theta_{L-1,1,j} u)
    \end{pmatrix}=\begin{pmatrix}
        \sum_{n=0}^\infty \frac{\sigma^{(n)}(0)}{n!}(t_1\theta_{L-1,1,j} u)^n\\
        \vdots\\
         \sum_{n=0}^\infty \frac{\sigma^{(n)}(0)}{n!}(t_N\theta_{L-1,1,j} u)^n
    \end{pmatrix}
    =\sum_{q=1}^{N} \frac{\sigma^{(n_q)}(0)}{n_q!}(\theta_{L-1,1,j} u)^{n_q}\begin{pmatrix}
        t_1^{n_q}\\
        \vdots\\
        t_N^{n_q}
    \end{pmatrix}+\mathcal{O}\begin{pmatrix}
        t_1^{n_{N+1}}\\
        \vdots\\
        t_N^{n_{N+1}}
    \end{pmatrix}
\end{align*}

Note 
\begin{align*}
    \det\begin{pmatrix}
        t_1^{n_1}&\cdots &t_1^{n_N}\\
        \vdots&&\vdots\\
        t_N^{n_1}&\cdots&t_N^{n_N}
    \end{pmatrix}\ne0.
\end{align*}
Thus the $N$ vectors $\begin{pmatrix}
        t_1^{n_q}\\
        \vdots\\
        t_N^{n_q}
    \end{pmatrix}$, $q=1,\cdots,N$, are linearly independent.

Let
$$
    G(z_1,\cdots,z_N)=\det\begin{pmatrix}
        z_1^{n_1}&\cdots&z_1^{n_N}\\
        \vdots&&\vdots\\
        z_N^{n_1}&\cdots&z_N^{n_N}
    \end{pmatrix}.
$$
Then $G$ is a nonzero polynomial, since for distinct positive $z_1,\cdots,z_N$, the generalized Vandermonde matrix above is nonsingular. Let $e_1$ be the first standard basis vector in $\RR^d$. The polynomial
$$
    \theta_{L-1,1}\mapsto G(\theta_{L-1,1,1}e_1,\cdots,\theta_{L-1,1,N}e_1)
$$
is not identically zero. Hence its zero set, denoted by $\cB_{\theta_{L-1,1}}$, has measure zero in $\RR^{\mathrm{dim}\theta_{L-1,1}}$. Outside this exceptional set, we choose $u=e_1$, and then
\begin{align*}
     &\det\begin{pmatrix}
   \frac{\sigma^{(n_1)}(0)}{n_1!}(\theta_{L-1,1,1} u)^{n_1}&\cdots& \frac{\sigma^{(n_N)}(0)}{n_N!}(\theta_{L-1,1,1} u)^{n_N}\\
    \vdots&&\vdots\\
     \frac{\sigma^{(n_1)}(0)}{n_1!}(\theta_{L-1,1,N} u)^{n_1}&\cdots&\frac{\sigma^{(n_N)}(0)}{n_N!}(\theta_{L-1,1,N} u)^{n_N}
\end{pmatrix}\\
&=\prod_{q=1}^N \frac{\sigma^{(n_q)}(0)}{n_q!}\det\begin{pmatrix}
  (\theta_{L-1,1,1} u)^{n_1}&\cdots& (\theta_{L-1,1,1} u)^{n_N}\\
    \vdots&&\vdots\\
     (\theta_{L-1,1,N} u)^{n_1}&\cdots&(\theta_{L-1,1,N} u)^{n_N}
\end{pmatrix}\ne 0.
\end{align*}
Combining this with the Taylor expansion, the determinant of the first $N$ columns of the feature matrix has leading term
\begin{align*}
    \det\begin{pmatrix}
        t_1^{n_1}&\cdots&t_1^{n_N}\\
        \vdots&&\vdots\\
        t_N^{n_1}&\cdots&t_N^{n_N}
    \end{pmatrix}
    \prod_{q=1}^N\frac{\sigma^{(n_q)}(0)}{n_q!}
    \det\begin{pmatrix}
      (\theta_{L-1,1,1} u)^{n_1}&\cdots&(\theta_{L-1,1,1} u)^{n_N}\\
        \vdots&&\vdots\\
      (\theta_{L-1,1,N} u)^{n_1}&\cdots&(\theta_{L-1,1,N} u)^{n_N}
    \end{pmatrix},
\end{align*}
which is nonzero. Therefore, by choosing $0<t_1<t_2<\cdots<t_N$ sufficiently small, the determinant is nonzero.
Therefore, 
    $\begin{pmatrix}
        \sigma(t_1\theta_{L-1,1} u)^\top\\
        \vdots\\
        \sigma(t_N\theta_{L-1,1} u)^\top
    \end{pmatrix}$
is full rank $N$, and then $\begin{pmatrix}
    Q_1\\\vdots\\Q_N
\end{pmatrix}$ is full rank $Nd$. 

In the normalized case, the invertibility of $P_i$ can be justified as follows. Since
\begin{align*}
    \nabla_u\tau_\epsilon(u)
    =\frac{1}{(\|u\|^2+\epsilon^2)^{1/2}}I
    -\frac{uu^\top}{(\|u\|^2+\epsilon^2)^{3/2}},
\end{align*}
$\nabla_u\tau_\epsilon(u)$ has eigenvalue $(\|u\|^2+\epsilon^2)^{-1/2}$ on $u^\perp$ and eigenvalue $\epsilon^2(\|u\|^2+\epsilon^2)^{-3/2}$ in the radial direction. Therefore, $\nabla_u\tau_\epsilon(u)$ is invertible for every $u$, and $P_i=\theta_L\nabla_u\tau_\epsilon(u_{L,i})$ is invertible whenever $\theta_L$ is invertible.

Also, since $\Leb_{\mathrm{dim}\theta_L}(\{\theta\in\RR^{d\times d}:\det(\theta)=0\})=0$, $P_i$ is invertible for any $u_{L,i}$, and $\begin{pmatrix}
    P_1&&\\
    &\ddots&\\
    &&P_N
\end{pmatrix}$ is invertible. Thus
\begin{align*}
    \begin{pmatrix}
        \nabla _{\tau_\epsilon}\varphi_{L}(\tau_\epsilon(u_{L,1}))\nabla_{u}\tau_\epsilon(u_{L,1})\nabla_{\theta_{L-1,2}}{\varphi}(u_{L-1,1})\\
        \vdots\\
        \nabla _{\tau_\epsilon}\varphi_{L}(\tau_\epsilon(u_{L,N}))\nabla_{u}\tau_\epsilon(u_{L,N})\nabla_{\theta_{L-1,2}}{\varphi}(u_{L-1,N})
    \end{pmatrix}=\begin{pmatrix}
    P_1&&\\
    &\ddots&\\
    &&P_N
\end{pmatrix}\begin{pmatrix}
    Q_1\\\vdots\\Q_N
\end{pmatrix}
\end{align*}
is full rank $Nd$, where $u_{L-1,i}=t_i u$.
\end{proof}

\begin{lemma}
\label{lem:example_exist_Nd_full_rank}
    The following structure satisfies Assumption~\ref{assump:architecture} with $\bar\ell=L-1$:
    $$
        \varphi_L(\theta;u)=\theta_L u,\qquad
        \varphi_{L-1}(\theta;u)=\theta_{L-1,2}\sigma(\theta_{L-1,1}u),
    $$
    where $\sigma(\cdot)$ is GELU, $\theta_L\in\RR^{d\times d}$, $\theta_{L-1,1}\in\RR^{m_{L-1}\times d}$, $\theta_{L-1,2}\in\RR^{d\times m_{L-1}}$, and $m_{L-1}>Nd$.
\end{lemma}

\begin{proof}[Proof of Lemma~\ref{lem:example_exist_Nd_full_rank}]
First, we compute the Jacobian of the neural network w.r.t. $\theta_{L-1,2}$
\begin{align*}
     &\nabla_{\theta_{L-1,2}}f(\theta;x_i)\\
     &=\nabla _{u}\varphi_{L}(u_{L,i})\nabla_{\theta_{L-1,2}}{\varphi}(u_{L-1,i})\\
     &=\underbrace{{\theta_L}}_{P_i}\underbrace{\begin{pmatrix}
         \sigma(\theta_{L-1,1}u_{L-1,i} )^\top &&\\
         & \ddots &\\
         & &\sigma(\theta_{L-1,1}u_{L-1,i} )^\top
     \end{pmatrix}}_{Q_i}
\end{align*}
where in $\nabla_{\theta_{L-1,2}}{\varphi}(u_{L-1,i})$, we vectorize $\theta_{L-1,2}$ in row.

From the same argument in the proof of Lemma~\ref{lem:example_2_exist_Nd_full_rank}, we have $\begin{pmatrix}
    Q_1\\\vdots\\Q_N
\end{pmatrix}$ is full rank $Nd$.

Also, since $P_i=\theta_L$ is invertible for all $\theta_L$ except for a measure-zero set, we have $\begin{pmatrix}
    P_1&&\\
    &\ddots&\\
    &&P_N
\end{pmatrix}$ is invertible. Thus
\begin{align*}
    \begin{pmatrix}
        \nabla _{u}\varphi_{L}(u_{L,1})\nabla_{\theta_{L-1,2}}{\varphi}(u_{L-1,1})\\
        \vdots\\
        \nabla _{u}\varphi_{L}(u_{L,N})\nabla_{\theta_{L-1,2}}{\varphi}(u_{L-1,N})
    \end{pmatrix}=\begin{pmatrix}
    P_1&&\\
    &\ddots&\\
    &&P_N
\end{pmatrix}\begin{pmatrix}
    Q_1\\\vdots\\Q_N
\end{pmatrix}
\end{align*}
is full rank $Nd$, where $u_{L-1,i}=t_i u$ and $P_i=\theta_L$.
\end{proof}

\begin{proof}[Proof of Lemma~\ref{lem:example_exist_Nd_full_rank_bias}]
First, we compute the Jacobian of the neural network w.r.t. $\theta_{L-1,2}$. For each $i$,
\begin{align*}
     &\nabla_{\theta_{L-1,2}}\tilde f_{L-1}(\theta;u_{L-1,i})\\
     &=\nabla _u\varphi_L(u_{L,i})\nabla_{\theta_{L-1,2}}\varphi_{L-1}(u_{L-1,i})\\
     &=\underbrace{\theta_L}_{P_i}\underbrace{\begin{pmatrix}
         \sigma(\theta_{L-1,1}u_{L-1,i}+b_{L-1,1})^\top &&\\
         & \ddots &\\
         & &\sigma(\theta_{L-1,1}u_{L-1,i}+b_{L-1,1})^\top
     \end{pmatrix}}_{Q_i}
\end{align*}
where in $\nabla_{\theta_{L-1,2}}\varphi_{L-1}(u_{L-1,i})$, we vectorize $\theta_{L-1,2}$ in row.

We use the same GELU derivative facts and the same index set as in the proof of Lemma~\ref{lem:example_2_exist_Nd_full_rank}: let $n_1=1$ and $n_q=2q-2$ for $q=2,\cdots,N+1$, so that $\sigma^{(n_q)}(0)\ne0$ for $q=1,\cdots,N$.

It remains to show that the feature matrix produced by the biased hidden layer has rank $N$ for a suitable tuple $(u_{L-1,1},\cdots,u_{L-1,N})$, except for a measure-zero set of $(\theta_{L-1,1},b_{L-1,1})$. As in the proof of Lemma~\ref{lem:example_exist_Nd_full_rank}, choose distinct positive numbers $t_1,\cdots,t_N$ and $a_1,\cdots,a_N$, and let $e_1$ be the first standard basis vector in $\RR^d$. For $u_{L-1,i}=t_i e_1$, consider the parameter choice
$$
    b_{L-1,1,j}=0,\qquad
    \theta_{L-1,1,j}=\delta a_j e_1^\top,\qquad j=1,\cdots,N,
$$
where $\delta>0$ is sufficiently small. This special choice is used only to prove that the determinant below is not identically zero. For the $N\times N$ minor formed by the first $N$ hidden units, Taylor expansion gives
\begin{align*}
    \det\left(\sigma(\delta a_jt_i)\right)_{i,j=1}^N
    &=
    \delta^{\sum_{q=1}^N n_q}
    \left(\prod_{q=1}^N\frac{\sigma^{(n_q)}(0)}{n_q!}\right)
    \det\left(t_i^{n_q}\right)_{i,q=1}^N
    \det\left(a_j^{n_q}\right)_{j,q=1}^N\\
    &\qquad+o\left(\delta^{\sum_{q=1}^N n_q}\right).
\end{align*}
The two generalized Vandermonde determinants above are nonzero, and hence, for sufficiently small $\delta>0$,
$$
    \det\left(\sigma(\delta a_jt_i)\right)_{i,j=1}^N\ne0.
$$

Now fix such a tuple $u_{L-1,i}=t_i e_1$ and define the analytic function
$$
    D(\theta_{L-1,1},b_{L-1,1})
    =
    \det\left(\sigma(\theta_{L-1,1,j}u_{L-1,i}+b_{L-1,1,j})\right)_{i,j=1}^N,
$$
where $\theta_{L-1,1,j}$ denotes the $j$th row of $\theta_{L-1,1}$. The calculation above shows that $D$ is not identically zero. By Theorem~\ref{thm:zero_set_analytic_function}, its zero set has measure zero. Therefore, except for a measure-zero set of $(\theta_{L-1,1},b_{L-1,1})$, the matrix
$$
    S=
    \begin{pmatrix}
        \sigma(\theta_{L-1,1}u_{L-1,1}+b_{L-1,1})^\top\\
        \vdots\\
        \sigma(\theta_{L-1,1}u_{L-1,N}+b_{L-1,1})^\top
    \end{pmatrix}
$$
has rank $N$. After a permutation of rows, the stacked matrix $\begin{pmatrix}Q_1\\ \vdots\\ Q_N\end{pmatrix}$ is block diagonal with $d$ copies of $S$, and hence it has rank $Nd$. Finally, as in the proof of Lemma~\ref{lem:example_exist_Nd_full_rank}, $P_i=\theta_L$ is invertible for all $\theta_L$ except for a measure-zero set. Therefore the stacked Jacobian with respect to $\theta_{L-1,2}$ has full row rank $Nd$. Since this is a submatrix of the stacked Jacobian with respect to $\theta_{L-1}$, the full stacked Jacobian also has full row rank $Nd$. The exceptional set is still measure zero after including the remaining parameters $\theta_{L-1,2}$ and $b_{L-1,2}$, which do not affect the determinant $D$. Thus Assumption~\ref{assump:architecture} holds.
\end{proof}

\subsection{Example of Model Assumptions}
\label{app:example_model_assumptions}

\begin{proof}[Proof of Lemma~\ref{lem:prenorm_transformer_model_assumptions}]
The proof is a direct consequence of Lemmas~\ref{cor:poly_smooth_with_NN},~\ref{cor:analytic_with_NN}, and~\ref{lem:example_exist_Nd_full_rank_bias}, together with the block decompositions in~\eqref{eqn:ex_block_ff} and~\eqref{eqn:ex_block_mha}. Indeed, by Lemma~\ref{cor:analytic_with_NN}, layer normalization, softmax, feedforward blocks, and finite compositions of these functions are real-analytic, so Assumption~\ref{assump:analytic} holds. By Lemma~\ref{cor:poly_smooth_with_NN}, the same building blocks satisfy polynomial boundedness, polynomial generalized continuity, and polynomial generalized smoothness, and therefore Assumption~\ref{assump:generalized_smoothness_detail} holds for their finite composition. Finally, Assumption~\ref{assump:architecture} is verified by Lemma~\ref{lem:example_exist_Nd_full_rank_bias} when the last feedforward block has hidden width $m_{L-1}>Nd$. Hence the pre-normalization transformer satisfies all model assumptions in Section~\ref{subsec:model_assumptions}.
\end{proof}

\section{Optimization: Proof of Theorem~\ref{thm:training}}
\label{app:optimization}

We recall that the architecture of neural network is defined as follows
\begin{align*}
    u_{0,i}&=x_i\\
    u_{\ell+1,i}&=u_{\ell,i}+{\varphi}_\ell(\theta;u_{\ell,i}),\ \ell=0,\cdots,L-1\\
    f(x_i)&={\varphi}_L(\theta;u_{L,i}).
\end{align*}
The loss function of each data point is
\begin{align*}
    l(f(x_i),y_i)=\frac{1}{2}\|y_i-f(x_i)\|^2
\end{align*}
and the total loss is
\begin{align*}
    \cL(f(x_i),y_i)=\frac{1}{N}\sum_{i=1}^Nl(f(x_i),y_i).
\end{align*}

We then compute the Jacobian of the following maps that will be used in the proof:
\begin{align*}
    \nabla_f\,l(x_i)=(f(x_i)-y_i)^\top
\end{align*}
\begin{align*}
    \nabla_{\theta_j}f(x_i)&=\sum_{\ell\in\cJ_j}\nabla_u{\varphi}_L(u_{L,i})(I+\nabla_u{\varphi}_{L-1}(u_{L-1,i}))\cdots (I+\nabla_u{\varphi}_{\ell+1}(u_{\ell+1,i}))\nabla_{\theta_j}{\varphi}_{\ell}(u_{\ell,i};\theta)
\end{align*}
where $\nabla_u{\varphi}_{\ell}(u_{\ell})$ is the Jacobian of the map ${\varphi}_{\ell}(u_{\ell})$ with respect to $u_\ell$, $\nabla_{\theta_\ell}f(x_i)$ is also the Jacobian; the set $\cJ_j$ denotes the indices of layers in which $\theta_j$ appears.

Then
\begin{align*}
\nabla_{\theta_j}l(x_i)&=\nabla_f \,l(x_i)\nabla_{\theta_\ell}f(x_i)\\
    &=\sum_{\ell\in\cJ_j}\nabla_f \,l(x_i)\nabla_u{\varphi}_L(u_{L,i})(I+\nabla_u{\varphi}_{L-1}(u_{L-1,i}))\cdots (I+\nabla_u{\varphi}_{\ell+1}(u_{\ell+1,i}))\nabla_{\theta_j}{\varphi}_{\ell}(u_{\ell,i};\theta)
\end{align*}
and we have
\begin{align*}
    \nabla_{\theta_j}\cL=\frac{1}{N}\sum_{i=1}^N\nabla_{\theta_j}l(x_i).
\end{align*}
As shown in equation~\eqref{eqn:NN}, we consider the case where the weights at each layer are distinct, i.e., $\cJ_j=\{j\}$.

In this section, we also define the following function for any matrix function $M(x)\in \RR^{n\times m}$ with $m\ge n$
\begin{align*}
    \detr(M(x))=\sum_{i=1}^{m\choose n}(\det M_i(x))^2,
\end{align*}
where $M_i(x)\in\RR^{n\times n}$ is the matrix with $n$ columns of $M(x)$. Then ${\det}_{\rm r}(M(x))\ne 0$ if and only if $M(x)$ is full rank at $x$.

The following is a more complete version of the convergence result in Theorem~\ref{thm:training}. It also records the finite-step decay bound and explicitly excludes only finitely many exceptional learning rate values.
\begin{theorem}
\label{thm:complete_convergence}
Suppose GD is applied to the neural network~\eqref{eqn:NN} under the following requirements, with probability 1 over the joint distribution $\pi$ of the input data $x_1,\cdots,x_N$:
\begin{enumerate}
    \item \textbf{Data and Model Assumptions.} Assumptions~\ref{assump:data_distribution_absolutely_continuous_Lebesgue_ground_truth},~\ref{assump:architecture},~\ref{assump:generalized_smoothness_detail}, and~\ref{assump:analytic} hold.
    \item \textbf{Initialization.} The initialization $\theta^0$ belongs to $\RR^{\dim \theta}$ except for a measure-zero set.
\end{enumerate}
Then, under some arbitrarily small adjustment on the scale of $\varphi_\ell$ for $\ell=0,\cdots,L-1$,
the following conclusions hold.
\begin{enumerate}
    \item \textbf{Finite-Time Convergence Bound.} Let the learning rate $\eta$ be chosen from the interval
    $$
        0<\eta\le \min\left\{\frac{2-\delta}{\mathsf L_1},1\right\},
    $$
    excluding at most finitely many exceptional values,
    where
    $$
        \mathsf{L}_1
        =
        S\left(\cdots,\|\theta_i^0\|+\sqrt{\frac{2C}{\delta}\cL(\theta^0)}
        +\eta\max_{\|\theta\|\le \|\theta^0\|+\sqrt{\frac{2C}{\delta}\cL(\theta^0)}}\|\nabla_{\theta_i}\cL(\theta)\|,\cdots\right),
    $$
    and $S(\cdot)$ is a polynomial whose coefficients are all positive. Then for all $k+1\le T=C/\eta$,
    $$
        \cL(\theta^k)
        \le
        \min\left\{
        \prod_{s=0}^{k-1}\left(1-\eta\left(1-\frac{\eta\mathsf L_1}{2}\right)\frac{2\mu_{\mathrm{low},s,X}}{N}\right),
        \left(1-\eta\left(1-\frac{\eta\mathsf L_1}{2}\right)\frac{\epsilon_\cL^2}{\cL(\theta^0)}\right)^k
        \right\}\cL(\theta^0),
    $$
    where $\mu_{\mathrm{low},s,X}>0$ depends on $\theta^s$ and $x_1,\cdots,x_N$, and $C>0$ is a universal constant.
    \item \textbf{Convergence Under a Local Dissipative Condition.} Assume, in addition, that the loss $\cL(\theta)$ satisfies the $(\theta^0,\theta^*,R,\rho,\epsilon_\cL)$-dissipative condition in Definition~\ref{def:dissipative}, where $\theta^*$ is a stationary point, $0<\delta<2$, $\rho\in\RR$, $\epsilon_\cL\ge0$, and
    $$
        R=\sqrt{\|\theta^0-\theta^*\|^2+\max\left\{\frac{4\rho+2}{\delta},0\right\}\cL(\theta^0)}.
    $$
    Let the learning rate $\eta$ be chosen from the interval
    $$
        0<\eta\le \min\left\{\frac{2-\delta}{\mathsf L_2},1\right\},
    $$
    excluding at most finitely many exceptional values,
    where
    $$
        \mathsf{L}_2
        =
        S\left(\cdots,\|\theta_i^*\|+R+\max_{\theta\in B_{\theta^*}(R)}\|\nabla_{\theta_i}\cL(\theta)\|,\cdots\right),
    $$
    and $S(\cdot)$ is a polynomial whose coefficients are all positive. Then GD converges to the region $\|\nabla\cL(\theta)\|\le\epsilon_\cL$. Moreover, for all $k\ge1$ such that $\|\nabla\cL(\theta^k)\|>\epsilon_\cL$,
    $$
        \cL(\theta^k)
        \le
        \min\left\{
        \prod_{s=0}^{k-1}\left(1-\eta\left(1-\frac{\eta\mathsf L_2}{2}\right)\frac{2\mu_{\mathrm{low},s,X}}{N}\right),
        \left(1-\eta\left(1-\frac{\eta\mathsf L_2}{2}\right)\frac{\epsilon_\cL^2}{\cL(\theta^0)}\right)^k
        \right\}\cL(\theta^0),
    $$
    where $\mu_{\mathrm{low},s,X}>0$ depends on $\theta^s$ and $x_1,\cdots,x_N$.
\end{enumerate}
\end{theorem}

\begin{proof}
Define the GD iteration map to be 
\begin{align*}
    \Psi(\theta)=\theta-\eta\nabla_\theta\cL(\theta),\text{ and }\Psi^k(\theta)=\underbrace{\Psi\circ\cdots\circ\Psi}_{k}(\theta).
\end{align*}

First, consider fixed $x_1,\cdots,x_N$ and note that all the sets of $\theta$ discussed below depend on the data $x_1,\cdots,x_N$. Let
\begin{align*}
    \cB_{\theta,\rm rank}=\{\theta:\nabla_\theta\Psi(\theta)\text{ is not full rank}\}
\end{align*}

By Lemma~\ref{lem:full_rank_jacobian}, applied to $\Psi(\theta)=\theta+\eta(-\nabla_\theta\cL(\theta))$, we know that for $\eta\in \cD_{\eta,0}$, where $\cD_{\eta,0}$ is the relevant learning rate interval $(0,\min\{\frac{2-\delta}{\mathsf L_i},1\})$ excluding finitely many values, we have
\begin{align*}
    \Leb_{\dim\theta}(\cB_{\theta,\rm rank})=0.
\end{align*}
Next we show that along the trajectory, i.e., for any $k$, the union of ``bad" sets of the initial condition $\theta$ leading to the degeneracy of the map $\Psi^k$ is a measure-zero set.

Now consider any measure-zero set $\cB_{\theta,0}$ in $\RR^{\dim\theta}$. By Theorem~\ref{thm:inverse_map_null_set}, 
\begin{align*}
    \Leb_{\dim\theta}(\Psi^{-1}(\cB_{\theta,0}))=0.
\end{align*}

We then prove that $(\Psi^k)^{-1}(\cB_{\theta,0})$ is measure-zero for any $k\ge 1$. Suppose $(\Psi^k)^{-1}(\cB_{\theta,0})$ is measure-zero. We show that $(\Psi^{k+1})^{-1}(\cB_{\theta,0})$ is measure zero. This is obvious since $(\Psi^{k+1})^{-1}(\cB_{\theta,0})=\Psi^{-1}((\Psi^{k})^{-1}(\cB_{\theta,0}))$.

Let $\cB_{\theta,0}$ be the union of all the bad measure-zero sets of $\theta$ in Assumption~\ref{assump:architecture} and Lemmas~\ref{lem:gradient_lower_bound}, including $\cB_{\theta,\rm rank}$, and there are finite union of such sets, which implies $\Leb_{\dim \theta}(\cB_{\theta,0})=0$. Let
\begin{align*}
    \cB_{\theta}=\bigcup_{k=0}^\infty (\Psi^{k})^{-1}(\cB_{\theta,0}),
\end{align*}
where $\Psi^0$ is the identity map.
Then
\begin{align*}
    \Leb_{\dim\theta}(\cB_{\theta})=\Leb_{\dim\theta}\left(\bigcup_{k=0}^\infty (\Psi^{k})^{-1}(\cB_{\theta,0}) \right)\le \sum_{k=0}^\infty 0=0,
\end{align*}
namely, the initial condition set of $\theta$ where there exists some iteration $k\ge0$ s.t. $\theta^k\in\cB_{\theta,0}$, is measure-zero. In the rest of the proof, we just consider $\theta^0\in \RR^{\dim \theta}\backslash \cB_{\theta}$.

Next we consider the bad set of $(x_1,\cdots,x_N)$. For each $k\ge0$, define the following set
\begin{align*}
    \cB_{X,k}:=\{X:\text{Lemma~\ref{lem:gradient_lower_bound} fails at }(\theta^k(X),X)\}.
\end{align*}
By the analytic zero-set argument in Lemma~\ref{lem:gradient_lower_bound}, after substituting the analytic map $X\mapsto\theta^k(X)$ into the corresponding determinant, we have $\Leb_{Nd}(\cB_{X,k})=0$. Let $\cB_X:=\bigcup_{k=0}^\infty\cB_{X,k}$. Since there are at most countably many iterates,
\begin{align*}
    \Leb_{Nd}(\bigcup_{k=0}^\infty\cB_{X,k})\le \sum_{k=0}^\infty 0=0.
\end{align*}

Since $(x_1,\cdots,x_N)\sim \pi(z_1,\cdots,z_N)dz_1\cdots dz_N$, and $\pi\ll \Leb_{Nd}$, we have
\begin{align*}
   \pi(\cB_X)=0,
\end{align*}
namely, with probability 0, $(x_1,\cdots,x_N)$ will be in this set.

Then, consider some initial conditions $\theta^0$ except for a measure-zero set, and consider some dataset $\{(x_i,y_i)\}_{i=1}^N$ chosen with probability 1 over the joint distribution. By Lemma~\ref{lem:loss_function_generalized_smoothness} and Lemma~\ref{lem:smoothness_inequalities}, we have for $i=1,2$,
    \begin{align*}
    \cL(\theta^{k+1})&\le \cL(\theta^{k})+\nabla \cL(\theta^{k})^\top(\theta^{k+1}-\theta^{k})+\frac{S_{k}}{2}\|\theta^k-\theta^{k+1}\|^2\\
    &= \cL(\theta^{k})-\eta\left(1-\frac{\eta S_k}{2}\right)\|\nabla\cL(\theta^k)\|^2\\
    &\le \left(1- \eta\left(1-\frac{\eta S_k}{2}\right)\frac{2\mu_{\mathrm{low},k,X}}{N}\right)\cL(\theta^{k})\\
    &\le \left(1- \eta\left(1-\frac{\eta\, \mathsf{L}_i}{2}\right)\frac{2\mu_{\mathrm{low},k,X}}{N}\right)\cL(\theta^{k})\\
    &\le \prod_{s=0}^k\left(1- \eta\left(1-\frac{\eta\, \mathsf{L}_i}{2}\right)\frac{2\mu_{\mathrm{low},s,X}}{N}\right)\cL(\theta^{0})
\end{align*}
where the second inequality follows from Lemma~\ref{lem:gradient_lower_bound}, the third inequality follows from Lemma~\ref{lem:bound_S_k_theta_k} and $S(a_1,\cdots,a_{n_\theta})\le \mathsf{L}_i$ by the monotonicity of $S(\cdot)$ in $\RR_{\ge 0}\times\cdots\times\RR_{\ge 0}$; $\mu_{\mathrm{low},k,X}>0$ is some strictly positive constant depending on $\theta^k$ and $x_1,\cdots,x_N$.

By Lemma~\ref{lem:bound_S_k_theta_k}, when $i=1$, the above inequality holds for $k+1\le T=\frac{C}{\eta}$. When $i=2$ with the $(\theta^0,\theta^*,R,\rho,\epsilon_\cL)$-dissipative condition, it holds for any $k\ge 1$ such that $\|\nabla\cL(\theta^k)\|>\epsilon_\cL$.

Additionally, for all $\|\nabla\cL(\theta^k)\|>\epsilon_\cL$, we have
\begin{align*}
    \|\nabla\cL(\theta^k)\|^2\ge \frac{\epsilon_\cL^2}{\cL(\theta^0)}\cL(\theta^k),
\end{align*}
where the inequality follows from $\cL(\theta^k)\le\cL(\theta^0)$. Namely,
\begin{align*}
    \cL(\theta^{k})\le \left(1- \eta\left(1-\frac{\eta\, \mathsf{L}_i}{2}\right)\frac{\epsilon_\cL^2}{\cL(\theta^0)}\right)^{k}\cL(\theta^{0})
\end{align*}
Combining this estimate with the product estimate above gives the stated minimum of the two bounds.

Then, under the $(\theta^0,\theta^*,R,\rho,\epsilon_\cL)$-dissipative condition, GD will converge to the region $\|\nabla\cL(\theta)\|\le \epsilon_\cL$. Indeed, if $\|\nabla\cL(\theta^k)\|>\epsilon_\cL$ for all sufficiently large $k$, then $\cL$ will keep decreasing until it enters the neighborhood of $0$ which is the global minimum of $\cL$.


\end{proof}

The following lemma shows that the loss function is Lipschitz smooth along the GD iteration.
\begin{lemma}
\label{lem:bound_S_k_theta_k}
Assume $\cL(\theta)$ satisfies polynomial generalized smoothness. 

\begin{enumerate}
    \item Without the $(\theta^0,\theta^*,R,\rho,\epsilon_\cL)$-dissipative condition, let $\eta\le\min\{\frac{2-\delta}{\mathsf{L}_2},1\}$, where \begin{align*}
    \mathsf{L}_2=S\left(\cdots,\|\theta_i^0\|+\sqrt{\frac{2C}{\delta}\cL(\theta^0)}+\max_{\|\theta\|\le \|\theta^0\|+\sqrt{\frac{2C}{\delta}\cL(\theta^0)}}\|\nabla_{\theta_i}\cL(\theta)\|,\cdots)\right).
\end{align*}
Then 
\begin{align*}
    &\theta^k\in B_0\left(\|\theta^0\|+\sqrt{\frac{2C}{\delta}\cL(\theta^0)}\right),\\
    \text{and }\quad &S_k\le \mathsf{L}_2,\ \forall\, k+1\le T=\frac{C}{\eta}.
\end{align*}
    \item With the $(\theta^0,\theta^*,R,\rho,\epsilon_\cL)$-dissipative condition, let $\eta\le\min\{\frac{2-\delta}{\mathsf{L}_2},1\}$, where \begin{align*}
    \mathsf{L}_2=S\left(\cdots,\|\theta_i^*\|+R+\max_{ \theta\in B_{\theta^*}\left(R\right)}\|\nabla_{\theta_i}\cL(\theta)\|,\cdots\right),\text{ with }R=\sqrt{\|\theta^{0}-\theta^{*}\|^2+\max\{\frac{4\rho+2}{\delta},0\}\cL(\theta^0)}.
\end{align*}
Then
\begin{align*}
    &\theta^k\in B_{\theta^*}(R),\\
    \text{and }\quad &S_k\le \mathsf{L}_2,\ \forall\, k\ge 0\text{ and }\|\nabla\cL(\theta^k)\|> \epsilon_\cL.
\end{align*}
\end{enumerate}
\end{lemma}

\begin{proof}
    First, note
    \begin{align*}
        \max\{\|\theta^k_i\|,\|\theta^{k+1}_i\|\}\le\|\theta_i^k\|+\eta\|\nabla_{\theta_i}\cL(\theta^k)\|.
    \end{align*}
    By the monotonicity of $S(\cdot)$, we only need to show $\|\theta^k_i\|$ is bounded under the two cases.
    
    Without the dissipative condition, we would like to show that 
    \begin{align*}
        \|\theta^k_i\|\le \|\theta_i^0\|+\sqrt{\frac{2C}{\delta}\cL(\theta^0)},\ \forall\, k+1\le T=\frac{C}{\eta}.
    \end{align*}
    Suppose the above inequality holds for all $\theta^j_i$, where $j\le k$ and $i=1,\cdots,n_\theta$. Then by $\eta\le  1$, we have
    \begin{align*}
        S_k&\le\bar{S}_k=S(\cdots,\|\theta_i^k\|+\eta\|\nabla_{\theta_i}\cL(\theta^k)\|,\cdots))\\
        &\le \mathsf{L}_2=S(\cdots,\|\theta_i^0\|+\sqrt{\frac{2C}{\delta}\cL(\theta^0)}+\max_{\|\theta\|\le \|\theta^0\|+\sqrt{\frac{2C}{\delta}\cL(\theta^0)}}\|\nabla_{\theta_i}\cL(\theta)\|,\cdots))
    \end{align*}
    namely, \begin{align*}
        \eta\le\frac{2-\delta}{\mathsf{L}_2}\le\frac{2-\delta}{\bar{S}_k}.
    \end{align*}

By Lemma~\ref{lem:loss_function_generalized_smoothness} and Lemma~\ref{lem:smoothness_inequalities}, we have
    \begin{align*}
    0\le \cL(\theta^{k+1})&\le \cL(\theta^{k})+\nabla \cL(\theta^{k})^\top(\theta^{k+1}-\theta^{k})+\frac{S_{k}}{2}\|\theta^k-\theta^{k+1}\|^2\\
    &= \cL(\theta^{k})-\eta\left(1-\frac{\eta S_k}{2}\right)\|\nabla\cL(\theta^k)\|^2\\
    &\le \cL(\theta^0)-\eta\sum_{j=0}^k\left(1-\frac{\eta S_j}{2}\right)\|\nabla\cL(\theta^j)\|^2\\
    &\le \cL(\theta^0)-\frac{\delta}{2}\eta\sum_{j=0}^k\|\nabla\cL(\theta^j)\|^2,
\end{align*}
namely,
\begin{align}
\label{eqn:sum_grad_square}
    \sum_{j=0}^k\|\nabla\cL(\theta^j)\|^2\le \frac{2}{\delta\eta}(\cL(\theta^0)-\cL(\theta^{k+1}))\le\frac{2}{\delta\eta}\cL(\theta^0).
\end{align}

    Then
    \begin{align*}
        \|\theta^{k+1}_i\|&=\|\theta^{k}_i-\eta\nabla_{\theta_i}\cL(\theta^k)\|=\left\|\theta^0_i-\eta\sum_{j=0}^{k}\nabla_{\theta_i}\cL(\theta^j)\right\|\\
        &\le \|\theta_i^0\|+\eta\sum_{j=0}^{k}\left\|\nabla\cL(\theta^j)\right\|\\
        &\le \|\theta_i^0\|+\eta \sqrt{(k+1)\sum_{j=0}^{k}\left\|\nabla\cL(\theta^j)\right\|^2}\\
        &\le \|\theta_i^0\|+\sqrt{\frac{2}{\delta}(k+1)\eta\cL(\theta^0)}\\
        &\le  \|\theta_i^0\|+\sqrt{\frac{2C}{\delta}\cL(\theta^0)}
    \end{align*}
    where the second inequality follows from Cauchy-Schwarz inequality, the third inequality follows from~\eqref{eqn:sum_grad_square}, and the last inequality follows from $k+1\le T=\frac{C}{\eta}$. 

    The same derivation also gives
    \begin{align*}
        \|\theta^{k+1}\|\le \|\theta^0\|+\sqrt{\frac{2C}{\delta}\cL(\theta^0)}.
    \end{align*}
    Therefore $S_{k+1}\le \mathsf{L}_2$.

Next, under the dissipative condition, we would like to show that 
\begin{align*}
    \theta^k_i\in B_{\theta^*_i}\left(\sqrt{\|\theta^{0}-\theta^{*}\|^2+\max\{\frac{4\rho+2}{\delta},0\}\cL(\theta^0)}\right),\ i.e.,\ \|\theta_i^k-\theta_i^*\|^2\le \|\theta^{0}-\theta^{*}\|^2+\frac{4\rho+2}{\delta}\cL(\theta^0)
\end{align*}

Suppose the above inequality holds for all $j\le k$. Then we have
\begin{align*}
    S_k&\le\bar{S}_k=S(\cdots,\|\theta_i^k\|+\eta\|\nabla_{\theta_i}\cL(\theta^k)\|,\cdots)\\
    &\le \mathsf{L}_2=S(\cdots,\|\theta_i^*\|+\sqrt{\|\theta^{0}-\theta^{*}\|^2+\max\{\frac{4\rho+2}{\delta},0\}\cL(\theta^0)}+\max_{ \theta\in B_{\theta^*}\left(\sqrt{\|\theta^{0}-\theta^{*}\|^2+\max\{\frac{4\rho+2}{\delta},0\}\cL(\theta^0)}\right)}\|\nabla_{\theta_i}\cL(\theta)\|,\cdots)
\end{align*}
namely, \begin{align*}
    \eta\le\frac{2-\delta}{\mathsf{L}_2}\le\frac{2-\delta}{\bar{S}_k}.
\end{align*}

Then, consider the $k+1$th iteration
 \begin{align*}
                \|\theta^{k+1}-\theta^{*}\|^2
                &=\|\theta^{k}-\theta^{*}-\eta\nabla \cL(\theta^k)\|^2\\
        &=\|\theta^{k}-\theta^{*}\|^2-2\eta\nabla \cL(\theta^k)^\top (\theta^{k}-\theta^{*})+\eta^2\|\nabla \cL(\theta^k)\|^2\\
        &\le \|\theta^{k}-\theta^{*}\|^2+\eta(2\rho+\eta)\|\nabla \cL(\theta^k)\|^2\\
        &\le \|\theta^{0}-\theta^{*}\|^2+\eta(2\rho+\eta)\sum_{j=0}^k\|\nabla \cL(\theta^j)\|^2\\
        &\le \|\theta^{0}-\theta^{*}\|^2+\max\{\frac{4\rho+2\eta}{\delta},0\}\cL(\theta^0)
    \end{align*}
    where the first inequality follows from Definition~\ref{def:dissipative}, and the last inequality follows from~\eqref{eqn:sum_grad_square}, which also holds in the case with the dissipative condition.

    Then together with $\eta\le 1$ and
    \begin{align*}
        \|\theta_i^{k+1}\|\le\|\theta^*_i\|+\|\theta_i^{k+1}-\theta_i^*\|\le \|\theta^*_i\|+\|\theta^{k+1}-\theta^*\|,
    \end{align*}
    we obtain $S_{k+1}\le \mathsf{L}_2$ for all k.
\end{proof}

\subsection{Lower Bound of Gradient}
\label{subapp:lower_bound_gradient}

In this section, we use $\theta_i$ and $\theta_{\bar{\ell},i}$ to represent the scalar elements of $\theta$ and $\theta_{\bar{\ell}}$ respectively. The main lemma in this section is as follows.
\begin{lemma}
\label{lem:gradient_lower_bound}
    Under \cref{assump:architecture}, and some small adjustment of the scale of $\varphi_\ell$ for $\ell=0,\cdots,L-1$, given fixed $\theta$ except for a measure-zero set in $\RR^{\dim \theta}$, and $X=(x_1,\cdots,x_N)$ except for a measure-zero set in $\RR^{Nd}$ that depends on $\theta$, the gradient satisfies the following inequality
    \begin{align*}
        \|\nabla_\theta\cL(\theta;X)\|^2\ge \frac{2\mu_{\rm low,\theta,X}}{N}\cL(\theta;X)
    \end{align*}
    where $\mu_{\rm low,\theta,X}>0$ is a strictly positive constant depending on $\theta$ and $X$. 

\end{lemma}

\begin{proof}
Since $l(f(x_i),y_i)=\frac{1}{2}\|y_i-f(x_i)\|^2$, we have $\nabla_f l(x_i)=f(x_i)-y_i$ and
$$
    \sum_{i=1}^N\|\nabla_f l(x_i)\|^2=2N\cL(\theta;X).
$$
Consider
    \begin{align*}
        \|\nabla_\theta\cL(\theta;X)\|^2&=\left\| \frac{1}{N}\sum_{i=1}^N \nabla_f l(x_i)\nabla_\theta f(x_i) \right\|^2\\
        &\ge \frac{1}{N^2}\mu_{\rm low,\theta,X}\sum_{i=1}^N\|\nabla_fl(x_i)\|^2=\frac{2\mu_{\rm low,\theta,X}}{N}\cL(\theta;X)
    \end{align*}
    where the first inequality follows from Lemma~\ref{lem:lower_frame_bound}, and $\mu_{\rm low,\theta,X}>0$ is a constant depending on $\theta$ and $X$.
\end{proof}

Before presenting other lemmas, we first introduce the definition of a frame as follows.
\begin{definition}[Frame]
    The set of vectors $\{e_k\}$ in an inner-product vector space $V$ is a frame of $V$ if there exist constants $0< \mu_{\rm low}\le \mu_{\rm up}<\infty$ such that
    \begin{align*}
        \mu_{\rm low}\|v\|^2\le\sum_k |\ip{v}{e_k}|^2\le \mu_{\rm up}\|v\|^2,\ \forall v\in V.
    \end{align*}
    The frame operator $T:V\to V$ is defined as
    \begin{align*}
        Tv=\sum_k\ip{v}{e_k}e_k=\left(\sum_k e_k e_k^\top\right)\,v.
    \end{align*}
\end{definition}

\begin{proposition}
\label{prop:frame}
    Consider a set of vectors $\{e_k\}_{k=1}^n$ in $\RR^d$. If ${\rm span}\{e_1,\cdots,e_n\}=\RR^d$, then $\{e_k\}_{k=1}^n$ is a frame for $\RR^d$.
\end{proposition}

Then the gradient of the network $f$ at all data points forms a frame in $\RR^{Nd}$:
\begin{lemma}
\label{lem:grad_frame}
Under \cref{assump:architecture} and some small adjustment of the scale of $\varphi_\ell$ for $\ell=0,\cdots,L-1$, for any fixed $\theta$ except for a measure-zero set,
$\left\{\begin{pmatrix}
    \nabla_{\theta_i}f(x_1)\\
    \vdots\\
    \nabla_{\theta_i}f(x_N))
\end{pmatrix}\right\}_{i=1}^{\dim\theta}$ form a frame except for a measure-zero set of $(x_1,\cdots,x_N)$ in $\RR^{Nd}$.
\end{lemma}
\begin{proof}

By Assumption~\ref{assump:architecture}, there exists $\bar{\ell}\in\{0,\cdots,L-1\}$ and $U=(u_1,\cdots,u_N)\in\RR^{Nd},\theta$, s.t. 
\begin{align*}
    \detr\begin{pmatrix}
       \nabla_{\theta_{\bar{\ell}}}\tilde{f}_{\bar{\ell}}(\theta;u_1)\\
        \vdots\\
        \nabla_{\theta_{\bar{\ell}}}\tilde{f}_{\bar{\ell}}(\theta;u_N)
    \end{pmatrix}\ne 0,
\end{align*}
where $\tilde{f}_{\bar{\ell}}(\theta;u_{\bar{\ell}})=f(\theta;x)$.

Then by Theorem~\ref{thm:zero_set_analytic_function},
\begin{align*}
    \mathrm{Leb}_{Nd}\left(\left\{U\in \RR^{Nd}: \detr\begin{pmatrix}
       \nabla_{\theta_{\bar{\ell}}}\tilde{f}_{\bar{\ell}}(\theta;u_1)\\
        \vdots\\
        \nabla_{\theta_{\bar{\ell}}}\tilde{f}_{\bar{\ell}}(\theta;u_N)
    \end{pmatrix}= 0\right\}\right)=0,
\end{align*}
namely,
\begin{align*}
    \mathrm{Leb}_{Nd}\left(\left\{U\in \RR^{Nd}:\begin{pmatrix}
       \nabla_{\theta_{\bar{\ell}}}\tilde{f}_{\bar{\ell}}(\theta;u_1)\\
        \vdots\\
        \nabla_{\theta_{\bar{\ell}}}\tilde{f}_{\bar{\ell}}(\theta;u_N)
    \end{pmatrix}\text{ is not full rank}\right\}\right)=0
\end{align*}

Next consider the map $U=U(X)=\begin{pmatrix}
    u_{\bar{\ell}}(x_1)\\\vdots\\u_{\bar{\ell}}(x_N)
\end{pmatrix}:\RR^{Nd}\to\RR^{Nd}$. By Lemma~\ref{lem:preimage_u_measure_zero}, 
\begin{align*}
    \mathrm{Leb}_{Nd}\left(\left\{(x_1,\cdots,x_N)\in \RR^{Nd}:\begin{pmatrix}
       \nabla_{\theta_{\bar{\ell}}}\tilde{f}_{\bar{\ell}}(\theta;u(x_1))\\
        \vdots\\
        \nabla_{\theta_{\bar{\ell}}}\tilde{f}_{\bar{\ell}}(\theta;u(x_N))
    \end{pmatrix}\text{ is not full rank}\right\}\right)=0
\end{align*}
Therefore by Proposition~\ref{prop:frame}, for any fixed $\theta$ except for a measure-zero set, $\left\{\begin{pmatrix}
    \nabla_{\theta_i}f(x_1)\\
    \vdots\\
    \nabla_{\theta_i}f(x_N))
\end{pmatrix}\right\}_{i=1}^{\dim \theta}$ forms a frame except for a measure-zero set of $(x_1,\cdots,x_N)$.
\end{proof}

\begin{lemma}
\label{lem:lower_frame_bound}
    Under Assumption~\ref{assump:architecture} and some small adjustment of the scale of $\varphi_\ell$ for $\ell=0,\cdots,L-1$, consider fixed $\theta$ except for a measure-zero set. The lower frame bound $\mu_{\rm low,\theta,X}$ is strictly positive except for a measure-zero set of $(x_1,\cdots,x_N)$ in $\RR^{Nd}$. 
\end{lemma}
\begin{proof}

Let $A(\theta,X)$ be the frame operator, i.e.,
\begin{align*}
A(\theta,X)=\sum_{i=1}^{\dim\theta}\begin{pmatrix}
    \nabla_{\theta_i} f(x_1)\\
    \vdots\\
    \nabla_{\theta_i} f(x_N)
\end{pmatrix}\begin{pmatrix}
    \nabla_{\theta_i} f(x_1)\\
    \vdots\\
    \nabla_{\theta_i} f(x_N)
\end{pmatrix}^\top=
\begin{pmatrix}
    \nabla_\theta f(x_1)\\
    \vdots\\
    \nabla_\theta f(x_N)
\end{pmatrix}\begin{pmatrix}
    \nabla_\theta f(x_1)\\
    \vdots\\
    \nabla_\theta f(x_N)
\end{pmatrix}^\top
\end{align*}
We would like to bound $\lambda_{\min}(A)$. Consider any fixed $\theta$ except for a measure-zero set.

Let 
\begin{align*}
    B(\theta,X)&=\sum_{i=1}^{\dim\theta_{\bar{\ell}}}\begin{pmatrix}
    \nabla_{\theta_{\bar{\ell},i}} f(x_1)\\
    \vdots\\
    \nabla_{\theta_{\bar{\ell},i}} f(x_N)
\end{pmatrix}\begin{pmatrix}
    \nabla_{\theta_{\bar{\ell},i}} f(x_1)\\
    \vdots\\
    \nabla_{\theta_{\bar{\ell},i}} f(x_N)
\end{pmatrix}^\top\\
&=\sum_{i=1}^{\dim \theta_{\bar{\ell}}}\begin{pmatrix}
       \nabla_{\theta_{\bar{\ell},i}}\tilde{f}_{\bar{\ell}}(\theta;u_{\bar{\ell}}(x_1))\\
        \vdots\\
        \nabla_{\theta_{\bar{\ell},i}}\tilde{f}_{\bar{\ell}}(\theta;u_{\bar{\ell}}(x_N))
    \end{pmatrix}\begin{pmatrix}
       \nabla_{\theta_{\bar{\ell},i}}\tilde{f}_{\bar{\ell}}(\theta;u_{\bar{\ell}}(x_1))\\
        \vdots\\
        \nabla_{\theta_{\bar{\ell},i}}\tilde{f}_{\bar{\ell}}(\theta;u_{\bar{\ell}}(x_N))
    \end{pmatrix}^\top
\end{align*}
We also denote
\begin{align*}
    \tilde{B}(\theta,U)=\tilde{B}(\theta,(u_{\bar{\ell},1},\cdots,u_{\bar{\ell},N}))\overset{\Delta}{=}B(\theta,X)
\end{align*}

Then obviously, we have
\begin{align*}
    \lambda_{\min}(A)\ge \lambda_{\min}(B)=\lambda_{\min}(\tilde{B}).
\end{align*}
By Lemma~\ref{lem:grad_frame}, the stacked Jacobian with respect to $\theta_{\bar{\ell}}$ has full row rank except for a measure-zero set of $(x_1,\cdots,x_N)$ in $\RR^{Nd}$. Hence $\lambda_{\min}(B)>0$ outside this set.
\end{proof}

\begin{proof}[Proof of Lemma~\ref{lem:mu_low_along_gd}]
The first statement is Lemma~\ref{lem:lower_frame_bound}. The proof of Theorem~\ref{thm:training} removes the countable union of the preimages of the exceptional parameter sets under the GD map, so every iterate $\theta^s$ avoids these sets. Therefore $\mu_{{\rm low},s,X}>0$ for all $s$. Applying Lemma~\ref{lem:gradient_lower_bound} at $\theta=\theta^s$ gives the displayed inequality.
\end{proof}

\subsection{Polynomial Generalized Smoothness of the Loss}
\label{subapp:poly_smoothness_loss}

In this section, we denote $\theta_{\max,i}=\arg\max\{\|\theta\|,\|\theta'\|\}$, and $u_{\max}=\arg\max\{\|u\|,\|u'\|\}$. All the constants $C$'s are $\ge0$ and independent of $\theta,\bar{\varphi},u$; all the $p$'s in the subscripts are in $\NN$. We denote $u_{\ell,i}=u_{\ell,i}(\theta)$ to emphasize the dependency on $\theta$, especially when comparing $u_{\ell,i}$ under different $\theta$'s. Also, note that $S(\cdot)$ has positive coefficients. Therefore $S(\cdot)$ is monotonically increasing on $\RR_{\ge 0}\times\cdots\times\RR_{\ge 0}$, and $$S(a_1,\cdots,a_i,\cdots,a_d)\le S(|a_1|,\cdots,|a_i|,\cdots,|a_d|).$$

Below is the general lemma about the poly-smoothness of the neural network.
\begin{lemma}
\label{lem:loss_function_generalized_smoothness}
    Under Assumption~\ref{assump:generalized_smoothness_detail}, $l(\theta;x_i)$ satisfies generalized smoothness
    \begin{align}
       \|\nabla_\theta l(\theta;x_i)-\nabla_\theta l(\theta';x_i)\|\le S_l(\|\theta_{\max,1}\|,\cdots,\|\theta_{\max,n_\theta}\|)\ \|\theta-\theta'\|,
    \end{align}
    and consequently, 
    \begin{align}
         \|\nabla_\theta\cL(\theta)-\nabla_\theta\cL(\theta')\|\le S(\|\theta_{\max,1}\|,\cdots,\|\theta_{\max,n_\theta}\|) \ \|\theta-\theta'\|
    \end{align}
    where $S_l(\|\theta_{\max,1}\|,\cdots,\|\theta_{\max,n_\theta}\|),S(\|\theta_{\max,1}\|,\cdots,\|\theta_{\max,n_\theta}\|)$ are some polynomials with positive coefficient.
\end{lemma}

\begin{proof}
    The proof follows directly from Proposition~\ref{prop:poly_generalized_properties}.
\end{proof}

\subsubsection{Example Calculation of Polynomial Functions of the Neural Network Objective}

Next, we aim to derive the polynomial function precisely for neural networks with normalization $\tau_\epsilon$ in each layer:
\begin{align}
\label{eqn:NN_normalization}
    u_{0,i}&=x_i;\quad\notag\\
    u_{\ell+1,i}&=u_{\ell,i}+\bar{\varphi}_\ell(\theta_\ell;u_{\ell,i}),\forall\ell=0,\cdots,L-1;\quad \\
    f(\theta;x_i)&=\bar{\varphi}_L(\theta_L;u_{L,i})\notag
\end{align}
where $\theta=(\theta_0,\cdots,\theta_L)$ is the collection of weights with the $\theta_i$ corresponding to the $i$th layer, and $\bar{\varphi}_{\ell}(\theta_\ell;u)=\varphi_{\ell}(\theta_\ell;\tau_\epsilon(u))$ with $\varphi_\ell(\theta_\ell;\cdot):\RR^d\to \RR^d$ and the normalization function $\tau_\epsilon$ (see Section~\ref{sec:preliminary}) for $\ell=0,\cdots,L$. 
Before the precise derivation, we provide a more detailed version of Assumption~\ref{assump:generalized_smoothness_detail} as follows:
\begin{assumption}
\label{assump:app_generalized_smoothness_detail}
Assume for all $\ell=0,\cdots,L$,
\begin{align}
\label{eqn:varphi_upper_bound}
    \|{\varphi}_\ell(\theta;u)\|\le\sum_{t=1}^{n_{0,\ell}} C_{0,\ell,t}\prod_{i=1}^{n_\theta} \|\theta_i\|^{p_{0,\ell,t, i}}\|u\|^{p_{0,\ell,t}}
\end{align}
\begin{align}
\label{eqn:grad_varphi_theta_upper_bound}
    \|\nabla_{\theta_\ell}{\varphi}_\ell(\theta;u)\|\le \sum_{t=1}^{n_{1,\ell}}C_{1,\ell,t}\prod_{i=1}^{n_\theta} \|\theta_i\|^{p_{1,\ell,t, i}}\|u\|^{p_{1,\ell,t}}
\end{align}
\begin{align}
\label{eqn:grad_varphi_u_upper_bound}
    \|\nabla_u{\varphi}_\ell(\theta;u)\|\le \sum_{t=1}^{n_{2,\ell}}C_{2,\ell,t}\prod_{i=1}^{n_\theta} \|\theta_i\|^{p_{2,\ell,t, i}}\|u\|^{p_{2,\ell,t}}
\end{align}
\begin{align}
\label{eqn:bar_varphi_Lipschitz}
    \| \bar{\varphi}_\ell(\theta;u)-\bar{\varphi}_\ell(\theta';u')\|&\le \big(C_{1,1,\ell}\prod_{i=1}^{n_\theta}\|\theta_{\max, i}\|^{p_{1,1,\ell, i}}\|u_{\max}\|^{p_{1,1,\ell}}+C_{1,1,\ell,0}\big)\ \|\theta-\theta'\|\notag\\
    &\qquad+\big(C_{1,2,\ell}\prod_{i=1}^{n_\theta}\|\theta_{\max, i}\|^{p_{1,2,\ell, i}}\|u_{\max}\|^{p_{1,2,\ell}}+C_{1,2,\ell,0}\big)\ \|u-u'\|
\end{align}
\begin{align}
\label{eqn:grad_bar_varphi_theta_Lipschitz}
    \|\nabla_{\theta_\ell} \bar{\varphi}_\ell(\theta;u)-\nabla_{\theta_\ell} \bar{\varphi}_\ell(\theta';u')\|&\le \big(C_{2,1,\ell}\prod_{i=1}^{n_\theta}\|\theta_{\max, i}\|^{p_{2,1,\ell, i}}\|u_{\max}\|^{p_{2,1,\ell}}+C_{2,1,\ell,0}\big)\ \|\theta-\theta'\|\notag\\
    &\qquad+\big(C_{2,2,\ell}\prod_{i=1}^{n_\theta}\|\theta_{\max, i}\|^{p_{2,2,\ell, i}}\|u_{\max}\|^{p_{2,2,\ell}}+C_{2,2,\ell,0}\big)\ \|u-u'\|
\end{align}
\begin{align}
\label{eqn:grad_bar_varphi_u_Lipschitz}
    \|\nabla_{u} \bar{\varphi}_\ell(\theta;u)-\nabla_{u} \bar{\varphi}_\ell(\theta';u')\|&\le \big(C_{3,1,\ell}\prod_{i=1}^{n_\theta}\|\theta_{\max, i}\|^{p_{3,1,\ell, i}}\|u_{\max}\|^{p_{3,1,\ell}}+C_{3,1,\ell,0}\big)\ \|\theta-\theta'\|\notag\\
    &\qquad+\big(C_{3,2,\ell}\prod_{i=1}^{n_\theta}\|\theta_{\max, i}\|^{p_{3,2,\ell, i}}\|u_{\max}\|^{p_{3,2,\ell}}+C_{3,2,\ell,0}\big)\ \|u-u'\|
\end{align}
\end{assumption}

\begin{lemma}
\label{lem:bar_varphi_u_uLipschitz_upper_bound_no_u}
    Under Assumption~\ref{assump:app_generalized_smoothness_detail}, we have that, for all $\ell=0,\cdots,L$,
    \begin{align}    \label{eqn:bar_varphi_upper_bound}
    \|\bar{\varphi}_\ell(\theta;u)\|\le \sum_{t=1}^{n_{0,\ell}} C_{0,\ell,t}\prod_{i=1}^{n_\theta} \|\theta_i\|^{p_{0,\ell,t, i}}
\end{align}
\begin{align}
    \label{eqn:grad_bar_varphi_theta_upper_bound}
    \|\nabla_{\theta_j}\bar{\varphi}_\ell(\theta;u)\|\le \sum_{t=1}^{n_{1,\ell}} C_{1,\ell,t}\prod_{i=1}^{n_\theta} \|\theta_i\|^{p_{1,\ell,t, i}}
\end{align}
\begin{align}
    \label{eqn:grad_bar_varphi_u_upper_bound}
    \|\nabla_u\bar{\varphi}_\ell(\theta;u)\|\le \sum_{t=1}^{n_{2,\ell}} C_{2,\ell,t}\prod_{i=1}^{n_\theta} \|\theta_i\|^{p_{2,\ell,t, i}}
\end{align}
\begin{align}
\label{eqn:u_upper_bound}
    \|u_{\ell,i}\|\le \|x_i\|+\sum_{j=0}^{\ell-1} \sum_{t=1}^{n_{0,j}} C_{0,j,t}\prod_{i=1}^{n_\theta} \|\theta_i\|^{p_{0,j,t, i}}
\end{align}
\begin{align}
\label{eqn:u_Lipschitz_no_u}
    \|u_{\ell,i}(\theta)-u_{\ell,i}(\theta')\|\le \sum_{j=0}^{g(\ell)}C_{u,j,i}\prod_{q=1}^{n_\theta} \|\theta_{\max, q}\|^{p_{u,j,q}}\ \|\theta-\theta'\|
\end{align}
\end{lemma}
\begin{proof}

    By the definition of $\epsilon-$normalization, we have
    \begin{align*}
        \|\tau_\epsilon(u)\|=\left\|\frac{u}{\sqrt{\|u\|^2+\epsilon^2}}\right\|\le 1.
    \end{align*}
    Then we have
    \begin{align*}
    \|\bar{\varphi}_\ell(\theta;u)\|=\|\varphi(\theta;\tau_\epsilon(u))\|\le \sum_{t=1}^{n_{0,\ell}} C_{0,\ell,t}\prod_{i=1}^{n_\theta} \|\theta_i\|^{p_{0,\ell,t, i}}\|\tau_\epsilon(u)\|^{p_{0,\ell,t}}\le \sum_{t=1}^{n_{0,\ell}} C_{0,\ell,t}\prod_{i=1}^{n_\theta} \|\theta_i\|^{p_{0,\ell,t, i}}.
    \end{align*}
    The inequality of $\nabla_{\theta_j}\bar{\varphi}_\ell(\theta;u)$ follows the same idea. For $\nabla_{u}\bar{\varphi}_\ell(\theta;u)$, first consider
        \begin{align*}
        \nabla \tau_\epsilon(u)=\frac{1}{\sqrt{\|u\|^2+\epsilon^2}}I-\frac{1}{(\sqrt{\|u\|^2+\epsilon^2})^3}uu^\top
    \end{align*}
    Then by Weyl's theorem,
    \begin{align*}
        \|\nabla\tau_\epsilon(u)\|\le \frac{1}{\sqrt{\|u\|^2+\epsilon^2}}+\frac{\|u\|^2}{(\sqrt{\|u\|^2+\epsilon^2})^3}\le \frac{1}{\epsilon}
    \end{align*}
    where the equality of the last inequality is achieved at $\|u\|=0$.
    Then
    \begin{align*}
        \|\nabla_{u}\bar{\varphi}_\ell(\theta;u)\|&=\|\nabla_{\tau_\epsilon(u)}{\varphi}_\ell(\theta;\tau_\epsilon(u))\nabla\tau_\epsilon(u)\|\\
        &\le \|\nabla_{\tau_\epsilon(u)}{\varphi}_\ell(\theta;\tau_\epsilon(u))\| \|\nabla\tau_\epsilon(u)\|\\
        &\le \sum_{t=1}^{n_{2,\ell}} C_{2,\ell,t}\prod_{i=1}^{n_\theta} \|\theta_i\|^{p_{2,\ell,t, i}}
    \end{align*}
    where $C_{2,\ell,t}$ depends on $\epsilon$ and the last inequality follows the same idea as the previous two inequalities.

    For the upper bound of $u_{\ell,i}$, note that 
    \begin{align*}
        u_{\ell+1,i}&=u_{\ell,i}+\bar{\varphi}_\ell(\theta;u_{\ell,i})=\cdots=u_{0,i}+\sum_{j=0}^\ell \bar{\varphi}_j(\theta;u_{j,i}).
    \end{align*}
    Then,
    \begin{align*}
        \|u_{\ell,i}\|\le \|x_i\|+\sum_{j=0}^{\ell-1}\|\bar{\varphi}_j(\theta;u_{j,i})\|\le \|x_i\|+\sum_{j=0}^{\ell-1} \sum_{t=1}^{n_{0,j}} C_{0,j,t}\prod_{i=1}^{n_\theta} \|\theta_i\|^{p_{0,j,t, i}}
    \end{align*}
    where the second inequality follows from~\eqref{eqn:bar_varphi_upper_bound}.

    Also,
    \begin{align*}
        \|u_{\ell,i}(\theta)-u_{\ell,i}(\theta')\|&=\|x_0+\sum_{j=0}^{\ell-1} \bar{\varphi}_j(\theta;u_{j,i}(\theta))-x_0-\sum_{j=0}^{\ell-1} \bar{\varphi}_j(\theta';u_{j,i}(\theta'))\|\\
        &\le \sum_{j=0}^{\ell-1}\|\bar{\varphi}_j(\theta;u_{j,i}(\theta))-\bar{\varphi}_j(\theta';u_{j,i}(\theta'))\|\\
        &\le \sum_{j=0}^{\ell-1} \big(C_{1,1,j}\prod_{q=1}^{n_\theta}\|\theta_{\max, q}\|^{p_{1,1,j, q}}\|u_{\max}\|^{p_{1,1,j}}+C_{1,1,j,0}\big)\ \|\theta-\theta'\|\notag\\
    &\qquad+\sum_{j=0}^{\ell-1}\big(C_{1,2,j}\prod_{q=1}^{n_\theta}\|\theta_{\max, q}\|^{p_{1,2,j, q}}\|u_{\max}\|^{p_{1,2,j}}+C_{1,2,j,0}\big)\ \|u_{j,i}(\theta)-u_{j,i}(\theta')\|
    \end{align*}
    where the first inequality follows from~\eqref{eqn:bar_varphi_upper_bound}.

    Then we denote the above inequality as follows
    \begin{align*}
        \underbrace{\|u_{\ell,i}(\theta)-u_{\ell,i}(\theta')\|}_{a_\ell}&\le \underbrace{\sum_{j=0}^{\ell-1} \big(C_{1,1,j}\prod_{q=1}^{n_\theta}\|\theta_{\max, q}\|^{p_{1,1,j, q}}\|u_{\max}\|^{p_{1,1,j}}+C_{1,1,j,0}\big)\ \|\theta-\theta'\|}_{b_\ell}\notag\\
    &\qquad+\sum_{j=0}^{\ell-1}\underbrace{\big(C_{1,2,j}\prod_{q=1}^{n_\theta}\|\theta_{\max, q}\|^{p_{1,2,j, q}}\|u_{\max}\|^{p_{1,2,j}}+C_{1,2,j,0}\big)}_{\lambda_j}\ \underbrace{\|u_{j,i}(\theta)-u_{j,i}(\theta')\|}_{a_j}.
    \end{align*}
    Since $a_0=\|x_i-x_i\|=0$, define $b_0=0\ge a_0$, and then by discrete Gronwall's inequality (Proposition 4.1 in \citet{emmrich1999discrete} with $\theta=0,\tau_j=1$), we have
    \begin{align*}
        a_\ell\le b_\ell+\sum_{j=0}^{\ell-1}\lambda_jb_j\prod_{s=j+1}^{\ell-1}(1+\lambda_s),
    \end{align*}
    namely, 
    \begin{align*}
        &\|u_{\ell,i}(\theta)-u_{\ell,i}(\theta')\|\\&\le \bigg(\sum_{j=0}^{\ell-1} \big(C_{1,1,j}\prod_{q=1}^{n_\theta}\|\theta_{\max, q}\|^{p_{1,1,j, q}}\|u_{\max,j,i}\|^{p_{1,1,j}}+C_{1,1,j,0}\big)\\
        &\qquad+\sum_{j=1}^{\ell-1}\big(C_{1,2,j}\prod_{q=1}^{n_\theta}\|\theta_{\max, q}\|^{p_{1,2,j, q}}\|u_{\max,j,i}\|^{p_{1,2,j}}+C_{1,2,j,0}\big)\\
        &\qquad\times\big(\sum_{r=0}^{j-1} C_{1,1,r}\prod_{q=1}^{n_\theta}\|\theta_{\max, q}\|^{p_{1,1,r, q}}\|u_{\max,r,i}\|^{p_{1,1,r}}+C_{1,1,r,0}\big)\\
        &\qquad\times \prod_{s=j+1}^{\ell-1}\big( 1+C_{1,2,s}\prod_{q=1}^{n_\theta}\|\theta_{\max, q}\|^{p_{1,2,s, q}}\|u_{\max,s,i}\|^{p_{1,2,s}}+C_{1,2,s,0}\big)\bigg)\ \|\theta-\theta'\|\\
        &\le \bigg(\sum_{j=0}^{\ell-1} \big(C_{1,1,j}\prod_{q=1}^{n_\theta}\|\theta_{\max, q}\|^{p_{1,1,j, q}}\big(\sum_{s=0}^{j-1} C_{0,s}\prod_{q=1}^{n_\theta}\|\theta_{\max, q}\|^{p_{0,s, q}}+\|x_i\|\big)^{p_{1,1,s}}+C_{1,1,s,0}\big)\\
        &\qquad+\sum_{j=1}^{\ell-1}\big(C_{1,2,j}\prod_{q=1}^{n_\theta}\|\theta_{\max, q}\|^{p_{1,2,j, q}}\big(\sum_{s=0}^{j-1} C_{0,s}\prod_{q=1}^{n_\theta}\|\theta_{\max, q}\|^{p_{0,s, q}}+\|x_i\|\big)^{p_{1,2,j}}+C_{1,2,j,0}\big)\\
        &\qquad\times\big(\sum_{r=0}^{j-1} C_{1,1,r}\prod_{q=1}^{n_\theta}\|\theta_{\max, q}\|^{p_{1,1,r, q}}\big(\sum_{s=0}^{r-1} C_{0,s}\prod_{q=1}^{n_\theta}\|\theta_{\max, q}\|^{p_{0,s, q}}+\|x_i\|\big)^{p_{1,1,r}}+C_{1,1,r,0}\big)\\
        &\qquad\times \prod_{s=j+1}^{\ell-1}\big( 1+C_{1,2,s}\prod_{q=1}^{n_\theta}\|\theta_{\max, q}\|^{p_{1,2,s, q}}\big(\sum_{t=0}^{s-1} C_{0,t}\prod_{q=1}^{n_\theta}\|\theta_{\max, q}\|^{p_{0,t, q}}+\|x_i\|\big)^{p_{1,2,s}}+C_{1,2,s,0}\big)\bigg)\ \|\theta-\theta'\|\\  &= \sum_{j=0}^{g_3(\ell)}C_{u,j,i}\prod_{q=1}^{n_\theta} \|\theta_{\max, q}\|^{p_{u,j,q}}\ \|\theta-\theta'\|
    \end{align*}
    where the second inequality follows from~\eqref{eqn:u_upper_bound}, $g_3(\ell)$ is some polynomial of $\ell$, and $C_{u,j,i}$ is some constant depending on $\|x_i\|$.
\end{proof}

\begin{corollary}
\label{cor:bar_varphi_and_grad_Lipschitz_no_u}
    Under the same assumptions as Lemma~\ref{lem:bar_varphi_u_uLipschitz_upper_bound_no_u}, we have that, for all $\ell=0,\cdots,L$,
\begin{align}
\label{eqn:bar_varphi_Lipschitz_no_u}
    \| \bar{\varphi}_\ell(\theta;u_{\ell,i}(\theta))-\bar{\varphi}_\ell(\theta';u_{\ell,i}(\theta'))\|\le\sum_{j=0}^{\tilde{g}_3(\ell)}C_{\bar{\varphi},\ell,j,i}\prod_{q=1}^{n_\theta} \|\theta_{\max, q}\|^{p_{\bar{\varphi},\ell,j,q}}\ \|\theta-\theta'\|
\end{align}
\begin{align}
\label{eqn:grad_bar_varphi_theta_Lipschitz_no_u}
    \|\nabla_{\theta_\ell} \bar{\varphi}_\ell(\theta;u_{\ell,i}(\theta))-\nabla_{\theta_\ell} \bar{\varphi}_\ell(\theta';u_{\ell,i}(\theta'))\|\le\sum_{j=0}^{\tilde{g}_3(\ell)}C_{\nabla_\theta\bar{\varphi},\ell,j,i}\prod_{q=1}^{n_\theta} \|\theta_{\max, q}\|^{p_{\nabla_\theta\bar{\varphi},\ell,j,q}}\ \|\theta-\theta'\|
\end{align}
\begin{align}
\label{eqn:grad_bar_varphi_u_Lipschitz_no_u}
    \|\nabla_{u} \bar{\varphi}_\ell(\theta;u_{\ell,i}(\theta))-\nabla_{u} \bar{\varphi}_\ell(\theta';u_{\ell,i}(\theta'))\|\le\sum_{j=0}^{\tilde{g}_3(\ell)}C_{\nabla_u\bar{\varphi},\ell,j,i}\prod_{q=1}^{n_\theta} \|\theta_{\max, q}\|^{p_{\nabla_u\bar{\varphi},\ell,j,q}}\ \|\theta-\theta'\|
\end{align}
where $\tilde{g}_3(\ell)$ is some polynomial of $\ell$, and the $C$'s depends on $\|x_i\|$.
\end{corollary}
\begin{proof}
For the first inequality,
    \begin{align*}
    &\| \bar{\varphi}_\ell(\theta;u_{\ell,i}(\theta))-\bar{\varphi}_\ell(\theta';u_{\ell,i}(\theta'))\|\\
    &\le \big(C_{1,1,\ell}\prod_{q=1}^{n_\theta}\|\theta_{\max, q}\|^{p_{1,1,\ell, q}}\|u_{\max}\|^{p_{1,1,\ell}}+C_{1,1,\ell,0}\big)\ \|\theta-\theta'\|\notag\\
    &\qquad+\big(C_{1,2,\ell}\prod_{q=1}^{n_\theta}\|\theta_{\max, q}\|^{p_{1,2,\ell, q}}\|u_{\max}\|^{p_{1,2,\ell}}+C_{1,2,\ell,0}\big)\ \|u_{\ell,i}(\theta)-u_{\ell,i}(\theta')\|\\
    &\le \big(C_{1,1,\ell}\prod_{q=1}^{n_\theta}\|\theta_{\max, q}\|^{p_{1,1,\ell, q}}\big(\sum_{s=0}^{j-1} C_{0,s}\prod_{q=1}^{n_\theta}\|\theta_{\max, q}\|^{p_{0,s, q}}+\|x_i\|\big)^{p_{1,1,\ell}}+C_{1,1,\ell,0}\big)\ \|\theta-\theta'\|\notag\\
    &\qquad+\big(C_{1,2,\ell}\prod_{q=1}^{n_\theta}\|\theta_{\max, q}\|^{p_{1,2,\ell, q}}\big(\sum_{s=0}^{j-1} C_{0,s}\prod_{q=1}^{n_\theta}\|\theta_{\max, q}\|^{p_{0,s, q}}+\|x_i\|\big)^{p_{1,2,\ell}}+C_{1,2,\ell,0}\big)\\
    &\qquad\qquad\times\sum_{j=0}^{g_3(\ell)}C_{u,j,i}\prod_{q=1}^{n_\theta} \|\theta_{\max, q}\|^{p_{u,j,q}}\ \|\theta-\theta'\|\\
    &=\sum_{j=0}^{\tilde{g}_3(\ell)}C_{\bar{\varphi},\ell,j,i}\prod_{q=1}^{n_\theta} \|\theta_{\max, q}\|^{p_{\bar{\varphi},\ell,j,q}}\ \|\theta-\theta'\|
\end{align*}
where the second inequality follows from~\eqref{eqn:u_upper_bound} and~\eqref{eqn:u_Lipschitz_no_u}, $\tilde{g}_3(\ell)$ is some polynomial of $\ell$, $C_{\bar{\varphi},\ell,j,i}$ depends on $\|x_i\|$. The other two inequalities follow the same idea.
\end{proof}

\begin{lemma}
\label{lem:loss_function_generalized_smoothness_normalization}
    Under Assumption~\ref{assump:app_generalized_smoothness_detail}, $l(\theta;x_i)$ satisfies generalized smoothness
    \begin{align}
       \|\nabla_\theta l(\theta;x_i)-\nabla_\theta l(\theta';x_i)\|\le \sum_{j=1}^{n_\theta}\sum_{\ell\in\cJ_j} \sum_{j=0}^{{g}_4(\ell)}C_{l,\ell,j,i}\prod_{q=1}^{n_\theta} \|\theta_{\max, q}\|^{p_{l,\ell,j,q}}\ \|\theta-\theta'\|
    \end{align}
    where $g_4(\ell)$ is some polynomial of $\ell$. Consequently, 
    \begin{align}
         \|\nabla_\theta\cL(\theta)-\nabla_\theta\cL(\theta')\|\le S(\|\theta_{\max,1}\|,\cdots,\|\theta_{\max,n_\theta}\|) \ \|\theta-\theta'\|
    \end{align}
    where $S(\|\theta_{\max,1}\|,\cdots,\|\theta_{\max,n_\theta}\|)=\frac{1}{N}\sum_{i=1}^N\sum_{j=1}^{n_\theta}\sum_{\ell\in\cJ_j} \sum_{j=0}^{{g}_4(\ell)}C_{l,\ell,j,i}\prod_{q=1}^{n_\theta} \|\theta_{\max, q}\|^{p_{l,\ell,j,q}}$.
\end{lemma}

\begin{proof}
From the calculations at the beginning of this section, we know
\begin{align*}
    \nabla_{\theta_j}l(\theta;x_i)&=\nabla_f \,l(\theta;x_i)\nabla_{\theta_\ell}f(x_i)\\
    &=\sum_{\ell\in\cJ_j}\nabla_f \,l(\theta;x_i)\nabla_u\bar{\varphi}_L(\theta;u_{L,i})(I+\nabla_u\bar{\varphi}_{L-1}(\theta;u_{L-1,i}))\cdots (I+\nabla_u\bar{\varphi}_{\ell+1}(\theta;u_{\ell+1,i}))\nabla_{\theta_j}\bar{\varphi}_{\ell}(\theta;u_{\ell,i})
\end{align*}
Then
\begin{align*}
    &\| \nabla_{\theta_j}l(\theta;x_i)- \nabla_{\theta_j}l(\theta';x_i)\|\\
    &\le \sum_{\ell\in\cJ_j}\Big\|\nabla_f \,l(\theta;x_i)\nabla_u\bar{\varphi}_L(\theta;u_{L,i})(I+\nabla_u\bar{\varphi}_{L-1}(\theta;u_{L-1,i}))\cdots (I+\nabla_u\bar{\varphi}_{\ell+1}(\theta;u_{\ell+1,i}))\nabla_{\theta_j}\bar{\varphi}_{\ell}(\theta;u_{\ell,i})\\
   &\qquad -\nabla_f \,l(\theta';x_i)\nabla_u\bar{\varphi}_L(\theta';u_{L,i})(I+\nabla_u\bar{\varphi}_{L-1}(\theta';u_{L-1,i}))\cdots (I+\nabla_u\bar{\varphi}_{\ell+1}(\theta';u_{\ell+1,i}))\nabla_{\theta_j}\bar{\varphi}_{\ell}(\theta';u_{\ell,i})\Big\|\\
   &= \sum_{\ell\in\cJ_j}\Big\|(\bar{\varphi}_L(\theta;u_{L,i})-y_i)^\top\nabla_u\bar{\varphi}_L(\theta;u_{L,i})(I+\nabla_u\bar{\varphi}_{L-1}(\theta;u_{L-1,i}))\cdots (I+\nabla_u\bar{\varphi}_{\ell+1}(\theta;u_{\ell+1,i}))\nabla_{\theta_j}\bar{\varphi}_{\ell}(\theta;u_{\ell,i})\\
   &\qquad -(\bar{\varphi}_L(\theta';u_{L,i})-y_i)^\top\nabla_u\bar{\varphi}_L(\theta';u_{L,i})(I+\nabla_u\bar{\varphi}_{L-1}(\theta';u_{L-1,i}))\cdots (I+\nabla_u\bar{\varphi}_{\ell+1}(\theta';u_{\ell+1,i}))\nabla_{\theta_j}\bar{\varphi}_{\ell}(\theta';u_{\ell,i})\Big\|\\
   &\le \sum_{\ell\in\cJ_j} (\|\bar{\varphi}_L(\theta;u_{L,i})\|+\|y_i\|)\|\nabla_u\bar{\varphi}_L(\theta;u_{L,i})\| (1+\|\nabla_u\bar{\varphi}_{L-1}(\theta;u_{L-1,i})\|)\cdots\\
   &\qquad\qquad\times(1+\|\nabla_u\bar{\varphi}_{\ell+1}(\theta;u_{\ell+1,i}))\|)\|\nabla_{\theta_j}\bar{\varphi}_{\ell}(\theta;u_{\ell,i})-\nabla_{\theta_j}\bar{\varphi}_{\ell}(\theta';u_{\ell,i})\| \\
   &\qquad+\cdots+ (\|\bar{\varphi}_L(\theta;u_{L,i})\|+\|y_i\|)\|\nabla_u\bar{\varphi}_L(\theta;u_{L,i})\| (1+\|\nabla_u\bar{\varphi}_{L-1}(\theta;u_{L-1,i})\|)\cdots\\
&\qquad\qquad\times\|\nabla_u\bar{\varphi}_{s}(\theta;u_{s,i})-\nabla_u\bar{\varphi}_{s}(\theta';u_{s,i})\|\cdots (1+\|\nabla_u\bar{\varphi}_{\ell+1}(\theta';u_{\ell+1,i}))\|\nabla_{\theta_j}\bar{\varphi}_{\ell}(\theta';u_{\ell,i})\|\\
&\qquad+\cdots+\|\bar{\varphi}_L(\theta;u_{L,i})-\bar{\varphi}_L(\theta';u_{L,i})\| \|\nabla_u\bar{\varphi}_L(\theta';u_{L,i})\|(1+\|\nabla_u\bar{\varphi}_{L-1}(\theta';u_{L-1,i})\|)\cdots \\
&\qquad\qquad\times(1+\|\nabla_u\bar{\varphi}_{\ell+1}(\theta';u_{\ell+1,i})\|)\|\nabla_{\theta_j}\bar{\varphi}_{\ell}(\theta';u_{\ell,i})\|\\
&\le \sum_{\ell\in\cJ_j} \sum_{j=0}^{{g}_4(\ell)}C_{l,\ell,j,i}\prod_{q=1}^{n_\theta} \|\theta_{\max, q}\|^{p_{l,\ell,j,q}}\ \|\theta-\theta'\|
\end{align*}
where the second inequality follows from the triangle inequality, and the last inequality follows from Lemma~\ref{lem:bar_varphi_u_uLipschitz_upper_bound_no_u} and Corollary~\ref{cor:bar_varphi_and_grad_Lipschitz_no_u}; $C_{l,\ell,j,i}$ is a constant that depends on $\|x_i\|$ and $\|y_i\|$, and $g_4(\ell)$ is a polynomial of $\ell$.

Then 
\begin{align*}
        \|\nabla_\theta l(\theta;x_i)-\nabla_\theta l(\theta';x_i)\|&\le \sum_{j=1}^{n_\theta}\| \nabla_{\theta_j}l(\theta;x_i)- \nabla_{\theta_j}l(\theta';x_i)\|\\
        &\le \sum_{j=1}^{n_\theta}\sum_{\ell\in\cJ_j} \sum_{j=0}^{{g}_4(\ell)}C_{l,\ell,j,i}\prod_{q=1}^{n_\theta} \|\theta_{\max, q}\|^{p_{l,\ell,j,q}}\ \|\theta-\theta'\|
    \end{align*}
    and
     \begin{align*}
         \|\nabla_\theta\cL(\theta)-\nabla_\theta\cL(\theta')\|&\le\frac{1}{N}\sum_{i=1}^N \|\nabla_\theta l(\theta;x_i)-\nabla_\theta l(\theta';x_i)\|\\
         &\le \frac{1}{N}\sum_{i=1}^N\sum_{j=1}^{n_\theta}\sum_{\ell\in\cJ_j} \sum_{j=0}^{{g}_4(\ell)}C_{l,\ell,j,i}\prod_{q=1}^{n_\theta} \|\theta_{\max, q}\|^{p_{l,\ell,j,q}}\ \|\theta-\theta'\|
    \end{align*}

\end{proof}

\subsection{Analytic Measure-Zero Tools and Supplementary Lemmas}
\label{subapp:measure_theory_differential_topology}

This section collects the analytic measure-zero tools used to justify the generic nondegeneracy arguments in the proof. We first state the basic results on zero sets, preimages, and images of null sets, and then prove the supplementary lemmas used in the convergence analysis.

\subsubsection{Basic Theorems}
We first state several standard results that will be applied in the supplementary lemmas below.

\begin{theorem}[\citet{mityagin2015zero}]
\label{thm:zero_set_analytic_function}
    Let $f(x)$ be a real analytic function on (a connected open domain $U$ of) $\RR^d$. If $f$ is not identically zero, then its zero set
    \begin{align*}
        f^{-1}(0)=\{x\in U:f(x)=0\}
    \end{align*}
 has measure zero, i.e., $\mathrm{Leb}(f^{-1}(0))=0$.
\end{theorem}

\begin{theorem}[Theorem G.3 in \citet{wang2022large}]
\label{thm:inverse_map_null_set}Let $f:\RR^d\to\RR^d$ and $f\in\cC^1$. If the set of critical points of $f$ is a null-set, i.e.,
\begin{align*}
\cL(\{x\in\RR^d:\nabla f(x)\ \mathrm{is\ not\ invertible} \})=0,
\end{align*}
then $\cL(f^{-1}(B))=0$ for any null-set B.
\end{theorem}

\begin{theorem}[Lemma 3 in \citet{rader1973nice}]
\label{thm:image_measure_zero}
    Consider a function $f:U\subseteq\RR^d\to\RR^d$ that is pointwise Lipschitz, i.e., for any $x\in U$, there exists a neighborhood $U_x$ of $x$ and a constant $L_x$, s.t.
    \begin{align*}
        \|f(x)-f(y)\|\le L_x\|x-y\|,\ \forall y\in U_x.
    \end{align*}
    Then its image of a measure-zero set is still measure zero. 
\end{theorem}

\begin{theorem}[Preimage Theorem]
\label{thm:preimage_theorem}
    If $y$ is a regular value of a smooth map $f : X \to Y$, then the preimage
$f^{-1}(y)$ is a submanifold of $X$, with $\dim f^{-1}(y) = \dim X-\dim Y $.
\end{theorem}

\subsubsection{Supplementary Lemmas of the Proof}

We then show the supplementary lemmas for the convergence of neural network~\eqref{eqn:NN} by applying the above theorems.

{
\begin{lemma}
\label{lem:full_rank_jacobian}
    Let $f(x,y):\RR^{m}\times \RR^{n_2}\to \RR^m$ be an analytic function. Fix either $y=y_0$ or $x=x_0$. Then the following statements hold:
    \begin{enumerate}
        \item For any $\alpha\in\RR$ except for a finite set of points $E_{y_0}$ (or $E_{x_0}$), the Jacobian w.r.t. $x$ of $\alpha x+f(x,y)$ is full rank for any free variable $x$ or $y$ except for a measure-zero set.
        \item Similarly, for any $\eta\in\RR$ except for a finite set of points $1/E_{y_0}\cup\{0\}$ (or $1/E_{x_0}\cup\{0\}$), the Jacobian w.r.t. $x$ of $x+\eta f(x,y)$ is full rank for any free variable except for a measure-zero set.
    \end{enumerate}
\end{lemma}
\begin{proof}
    First fix $y=y_0$. Consider
    $$
        H_{\alpha,y_0}(x):=\det\big(\alpha I+\nabla_x f(x,y_0)\big),
    $$
    and define the set
    $$
        E_{y_0}:=\{\alpha\in\RR:H_{\alpha,y_0}(\cdot)\equiv0\}
        =\{\alpha\in\RR:H_{\alpha,y_0}(x)=0,\ \forall x\in\RR^m\}.
    $$
    Namely, $E_{y_0}$ is defined by the condition that $H_{\alpha,y_0}$ vanishes identically as a function of $x$, and therefore must be included in the set where $H_{\alpha,y_0}(x_1)=0$ for some fixed $x_1\in\RR^m$.

    Next, for any fixed $x_1\in\RR^m$, the function $H_{\alpha,y_0}(x_1)$ is a monic polynomial in $\alpha$ of degree $m$. Since $E_{y_0}$ is contained in the zero set of this polynomial, we have $E_{y_0}$ is finite. For every $\alpha\notin E_{y_0}$, the function $H_{\alpha,y_0}(x)$ is a nonzero analytic function of $x$. Therefore, by Theorem~\ref{thm:zero_set_analytic_function},
    \begin{align*}
        \mathrm{Leb}_{m}\big(\{x\in\RR^m:H_{\alpha,y_0}(x)=0\}\big)=0.
    \end{align*}
    Therefore, except for a measure-zero set of $x$, $\det(\alpha I+\nabla_x f(x,y_0))\ne0$, and the Jacobian w.r.t. $x$ of $\alpha x+f(x,y_0)$ is full rank. This proves the conclusion for fixed $y=y_0$.
    
    The same argument applies when $x=x_0$ is fixed. Define
    $$
        H_{\alpha,x_0}(y):=\det\big(\alpha I+\nabla_x f(x_0,y)\big),
    $$
    and
    $$
        E_{x_0}:=\{\alpha\in\RR:H_{\alpha,x_0}(\cdot)\equiv0\}
        =\{\alpha\in\RR:H_{\alpha,x_0}(y)=0,\ \forall y\in\RR^{n_2}\}.
    $$
    We similarly have $ E_{x_0}$ also contains finite elements.
    For every $\alpha\notin E_{x_0}$, $H_{\alpha,x_0}(y)$ is a nonzero analytic function of $y$, and Theorem~\ref{thm:zero_set_analytic_function} gives
    \begin{align*}
        \mathrm{Leb}_{n_2}\big(\{y\in\RR^{n_2}:H_{\alpha,x_0}(y)=0\}\big)=0.
    \end{align*}
    Thus the Jacobian w.r.t. $x$ of $\alpha x+f(x,y)$ is full rank for all free variables except for a measure-zero set.

    The above arguments also apply to the map $x+\eta f(x,y)$ with $\eta=1/\alpha$ and $\alpha\notin E_{x_0}\cup\{0\}$ or $\alpha\notin E_{y_0}\cup\{0\}$.

\end{proof}

}

\begin{lemma}
\label{lem:preimage_u_measure_zero}
    Fix some $\theta$ and assume that each residual map
$u\mapsto u+\varphi_\ell(\theta_\ell;u)$ has full-rank Jacobian except on a measure-zero set. Under Assumption~\ref{assump:analytic},
    \begin{enumerate}
        \item Let $u_{\ell}=u_{\ell}(x):\RR^d\to \RR^d$. The preimage of $u_\ell$ for any measure-zero set in $\RR^{d}$ is a measure-zero set in $\RR^{d}$, for all $\ell=0,\cdots,L-1$.
        \item Let $U_{\ell}(x_1,\cdots,x_N)=\begin{pmatrix}
        u_{\ell}(x_1)\\\vdots u_{\ell}(x_N)
    \end{pmatrix}:\RR^{Nd}\to\RR^{Nd}$. The preimage of $U_\ell$ for any measure-zero set in $\RR^{Nd}$ is a measure-zero set in $\RR^{Nd}$, for all $\ell=0,\cdots,L-1$.
    \end{enumerate}
\end{lemma}
\begin{proof}
The proof is by induction on $\ell$. The base case is $u_0(x)=x$, whose Jacobian is the identity. Suppose the critical point set of $u_\ell$ is measure zero. Then, by Theorem~\ref{thm:inverse_map_null_set}, the preimage of $u_\ell$ for any measure zero set is a null set.
    Consider 
\begin{align*}
    u_{\ell+1}=u_{\ell+1}(x),
\end{align*}
where its Jacobian is
\begin{align*}
    \nabla u_{\ell+1}(x)=(I+\nabla_u\bar{\varphi}_{\ell}(u_{\ell}))\cdots(I+\nabla_u\bar{\varphi}_{0}(u_{0})).
\end{align*}
We would like to prove that the critical point set of $u_{\ell+1}$ is measure zero, i.e.,
\begin{align*}
    \{x\in\RR^d:\nabla u_{\ell+1}(x)\text{ is not full rank}\}.
\end{align*}

Note that
\begin{align*}
    \nabla u_{\ell+1}(x)=(I+\nabla_u\bar{\varphi}_{\ell}(u_{\ell}))\nabla u_{\ell}(x),
\end{align*}
and thus
\begin{align*}
    &\{x\in\RR^d:\nabla u_{\ell+1}(x)\text{ is not full rank}\}\\
    &=\{x\in\RR^d:I+\nabla_u\bar{\varphi}_{\ell}(u_{\ell}(x))\text{ is not full rank}\}\cup \{x\in\RR^d:\nabla u_{\ell}(x)\text{ is not full rank}\}.
\end{align*}

By the induction hypothesis, the set
\begin{align*}
    \mathrm{Leb}_d\big(\{x\in\RR^d:\nabla u_{\ell}(x)\text{ is not full rank}\}\big)=0.
\end{align*}

Also, by the assumption on the residual map, 
\begin{align*}
     \mathrm{Leb}_d\big(\{u\in\RR^d:I+\nabla_u\bar{\varphi}_{\ell}(u)\text{ is not full rank}\}\big)=0,
\end{align*}
and then by Theorem~\ref{thm:inverse_map_null_set}, the preimage of the above measure zero set is still measure zero, i.e.,
\begin{align*}
     \mathrm{Leb}_d\big(\{x\in\RR^d:I+\nabla_u\bar{\varphi}_{\ell}(u_{\ell}(x))\text{ is not full rank}\}\big)=0.
\end{align*}

Thus the critical point set of $u_{\ell+1}$ is measure zero, i.e.
\begin{align*}
    \mathrm{Leb}_d\big(\{x\in\RR^d:\nabla u_{\ell+1}(x)\text{ is not full rank}\}\big)\le 0+0=0.
\end{align*}
By Theorem~\ref{thm:inverse_map_null_set}, the first conclusion holds for $u_{\ell+1}$.

For the second statement, the Jacobian of $U_\ell$ is block diagonal. In particular, 
\begin{align*}
    \nabla U_{\ell+1}(x_1,\cdots,x_N)&=\begin{pmatrix}
        \nabla u_{\ell+1}(x_1) & & \\
        & \ddots & \\
        & & \nabla u_{\ell+1}(x_N)
    \end{pmatrix}\\
    &=\begin{pmatrix}
        I+\nabla_u\bar{\varphi}_{\ell}(u_{\ell}(x_1)) & & \\
        & \ddots & \\
        & & I+\nabla_u\bar{\varphi}_{\ell}(u_{\ell}(x_N))
    \end{pmatrix} \nabla U_{\ell}(x_1,\cdots,x_N)\\
    &=\left(I_{Nd}+\begin{pmatrix}
        \nabla_u\bar{\varphi}_{\ell}(u_{\ell}(x_1)) & & \\
        & \ddots & \\
        & & \nabla_u\bar{\varphi}_{\ell}(u_{\ell}(x_N))
    \end{pmatrix}\right)\nabla U_{\ell}(x_1,\cdots,x_N).
\end{align*}
Thus the critical set of $U_{\ell+1}$ is contained in the finite union of the critical sets of the component maps $u_{\ell+1}(x_i)$, which has measure zero in $\RR^{Nd}$. Applying Theorem~\ref{thm:inverse_map_null_set} gives the second statement.
\end{proof}

\section{Interpretation of Convergence Theorem}
\label{app:interpretation_convergence_theorem}

\begin{proof}[Proof of Corollary~\ref{cor:Gaussian_learning_rate_order}]
Let
$$
    Z:=\sum_{j=1}^{L_{\rm wm}}\|\theta_j^0\|_F^2.
$$
Since the biases are zero, $Z=\|\theta^0\|^2$. For one weight matrix, Xavier initialization gives
$$
    \mathrm{Var}\big((\theta_j^0)_{ab}\big)=\frac{2}{d_{j,\mathrm{in}}+d_{j,\mathrm{out}}},
$$
and therefore
$$
    \EE\|\theta_j^0\|_F^2
    =d_{j,\mathrm{in}}d_{j,\mathrm{out}}\cdot \frac{2}{d_{j,\mathrm{in}}+d_{j,\mathrm{out}}}.
$$
Since
$$
    \min\{d_{j,\mathrm{in}},d_{j,\mathrm{out}}\}
    \le
    \frac{2d_{j,\mathrm{in}}d_{j,\mathrm{out}}}{d_{j,\mathrm{in}}+d_{j,\mathrm{out}}}
    \le
    2\min\{d_{j,\mathrm{in}},d_{j,\mathrm{out}}\},
$$
we have
$$
    r_j
    \le
    \EE\|\theta_j^0\|_F^2
    \le
    2r_j.
$$
Summing over weight matrices yields
$$
    \sum_{j=1}^{L_{\rm wm}}r_j
    \le
    \EE Z
    \le
    2\sum_{j=1}^{L_{\rm wm}}r_j.
$$

Let
$$
    S:=\sum_{j=1}^{L_{\rm wm}}r_j.
$$
Under Xavier-normal initialization, write
$$
    (\theta_j^0)_{ab}=\sqrt{\frac{2}{d_{j,\mathrm{in}}+d_{j,\mathrm{out}}}}\,g_{jab},
    \qquad
    g_{jab}\sim \mathcal N(0,1)
$$
independently. Then $Z$ is a weighted chi-square random variable,
$$
    Z=\sum_{j=1}^{L_{\rm wm}}\sum_{a,b}
    \frac{2}{d_{j,\mathrm{in}}+d_{j,\mathrm{out}}}\,g_{jab}^2.
$$
Let $a_{jab}:=2/(d_{j,\mathrm{in}}+d_{j,\mathrm{out}})$. Since $a_{jab}\le 1$, we have
$$
    \sum_{j,a,b}a_{jab}^2\le \sum_{j,a,b}a_{jab}=\EE Z,
    \qquad
    \max_{j,a,b}a_{jab}\le 1.
$$
The standard concentration inequality for weighted chi-square variables gives a universal constant $c>0$ such that, for all $t>0$,
$$
    \PP\left(|Z-\EE Z|\ge t\right)
    \le
    2\exp\left(
    -c\min\left\{
    \frac{t^2}{\sum_{j,a,b}a_{jab}^2},
    \frac{t}{\max_{j,a,b}a_{jab}}
    \right\}
    \right).
$$
Taking $t=\EE Z/2$ yields
$$
    \PP\left(|Z-\EE Z|\ge \frac{1}{2}\EE Z\right)
    \le
    2\exp(-c_{\rm G}\EE Z)
    \le
    2\exp(-c_{\rm G}S)
$$
after decreasing the universal constant $c_{\rm G}>0$ if necessary. By Assumption~\ref{assump:example_setting},
$$
    c_{\rm dim}rL_{\rm wm}\le S\le rL_{\rm wm},
    \qquad
    L_{\rm wm}=c_{\rm wm}L.
$$
Therefore, with probability at least
$$
    1-2\exp(-c_{\rm G}S)
    \ge
    1-2\exp(-c_{\rm G}c_{\rm dim}c_{\rm wm}rL),
$$
we have
$$
    \frac{1}{2}S\le Z\le 3S.
$$
Equivalently, with probability at least $1-2\exp(-c_{\rm G}c_{\rm dim}c_{\rm wm}rL)$, we have $\|\theta^0\|=\Theta(\sqrt{rL})$.

Then, under the same event,
$$
    \|\theta^0-\theta^*\|^2
    \le 2\|\theta^0\|^2+2\|\theta^*\|^2
    \le 2(1+c_*^2)\|\theta^0\|^2
    =\cO(rL).
$$
Let
\begin{align*}
    \cS&:=\{\theta:\|\theta\|=\cO(\sqrt{rL}),\ \|\nabla\cL(\theta)\|>\epsilon_\cL\},\\
    \rho_{\cS}
    &:=
    \sup_{\theta\in\cS}
    \frac{-\ip{\nabla\cL(\theta)}{\theta-\theta^*}}{\|\nabla\cL(\theta)\|^2}.
\end{align*}
Then $\rho=\rho_{\cS}$ is a valid dissipativity constant on $\cS$. By Cauchy-Schwarz and $\epsilon_\cL\ge c_\epsilon\sqrt{rL}^{-1}$ in Assumption~\ref{assump:example_setting},
\begin{align*}
    \rho_{\cS}
    \le
    \sup_{\theta\in\cS}\frac{\|\theta-\theta^*\|}{\|\nabla\cL(\theta)\|}=
    \frac{\cO(\sqrt{rL})}{\epsilon_\cL}
    =
    \cO(rL),
\end{align*}
and therefore $\rho=\cO(rL)$. Since $\delta=\Theta(1)$ and $\cL(\theta^0)\le C_0$,
\begin{align*}
    \max\left\{\frac{4\rho+2}{\delta},0\right\}\cL(\theta^0)
    =
    \cO(rL),
\end{align*}

Now, consider $R$. Combining the discussions above, we obtain that 
$$
    R^2
    =
    \|\theta^{0}-\theta^{*}\|^2+\max\left\{\frac{4\rho+2}{\delta},0\right\}\cL(\theta^0)
    =\cO(rL),
$$
and hence $R=\cO(\sqrt{rL})$.

It remains to bound $\mathsf{L}$. The constant $\mathsf{L}$ in Theorem~\ref{thm:training} is obtained by applying a fixed polynomial with positive coefficients to quantities such as $\|\theta_i^*\|+R$ and $\max_{\theta\in B_{\theta^*}(R)}\|\nabla\cL(\theta)\|$. 

By Assumption~\ref{assump:example_setting},
$$
    \max_{\theta\in B_{\theta^*}(R)}\|\nabla\cL(\theta)\|
    \le C_{\rm grad}(1+\sqrt{rL})^{p_{\rm grad}}.
$$

Therefore there exists a finite exponent $q_{\mathsf{L}}>0$, depending on this polynomial and on $p_{\rm grad}$ but not on $r$ or $L$, such that
$$
    \mathsf{L}
    =
    \cO\left((1+\sqrt{rL})^{q_{\mathsf{L}}}\right).
$$
Consequently,
$$
    \frac{2-\delta}{\mathsf{L}}
    =
    \Omega\left(\frac{1}{(1+\sqrt{rL})^{q_{\mathsf{L}}}}\right).
$$
\end{proof}

\begin{proof}[Proof of Corollary~\ref{cor:uniform_learning_rate_order}]
It remains to justify the uniform version. The expectation calculation above is unchanged, so with
$$
    S:=\sum_{j=1}^{L_{\rm wm}}r_j,
    \qquad
    Z:=\|\theta^0\|^2,
$$
we still have
$$
    S\le \EE Z\le 2S.
$$
Next, each squared entry satisfies
$$
    0\le ((\theta_j^0)_{ab})^2\le \frac{6}{d_{j,\mathrm{in}}+d_{j,\mathrm{out}}}.
$$
Hoeffding's inequality for the independent bounded variables $((\theta_j^0)_{ab})^2$ gives, for any $t>0$,
$$
    \PP\left(
    \left|
    Z-\EE Z
    \right|\ge t
    \right)
    \le
    2\exp\left(
    -\frac{2t^2}{
    \sum_{j=1}^{L_{\rm wm}}d_{j,\mathrm{in}}d_{j,\mathrm{out}}
    \left(\frac{6}{d_{j,\mathrm{in}}+d_{j,\mathrm{out}}}\right)^2}
    \right).
$$
Since
$$
    \frac{d_{j,\mathrm{in}}d_{j,\mathrm{out}}}
    {(d_{j,\mathrm{in}}+d_{j,\mathrm{out}})^2}\le \frac{1}{4},
$$
the denominator in the above Hoeffding's inequality is at most $9L_{\rm wm}$. Hence
$$
    \PP\left(
    \left|
    Z-\EE Z
    \right|\ge t
    \right)
    \le
    2\exp\left(-\frac{2t^2}{9L_{\rm wm}}\right).
$$
Taking $t=S/2$ gives
$$
    \PP\left(
    \left|Z-\EE Z\right|\ge \frac{1}{2}S
    \right)
    \le
    2\exp\left(-\frac{S^2}{18L_{\rm wm}}\right).
$$
The complement of this event is
$$
    \frac{1}{2}S\le Z\le 3S.
$$
By Assumption~\ref{assump:example_setting}, $S\ge c_{\rm dim}rL_{\rm wm}$ and $L_{\rm wm}=c_{\rm wm}L$. Hence
$$
    \PP\left(
    \frac{1}{2}S\le \|\theta^0\|^2\le 3S
    \right)
    \ge
    1-2\exp\left(-\frac{c_{\rm dim}^2c_{\rm wm}r^2L}{18}\right).
$$
Thus $\|\theta^0\|^2=\Theta(rL)$ and $\|\theta^0\|=\Theta(\sqrt{rL})$ under the uniform initialization as well. The remaining estimates for $R$, $\mathsf L$, and $(2-\delta)/\mathsf L$ then follow from the same deterministic argument used in Corollary~\ref{cor:Gaussian_learning_rate_order}.
\end{proof}

\section{Global Minimum Under Generic Nondegeneracy}
\label{app:global_min_generic_nondegeneracy}

It suffices to use the set $\cB_{\theta,0}$ defined in the proof of Theorem~\ref{thm:training} instead of $\cB_\theta$. We restate the theorem below:
\begin{theorem}
\label{thm:global_min_generic_nonlinearity_app}
Consider the same assumptions as Theorem~\ref{thm:training}. For any $\theta^*\notin\cB_{\theta,0}$, if $\theta^*$ is a stationary point of $\cL(\theta)$, then $\theta^*$ is a global minimizer of $\cL(\theta)$.
\end{theorem}

\begin{proof}[Proof of Theorem~\ref{thm:global_min_generic_nonlinearity} and \ref{thm:global_min_generic_nonlinearity_app}]
By construction in the proof of Theorem~\ref{thm:training}, the set $\cB_\theta$ contains the parameter values where Assumption~\ref{assump:architecture} or Lemmas~\ref{lem:gradient_lower_bound},~\ref{lem:lower_frame_bound} may fail in $\cB_{\theta,0}$ (see details in the proof of Theorem~\ref{thm:training}). Thus $\theta^*\notin\cB_\theta$ is sufficient for the generic full-rank argument below.

By Assumption~\ref{assump:architecture}, for fixed $\theta^*\notin\cB_\theta$, there exists a tuple $U=(u_1,\cdots,u_N)\in\RR^{Nd}$ such that
$$
    \begin{pmatrix}
       \nabla_{\theta_{\bar{\ell}}}\tilde{f}_{\bar{\ell}}(\theta^*;u_1)\\
        \vdots\\
        \nabla_{\theta_{\bar{\ell}}}\tilde{f}_{\bar{\ell}}(\theta^*;u_N)
    \end{pmatrix}
$$
has full row rank $Nd$. Therefore at least one $Nd\times Nd$ minor determinant of this matrix is not identically zero as a function of $U$. Since the network maps are analytic under Assumption~\ref{assump:analytic}, this minor determinant is a real-analytic function of $U$. By Theorem~\ref{thm:zero_set_analytic_function}, its zero set has Lebesgue measure zero. Hence the corresponding stacked Jacobian is full row rank for all $U$ except for a measure-zero set.

Passing from hidden states $U$ back to input data $X$ through the layer map gives the same conclusion for all $X$ except for a measure-zero set; this is precisely the argument formalized in Lemmas~\ref{lem:grad_frame} and~\ref{lem:lower_frame_bound}. Define the stacked output residual and stacked output Jacobian by
$$
    r(\theta^*)=
    \begin{pmatrix}
        f(\theta^*;x_1)-y_1\\
        \vdots\\
        f(\theta^*;x_N)-y_N
    \end{pmatrix},
    \qquad
    J_X(\theta^*)=
    \begin{pmatrix}
        \nabla_\theta f(\theta^*;x_1)\\
        \vdots\\
        \nabla_\theta f(\theta^*;x_N)
    \end{pmatrix}.
$$
Consequently,
$$
    \mu_{{\rm low},\theta^*,X}
    =
    \lambda_{\min}\left(J_X(\theta^*)J_X(\theta^*)^\top\right)>0
$$
for all such $X$. Thus $J_X(\theta^*)$ has full row rank. By the chain rule for the square loss,
$$
    \nabla_\theta\cL(\theta^*;X)=\frac{1}{N}J_X(\theta^*)^\top r(\theta^*).
$$
Since $\theta^*$ is stationary, we have
$$
    J_X(\theta^*)^\top r(\theta^*)=0.
$$
Multiplying by $r(\theta^*)^\top J_X(\theta^*)$ gives
$$
    r(\theta^*)^\top J_X(\theta^*)J_X(\theta^*)^\top r(\theta^*)=0.
$$
Because $J_X(\theta^*)$ has full row rank, the matrix $J_X(\theta^*)J_X(\theta^*)^\top$ is positive definite. Hence the last display implies $r(\theta^*)=0$. Therefore
$$
    \cL(\theta^*;X)=\frac{1}{2N}\|r(\theta^*)\|^2=0.
$$
Since the square loss is nonnegative for every $\theta$, zero is the global minimum value, and thus $\theta^*$ is a global minimizer.
\end{proof}

\bibliographystyle{plainnat}
\bibliography{ref}

\end{document}

%% file: commands.tex
\usepackage{amsmath,  amsfonts, amsbsy, latexsym}

\usepackage{cleveref}
\crefname{equation}{}{}
\crefname{enumi}{}{} 
\crefname{assumption}{Assumption}{Assumptions}
\Crefname{assumption}{Assumption}{Assumptions}

\makeatletter
\@ifpackageloaded{jmlr2e}{\let\ifjmlrstyle\iftrue}{\let\ifjmlrstyle\iffalse}
\makeatother
\ifjmlrstyle
\else
\usepackage{amsthm}
\fi
\usepackage{MnSymbol}
\usepackage{url}
\usepackage{tikz}
\usetikzlibrary{arrows.meta,positioning,calc}
\usepackage{natbib}
\usepackage{bbm}
\usepackage{tablefootnote}
\usepackage{mathtools}
\usepackage{enumitem}
\setlist{topsep=0pt, leftmargin=*}
\allowdisplaybreaks

\usepackage{parskip}
\newcommand{\ip}[2]{\left\langle #1, #2 \right\rangle}

\newcommand{\cC}{\mathcal{C}}

\newcommand{\cO}{\mathcal{O}}

\newcommand{\cS}{\mathcal{S}}

\newcommand{\cL}{\mathcal{L}}

\newcommand{\cI}{\mathcal{I}}
\newcommand{\cJ}{\mathcal{J}}
\newcommand{\cD}{\mathcal{D}}

\newcommand{\cB}{\mathcal{B}}

\newcommand{\PP}{\mathbb{P}}
\newcommand{\EE}{\mathbb{E}}
\newcommand{\RR}{\mathbb{R}}
\newcommand{\NN}{\mathbb{N}}

\newcommand{\detr}{{\det}_{\rm r}}
\newcommand{\Leb}{\mathrm{Leb}}

\ifjmlrstyle
\newtheorem{assumption}[theorem]{Assumption}
\else
\newtheorem{lemma}{Lemma}
\newtheorem{definition}{Definition}
\newtheorem{theorem}{Theorem}

\newtheorem{assumption}{Assumption}
\newtheorem{example}{Example}
\newtheorem{corollary}{Corollary}
\newtheorem{proposition}{Proposition}
\fi

%% file: ref.bib
@inproceedings{nguyen2017loss,
  title={The loss surface of deep and wide neural networks},
  author={Nguyen, Quynh and Hein, Matthias},
  booktitle={International conference on machine learning},
  pages={2603--2612},
  year={2017},
  organization={PMLR}
}

@inproceedings{xu2023over,
  title={Over-parameterization exponentially slows down gradient descent for learning a single neuron},
  author={Xu, Weihang and Du, Simon},
  booktitle={The Thirty Sixth Annual Conference on Learning Theory},
  pages={1155--1198},
  year={2023},
  organization={PMLR}
}

@article{liu2022loss,
  title={Loss landscapes and optimization in over-parameterized non-linear systems and neural networks},
  author={Liu, Chaoyue and Zhu, Libin and Belkin, Mikhail},
  journal={Applied and Computational Harmonic Analysis},
  volume={59},
  pages={85--116},
  year={2022},
  publisher={Elsevier}
}

@article{zhang2026deep,
  title={Deep delta learning},
  author={Zhang, Yifan and Liu, Yifeng and Wang, Mengdi and Gu, Quanquan},
  journal={arXiv preprint arXiv:2601.00417},
  year={2026}
}

@inproceedings{zhu2025hyper,
  title={Hyper-connections},
  author={Zhu, Defa and Huang, Hongzhi and Huang, Zihao and Zeng, Yutao and Mao, Yunyao and Wu, Banggu and Min, Qiyang and Zhou, Xun},
  booktitle={International Conference on Learning Representations},
  volume={2025},
  pages={97183--97219},
  year={2025}
}

@inproceedings{nguyen2021tight,
  title={Tight bounds on the smallest eigenvalue of the neural tangent kernel for deep relu networks},
  author={Nguyen, Quynh and Mondelli, Marco and Montufar, Guido F},
  booktitle={International Conference on Machine Learning},
  pages={8119--8129},
  year={2021},
  organization={PMLR}
}

@article{li2023convex,
  title={Convex and non-convex optimization under generalized smoothness},
  author={Li, Haochuan and Qian, Jian and Tian, Yi and Rakhlin, Alexander and Jadbabaie, Ali},
  journal={Advances in Neural Information Processing Systems},
  volume={36},
  pages={40238--40271},
  year={2023}
}

@article{zhang2019gradient,
  title={Why gradient clipping accelerates training: A theoretical justification for adaptivity},
  author={Zhang, Jingzhao and He, Tianxing and Sra, Suvrit and Jadbabaie, Ali},
  journal={arXiv preprint arXiv:1905.11881},
  year={2019}
}

@article{moudafi2004remark,
  title={A remark on the convergence of the Tikhonov regularization without monotonicity},
  author={Moudafi, Abdellatif},
  journal={MATHEMATICAL INEQUALITIES AND APPLICATIONS},
  volume={7},
  pages={283--288},
  year={2004},
  publisher={ELEMENT}
}

@book{emmrich1999discrete,
  title={Discrete versions of Gronwall's lemma and their application to the numerical analysis of parabolic problems},
  author={Emmrich, Etienne},
  year={1999},
  publisher={Techn. Univ.}
}

@article{mityagin2015zero,
  title={The zero set of a real analytic function},
  author={Mityagin, Boris},
  journal={arXiv preprint arXiv:1512.07276},
  year={2015}
}

@inproceedings{
wang2022large,
title={Large Learning Rate Tames Homogeneity: Convergence and Balancing Effect},
author={Yuqing Wang and Minshuo Chen and Tuo Zhao and Molei Tao},
booktitle={International Conference on Learning Representations},
year={2022},
url={https://openreview.net/forum?id=3tbDrs77LJ5}
}

@article{rader1973nice,
  title={Nice demand functions},
  author={Rader, Trout},
  journal={Econometrica: Journal of the Econometric Society},
  pages={913--935},
  year={1973},
  publisher={JSTOR}
}

@article{lewkowycz2020large,
  title={The large learning rate phase of deep learning: the catapult mechanism},
  author={Lewkowycz, Aitor and Bahri, Yasaman and Dyer, Ethan and Sohl-Dickstein, Jascha and Gur-Ari, Guy},
  journal={arXiv preprint arXiv:2003.02218},
  year={2020}
}

@article{wang2023good,
  title={Good regularity creates large learning rate implicit biases: edge of stability, balancing, and catapult},
  author={Wang, Yuqing and Xu, Zhenghao and Zhao, Tuo and Tao, Molei},
  journal={arXiv preprint arXiv:2310.17087},
  year={2023}
}

@article{grauert1958levi,
  title={On Levi's problem and the imbedding of real-analytic manifolds},
  author={Grauert, Hans},
  journal={Annals of Mathematics},
  volume={68},
  number={2},
  pages={460--472},
  year={1958},
  publisher={JSTOR}
}

@inproceedings{allen2019convergence,
  title={A convergence theory for deep learning via over-parameterization},
  author={Allen-Zhu, Zeyuan and Li, Yuanzhi and Song, Zhao},
  booktitle={International Conference on Machine Learning},
  pages={242--252},
  year={2019},
  organization={PMLR}
}

@inproceedings{du2019gradient,
  title={Gradient descent finds global minima of deep neural networks},
  author={Du, Simon and Lee, Jason and Li, Haochuan and Wang, Liwei and Zhai, Xiyu},
  booktitle={International conference on machine learning},
  pages={1675--1685},
  year={2019},
  organization={PMLR}
}

@article{zou2020gradient,
  title={Gradient descent optimizes over-parameterized deep ReLU networks},
  author={Zou, Difan and Cao, Yuan and Zhou, Dongruo and Gu, Quanquan},
  journal={Machine Learning},
  volume={109},
  number={3},
  pages={467--492},
  year={2020},
  publisher={Springer}
}

@article{lee2019wide,
  title={Wide neural networks of any depth evolve as linear models under gradient descent},
  author={Lee, Jaehoon and Xiao, Lechao and Schoenholz, Samuel and Bahri, Yasaman and Novak, Roman and Sohl-Dickstein, Jascha and Pennington, Jeffrey},
  journal={Advances in neural information processing systems},
  volume={32},
  year={2019}
}

@article{zou2019improved,
  title={An improved analysis of training over-parameterized deep neural networks},
  author={Zou, Difan and Gu, Quanquan},
  journal={Advances in neural information processing systems},
  volume={32},
  year={2019}
}

@inproceedings{ji2019polylogarithmic,
  title={Polylogarithmic width suffices for gradient descent to achieve arbitrarily small test error with shallow ReLU networks},
  author={Ji, Ziwei and Telgarsky, Matus},
  booktitle={International Conference on Learning Representations},
  year={2019}
}

@inproceedings{chen2020much,
  title={How Much Over-parameterization Is Sufficient to Learn Deep ReLU Networks?},
  author={Chen, Zixiang and Cao, Yuan and Zou, Difan and Gu, Quanquan},
  booktitle={International Conference on Learning Representations},
  year={2020}
}

@article{song2019quadratic,
  title={Quadratic suffices for over-parametrization via matrix chernoff bound},
  author={Song, Zhao and Yang, Xin},
  journal={arXiv preprint arXiv:1906.03593},
  year={2019}
}

@article{oymak2020toward,
  title={Toward moderate overparameterization: Global convergence guarantees for training shallow neural networks},
  author={Oymak, Samet and Soltanolkotabi, Mahdi},
  journal={IEEE Journal on Selected Areas in Information Theory},
  volume={1},
  number={1},
  pages={84--105},
  year={2020},
  publisher={IEEE}
}

@article{jacot2018neural,
  title={Neural tangent kernel: Convergence and generalization in neural networks},
  author={Jacot, Arthur and Gabriel, Franck and Hongler, Cl{\'e}ment},
  journal={Advances in neural information processing systems},
  volume={31},
  year={2018}
}

@article{song2018mean,
  title={A mean field view of the landscape of two-layers neural networks},
  author={Song, Mei and Montanari, Andrea and Nguyen, P},
  journal={Proceedings of the National Academy of Sciences},
  volume={115},
  number={33},
  pages={E7665--E7671},
  year={2018}
}

@article{chizat2018global,
  title={On the global convergence of gradient descent for over-parameterized models using optimal transport},
  author={Chizat, Lenaic and Bach, Francis},
  journal={Advances in neural information processing systems},
  volume={31},
  year={2018}
}

@article{rotskoff2018neural,
  title={Neural networks as interacting particle systems: Asymptotic convexity of the loss landscape and universal scaling of the approximation error},
  author={Rotskoff, Grant M and Vanden-Eijnden, Eric},
  journal={stat},
  volume={1050},
  pages={22},
  year={2018}
}

@article{wei2019regularization,
  title={Regularization matters: Generalization and optimization of neural nets vs their induced kernel},
  author={Wei, Colin and Lee, Jason D and Liu, Qiang and Ma, Tengyu},
  journal={Advances in Neural Information Processing Systems},
  volume={32},
  year={2019}
}

@article{chen2020generalized,
  title={A generalized neural tangent kernel analysis for two-layer neural networks},
  author={Chen, Zixiang and Cao, Yuan and Gu, Quanquan and Zhang, Tong},
  journal={Advances in Neural Information Processing Systems},
  volume={33},
  pages={13363--13373},
  year={2020}
}

@article{sirignano2020mean,
  title={Mean field analysis of neural networks: A law of large numbers},
  author={Sirignano, Justin and Spiliopoulos, Konstantinos},
  journal={SIAM Journal on Applied Mathematics},
  volume={80},
  number={2},
  pages={725--752},
  year={2020},
  publisher={SIAM}
}

@inproceedings{fang2021modeling,
  title={Modeling from features: a mean-field framework for over-parameterized deep neural networks},
  author={Fang, Cong and Lee, Jason and Yang, Pengkun and Zhang, Tong},
  booktitle={Conference on learning theory},
  pages={1887--1936},
  year={2021},
  organization={PMLR}
}

@article{allen2020towards,
  title={Towards understanding ensemble, knowledge distillation and self-distillation in deep learning},
  author={Allen-Zhu, Zeyuan and Li, Yuanzhi},
  journal={arXiv preprint arXiv:2012.09816},
  year={2020}
}

@inproceedings{allen2022feature,
  title={Feature purification: How adversarial training performs robust deep learning},
  author={Allen-Zhu, Zeyuan and Li, Yuanzhi},
  booktitle={2021 IEEE 62nd Annual Symposium on Foundations of Computer Science (FOCS)},
  pages={977--988},
  year={2022},
  organization={IEEE}
}

@article{cao2022benign,
  title={Benign Overfitting in Two-layer Convolutional Neural Networks},
  author={Cao, Yuan and Chen, Zixiang and Belkin, Mikhail and Gu, Quanquan},
  journal={arXiv preprint arXiv:2202.06526},
  year={2022}
}

@article{chen2022feature,
  title={On feature learning in neural networks with global convergence guarantees},
  author={Chen, Zhengdao and Vanden-Eijnden, Eric and Bruna, Joan},
  journal={arXiv preprint arXiv:2204.10782},
  year={2022}
}

@article{chen2025global,
  title={Global Convergence and Rich Feature Learning in $ L $-Layer Infinite-Width Neural Networks under $\mu $P Parametrization},
  author={Chen, Zixiang and Yang, Greg and Zhao, Qingyue and Gu, Quanquan},
  journal={arXiv preprint arXiv:2503.09565},
  year={2025}
}

@article{yang2020feature,
  title={Feature learning in infinite-width neural networks},
  author={Yang, Greg and Hu, Edward J},
  journal={arXiv preprint arXiv:2011.14522},
  year={2020}
}

@article{ba2022high,
  title={High-dimensional asymptotics of feature learning: How one gradient step improves the representation},
  author={Ba, Jimmy and Erdogdu, Murat A and Suzuki, Taiji and Wang, Zhichao and Wu, Denny and Yang, Greg},
  journal={Advances in Neural Information Processing Systems},
  volume={35},
  pages={37932--37946},
  year={2022}
}

@article{bordelon2022self,
  title={Self-consistent dynamical field theory of kernel evolution in wide neural networks},
  author={Bordelon, Blake and Pehlevan, Cengiz},
  journal={Advances in Neural Information Processing Systems},
  volume={35},
  pages={32240--32256},
  year={2022}
}

@inproceedings{ma2018power,
  title={The power of interpolation: Understanding the effectiveness of SGD in modern over-parametrized learning},
  author={Ma, Siyuan and Bassily, Raef and Belkin, Mikhail},
  booktitle={International Conference on Machine Learning},
  pages={3325--3334},
  year={2018},
  organization={PMLR}
}

@article{hardt2016identity,
  title={Identity matters in deep learning},
  author={Hardt, Moritz and Ma, Tengyu},
  journal={arXiv preprint arXiv:1611.04231},
  year={2016}
}

@article{liu2019towards,
  title={Towards understanding the importance of shortcut connections in residual networks},
  author={Liu, Tianyi and Chen, Minshuo and Zhou, Mo and Du, Simon S and Zhou, Enlu and Zhao, Tuo},
  journal={Advances in neural information processing systems},
  volume={32},
  year={2019}
}

@article{liu2023aiming,
  title={Aiming towards the minimizers: fast convergence of SGD for overparametrized problems},
  author={Liu, Chaoyue and Drusvyatskiy, Dmitriy and Belkin, Misha and Davis, Damek and Ma, Yian},
  journal={Advances in neural information processing systems},
  volume={36},
  pages={60748--60767},
  year={2023}
}

@article{scholkemper2024residual,
  title={Residual connections and normalization can provably prevent oversmoothing in gnns},
  author={Scholkemper, Michael and Wu, Xinyi and Jadbabaie, Ali and Schaub, Michael T},
  journal={arXiv preprint arXiv:2406.02997},
  year={2024}
}

@article{huang2020deep,
  title={Why Do Deep Residual Networks Generalize Better than Deep Feedforward Networks?---A Neural Tangent Kernel Perspective},
  author={Huang, Kaixuan and Wang, Yuqing and Tao, Molei and Zhao, Tuo},
  journal={Advances in neural information processing systems},
  volume={33},
  pages={2698--2709},
  year={2020}
}

@article{qin2025convergence,
  title={On the Convergence of Gradient Descent on Learning Transformers with Residual Connections},
  author={Qin, Zhen and Zhou, Jinxin and Zhu, Zhihui},
  journal={arXiv preprint arXiv:2506.05249},
  year={2025}
}

@article{bordelon2023dynamics,
  title={Dynamics of finite width kernel and prediction fluctuations in mean field neural networks},
  author={Bordelon, Blake and Pehlevan, Cengiz},
  journal={Advances in Neural Information Processing Systems},
  volume={36},
  pages={9707--9750},
  year={2023}
}

@article{han2025precise,
  title={Precise gradient descent training dynamics for finite-width multi-layer neural networks},
  author={Han, Qiyang and Imaizumi, Masaaki},
  journal={arXiv preprint arXiv:2505.04898},
  year={2025}
}

@inproceedings{shi2021theoretical,
  title={A Theoretical Analysis on Feature Learning in Neural Networks: Emergence from Inputs and Advantage over Fixed Features},
  author={Shi, Zhenmei and Wei, Junyi and Liang, Yingyu},
  booktitle={International Conference on Learning Representations},
  year={2021}
}

@article{telgarsky2022feature,
  title={Feature selection with gradient descent on two-layer networks in low-rotation regimes},
  author={Telgarsky, Matus},
  journal={arXiv preprint arXiv:2208.02789},
  year={2022}
}

@inproceedings{glorot2010understanding,
  title={Understanding the Difficulty of Training Deep Feedforward Neural Networks},
  author={Glorot, Xavier and Bengio, Yoshua},
  booktitle={Proceedings of the Thirteenth International Conference on Artificial Intelligence and Statistics},
  pages={249--256},
  year={2010},
  volume={9},
  series={Proceedings of Machine Learning Research},
  publisher={PMLR}
}

@inproceedings{vaswani2017attention,
  title={Attention Is All You Need},
  author={Vaswani, Ashish and Shazeer, Noam and Parmar, Niki and Uszkoreit, Jakob and Jones, Llion and Gomez, Aidan N. and Kaiser, Lukasz and Polosukhin, Illia},
  booktitle={Advances in Neural Information Processing Systems},
  volume={30},
  year={2017}
}

@inproceedings{ronneberger2015u,
  title={U-Net: Convolutional Networks for Biomedical Image Segmentation},
  author={Ronneberger, Olaf and Fischer, Philipp and Brox, Thomas},
  booktitle={Medical Image Computing and Computer-Assisted Intervention},
  pages={234--241},
  year={2015},
  organization={Springer}
}

@article{gu2023mamba,
  title={Mamba: Linear-Time Sequence Modeling with Selective State Spaces},
  author={Gu, Albert and Dao, Tri},
  journal={arXiv preprint arXiv:2312.00752},
  year={2023}
}

@inproceedings{he2016deep,
  title={Deep Residual Learning for Image Recognition},
  author={He, Kaiming and Zhang, Xiangyu and Ren, Shaoqing and Sun, Jian},
  booktitle={Proceedings of the IEEE Conference on Computer Vision and Pattern Recognition},
  pages={770--778},
  year={2016}
}

@inproceedings{ioffe2015batch,
  title={Batch Normalization: Accelerating Deep Network Training by Reducing Internal Covariate Shift},
  author={Ioffe, Sergey and Szegedy, Christian},
  booktitle={International Conference on Machine Learning},
  pages={448--456},
  year={2015},
  organization={PMLR}
}

@article{ba2016layer,
  title={Layer Normalization},
  author={Ba, Jimmy Lei and Kiros, Jamie Ryan and Hinton, Geoffrey E.},
  journal={arXiv preprint arXiv:1607.06450},
  year={2016}
}

@inproceedings{loshchilov2017decoupled,
  title={Decoupled Weight Decay Regularization},
  author={Loshchilov, Ilya and Hutter, Frank},
  booktitle={International Conference on Learning Representations},
  year={2019}
}
